\documentclass[11pt]{amsart}
\usepackage[letterpaper,margin=1in]{geometry}
\usepackage{bbm}
\usepackage{pdflscape}
\usepackage{float, graphicx}
\usepackage[]{epsfig}
\usepackage{amsmath, amsthm, amssymb}
\usepackage{epsfig}
\usepackage{verbatim}
\usepackage{multicol}
\usepackage{url}
\usepackage{booktabs}
\usepackage{colortbl}
\usepackage{booktabs}
\usepackage[foot]{amsaddr}
\usepackage{mathtools}
\usepackage{latexsym}
\usepackage{mathrsfs}
\usepackage[colorlinks, bookmarks=true]{hyperref}
\usepackage{graphicx}
\usepackage{enumerate}
\usepackage[normalem]{ulem}
\usepackage{bm}
\usepackage{dsfont}
\usepackage{stmaryrd}
\usepackage{enumitem}
\usepackage{soul}
\usepackage{pifont}
\usepackage{subcaption}
\usepackage[noadjust]{cite}
\usepackage{color}
\usepackage{xcolor}
\usepackage{hyperref}
\usepackage{indentfirst}
\usepackage{amssymb}
\usepackage[capitalise]{cleveref}
\crefname{assumption}{assumption}{assumptions}
\Crefname{assumption}{Assumption}{Assumptions}

\numberwithin{equation}{section}
\usepackage{mathtools}
\newcommand{\trace}{\mathrm{trace}}

\newcommand{\cE}{{\mathcal E}}

\newcommand{\bT}{{\mathbb T}}

\newcommand{\fc}{{\mathfrak c}}

\newcommand{\cD}{{\mathcal D}}

\newcommand{\cN}{{\mathcal N}}

\newcommand{\bR}{{\mathbb R}}

\newcommand{\wt}{\widetilde}
\newcommand{\rscore}{\bm{\epsilon}}

\def\TV{\mathsf{TV}}

\def\<{\langle}
\def\>{\rangle}
\def\dd{\mathrm{d}}

\def\op{\mbox{\rm \tiny op}}

\def\cE{\mathcal{E}}

\def\Lamp[#1]{\boldsymbol{\Lambda}_{\mathrm{AMP}}^{(#1)}}
\def\lalg[#1]{\Lambda_{\mathrm{alg}, #1}}

\def\cE{\mathcal{E}}

\def\de{{\rm d}}

\def\RR{\mathbb{R}}

\newcommand{\rd}{{\rm d}}

\newcommand{\del}{\partial}
\newcommand{\al}{\alpha}

\def\cO{\mathcal{O}}

\def\id{{\mathbb I}}

\newcommand{\N}{\mathcal{N}}

\newcommand{\RN}[1]{%
  \textup{\uppercase\expandafter{\romannumeral#1}}%
}

\newcommand{\RNum}[1]{\uppercase\expandafter{\romannumeral #1\relax}}

\newcommand\qU{U}
\newcommand\svarrho{\widehat{\varrho}}

\newcommand\sU{\widehat{U}}
\newcommand\sV{\widehat{V}}
\newcommand\sY{\widehat{Y}}

\newcommand\dV{\widetilde{V}}
\newcommand\dU{\widetilde{U}}
\newcommand\dY{\widetilde{Y}}
\newcommand\dvarrho{\widetilde{\varrho}}

\newcommand\muast{\mu_{\ast}}
\newcommand\mft{\tau}

\newcommand\baV{\overline{V}}
\newcommand\baM{\overline{M}}
\newcommand\dM{\widetilde{M}}
\newcommand\dphi{\widetilde{\phi}}
\newcommand\baphi{\overline{\phi}}

\newcommand{\tick}{\ding{51}} 
\newcommand{\cross}{\ding{55}}

\theoremstyle{plain} 
\newtheorem{theorem}{Theorem}[section]
\newtheorem*{theorem*}{Theorem}
\newtheorem{lemma}[theorem]{Lemma}
\newtheorem*{lemma*}{Lemma}

\newtheorem*{corollary*}{Corollary}
\newtheorem{proposition}[theorem]{Proposition}
\newtheorem*{proposition*}{Proposition}
\newtheorem{assumption}[theorem]{Assumption}
\newtheorem*{assumption*}{Assumption}

\newtheorem*{definition*}{Definition}

\newtheorem*{example*}{Example}
\newtheorem{remark}[theorem]{Remark}

\newtheorem*{remark*}{Remark}
\newtheorem*{remarks*}{Remarks}

\definecolor{darkred}{rgb}{.6,0,0}
\definecolor{darkblue}{rgb}{0,0,.7}
\definecolor{darkgreen}{rgb}{0,.7,0}
\definecolor{darkbrown}{rgb}{0.8,0.4,0.4}
\definecolor{purple}{rgb}{0.5,0.0,0.5}

\begin{document}

\title{Fast Convergence for High-Order ODE Solvers in Diffusion Probabilistic Models}

\author{Daniel~Zhengyu~Huang\textsuperscript{1}}
\address{\textsuperscript{1}Beijing International Center for Mathematical Research, Center for Machine Learning Research, Peking University, Beijing, China}
\email{huangdz@bicmr.pku.edu.cn}

\author{Jiaoyang Huang\textsuperscript{2}}
\address{\textsuperscript{2}Department of Statistics and Data Science, University of Pennsylvania, Philadelphia, PA, USA}
\email{huangjy@wharton.upenn.edu}

\author{Zhengjiang Lin\textsuperscript{3}}
\address{\textsuperscript{3}Department of Mathematics, Massachusetts Institute of Technology, 77 Massachusetts Ave, 02139 Cambridge MA, USA}
\email{linzj@mit.edu}    
\date{}

\begin{abstract}

Diffusion probabilistic models generate samples by learning to reverse a noise-injection process that transforms data into noise. 
A key development is the reformulation of the reverse sampling process as a deterministic probability flow ordinary differential equation (ODE), which allows for efficient sampling using high-order numerical solvers. 
Unlike traditional time integrator analysis, the accuracy of this sampling procedure depends not only on numerical integration errors but also on the approximation quality and regularity of the learned score function, as well as their interaction.  
In this work, we present a rigorous convergence analysis of deterministic samplers derived from probability flow ODEs for general forward processes with arbitrary variance schedules. Specifically, we develop and analyze $p$-th order (exponential) Runge-Kutta (RK) schemes, under the practical assumption that the first and second derivatives of the learned score function are bounded. We prove that the total variation distance between the generated and  target distributions can be bounded as
\begin{align*}
   O\bigl(d^{\frac{7}{4}}\varepsilon_{\text{score}}^{\frac{1}{2}} +d(dH_{\max})^p\bigr),
\end{align*}
where $\varepsilon^2_{\text{score}}$ denotes the $L^2$ error in the score function approximation, $d$ is the data dimension, and $H_{\max}$ represents the maximum solver step size. 
Numerical experiments on benchmark datasets further confirm that the derivatives of the learned score function are bounded in practice.

\end{abstract}



\maketitle

{
\hypersetup{linkcolor=black}
\setcounter{tocdepth}{1}
\tableofcontents
}

\section{Introduction}

Score-based generative models have recently emerged as a powerful framework for generating data from complex, high-dimensional data distributions~\cite{sohl2015deep, ho2020denoising, song2019generative, song2020score, dhariwal2021diffusion}. Unlike traditional generative models that directly learn a mapping from random noise to samples of the target distribution~\cite{kingma2013auto, goodfellow2014generative, rezende2015variational, papamakarios2021normalizing},  score-based models leverage a pair of stochastic processes: a forward process that incrementally corrupts data into noise, and a reverse process that iteratively reconstructs the data.
The forward process, also known as the diffusion process, gradually perturbs data samples from the target distribution $\mu_*$ into pure noise. 
Notably, the target distribution may either possess a density $q_0$, or it may be supported on a low-dimensional manifold, in which case it does not admit a well-defined density.
The reverse process then reconstructs samples by reversing this transformation, guided by the score function—the gradient of the log-density—which is learned via score matching techniques~\cite{hyvarinen2005estimation, vincent2011connection, song2019generative, song2020score}. This approach has demonstrated remarkable success in generating high-fidelity audio and visual data~\cite{dhariwal2020jukebox, dhariwal2021diffusion, popov2021grad, ramesh2022hierarchical, esser2024scaling}.

The reverse process in score-based generative models can be implemented through either stochastic dynamics—referred to as the denoising diffusion probabilistic model (DDPM), originally proposed by \cite{sohl2015deep,ho2020denoising}—or deterministic dynamics, commonly known as the probability flow ordinary differential equation (ODE) sampler or the denoising diffusion implicit model \cite{song2020denoising,song2020score}. 
Both approaches are theoretically linked by the same underlying Fokker–Planck equation, which characterizes the evolution of the data distribution in the reverse process. However, their practical implementations involve different discretization schemes, which can lead to notable differences in algorithmic behavior, affecting sample quality, computational cost, and numerical stability. These differences underscore the importance of rigorous numerical analysis in the design of solvers for diffusion-based generative models.

The convergence properties of stochastic dynamics used in the reverse process, when discretized via the (exponential) Euler–Maruyama scheme, have been extensively studied in recent literature~\cite{block2020generative,de2021diffusion,yang2022convergence,kwon2022score,tang2024contractive,lee2022convergence,chen2022sampling,lee2023convergence,chen2023score,chen2023improved,wu2024theoretical,tang2024score,mooney2024global,chen2024accelerating,chen2025solving, zhou2024parallel},
providing polynomial complexity guarantees with respect to the dimension $d$, without relying on structural assumptions like log-concavity of the data distribution. Recent state-of-the-art theoretical results by \cite{li2024d} demonstrate that DDPM requires at most on the order of (up to logarithmic factors) $d/\varepsilon$ iterations to generate samples within $\varepsilon$ total variation (TV) distance from a fairly general class of target distributions, assuming perfect score function estimates. The stochastic nature of these dynamics plays a critical role in mitigating error accumulation. Additionally, recent advances by \cite{wu2024stochastic, yu2025advancing} investigate high-order stochastic samplers based on stochastic Runge–Kutta schemes.

Although diffusion-based samplers such as DDPM are renowned for generating high-fidelity samples, they typically suffer from slow sampling speeds due to requiring a large number of score function evaluations.
In contrast,  probability flow ODEs can be solved using numerical discretization methods, including forward Euler, Heun’s method, and higher-order (exponential) Runge-Kutta schemes~\cite{butcher2016numerical,hochbruck2010exponential}. Recent advances~\cite{song2020denoising, song2020score, lu2022dpm, lu2022dpm+, zhang2022fast, zhao2024unipc, karras2022elucidating, xue2024accelerating} either leverage classical high-order schemes or develop novel higher-order exponential Runge-Kutta schemes, significantly reducing the required denoising steps, often to as few as $\mathcal{O}(10)$, compared to the (exponential) Euler-Maruyama scheme traditionally used for stochastic simulations, which can necessitate up to $\mathcal{O}(1000)$ steps. Consequently, these advanced deterministic approaches offer enhanced sampling efficiency with only modest compromises in sample quality.

The theoretical analysis of the probability flow ODE is more challenging. At the continuous-time level, \cite{albergo2022building, benton2023error} analyze the convergence of deterministic dynamics arising from more general stochastic interpolants or flow-matching frameworks~\cite{lipman2022flow, liu2022flow, albergo2023stochastic, boffi2206probability}. Their error bounds exhibit exponential dependence on the Lipschitz constant of the score approximation. Besides, the behavior of probability flow ODEs with first-order solvers, such as the forward Euler scheme, has been extensively studied. For instance, \cite{chen2023restoration} assume no score-matching error and derive error bounds that exhibit significant dependence on the dimensionality and grow exponentially with the time-integrated Lipschitz constant of the score. In contrast, \cite{chen2024probability} assume an $L^2$ bound on the score estimation  and offer polynomial-time convergence guarantees for probability flow ODEs combined with a stochastic Langevin corrector. \cite{gao2024convergence} provides convergence analysis under a log-concave data distribution assumption, showing that the error grows exponentially with time $T$ when score-matching error is present.

However, empirical evidence suggests that the error does not necessarily grow exponentially in practice. More recently, in a series of works, \cite{li2023towards,li2024unified, li2024sharp} establish convergence guarantees for probability flow ODEs, showing that $d/\varepsilon$ iterations suffice to generate samples with total variation error $\varepsilon$. Their analysis requires control over the difference between both the values and the derivatives of the true and estimated scores. Under similar assumptions, \cite{li2024accelerating} analyzes the case of second-order ODE solvers. In our previous work~\cite{huang2024convergence}, we initiated the study of probability flow ODEs with general $p$-th order Runge-Kutta solvers, for any integer $p \geq 1$, assuming boundedness of the first $p$ derivatives of the estimated score.
In this work, we generalize our previous analysis framework~\cite{huang2024convergence} to include 
$p$-th order exponential Runge-Kutta solvers, widely used in practice. Notably, our analysis requires only the second derivatives of the approximate score function to be bounded. This goes beyond classical local error analysis, which typically assumes higher-order smoothness.

\begin{table}[ht]
\centering
\renewcommand{\arraystretch}{1.4}
\begin{tabular}{lccccc}
\toprule
\textbf{paper} & \textbf{smoothness of}& \textbf{score matching} &   \textbf{iteration} & \textbf{higher-order}  \\
 & \textbf{scores} &  \textbf{assumption}  & \textbf{complexity} & \textbf{solver }\\
\midrule
\cite{chen2023restoration} & $L$-Lipschitz & $s_t = s_t^*$ &
poly$(Ld)/\varepsilon^2$ & \cross \\
\cite{li2023towards} & no requirement & $s_t \approx s_t^*, \frac{\partial s_t}{\partial x} \approx \frac{\partial s_t^*}{\partial x}$  & $d^2/\varepsilon + d^3/\sqrt{\varepsilon}$  &\cross\\
\cite{li2024sharp}  & no requirement & $s_t \approx s_t^*, \frac{\partial s_t}{\partial x} \approx \frac{\partial s_t^*}{\partial x}$  & $d/\varepsilon + d^2$ &\cross \\
\cite{li2024unified} & $L$-Lipschitz & $s_t \approx s_t^*, \frac{\partial s_t}{\partial x} \approx \frac{\partial s_t^*}{\partial x}$  & $Ld(L+d)/\varepsilon$ &\cross \\
\cite{li2024accelerating} & no requirement & $s_t \approx s_t^*, \frac{\partial s_t}{\partial x} \approx \frac{\partial s_t^*}{\partial x}$ &$d^3/\sqrt{\varepsilon}$ &\tick $2$-nd order\\

\cite{huang2024convergence} & $C^p$ & $s_t \approx s_t^*$  & $ d^{1+1/p}/\varepsilon^{1/p}$ &\tick $p$-th order \\
\rowcolor{gray!20}  \Cref{theorem: main L^1 theorem}  & $C^2$ & $s_t \approx s_t^*$  & \cellcolor{gray!20}$ d^{1+1/p}/\varepsilon^{1/p}$&\cellcolor{gray!20} \tick $p$-th order  \\
\bottomrule
\end{tabular}
\vspace{0.2cm}
\caption{Comparison of convergence guarantees in total variation error and iteration complexities of recent works. 
Here, $s_t$ and $\frac{\partial s_t}{\partial x}$ are the estimated score function and its derivative, with $s_t^{*}$ and $\frac{\partial s_t^{*}}{\partial x}$ as their true counterparts.
$d$ represents the data dimension, 
$L$ denotes the corresponding Lipschitz constant of $s_t^{*}$, and 
$\varepsilon$ refers to the target distribution error. When computing iteration complexity, we neglect the score matching error and estimate the number of time steps required to achieve an error of $\varepsilon$.}
\label{table:comparison}
\end{table}


\subsection{Our Contributions}
Specifically, our main contributions are summarized as follows:

\begin{itemize}
\item We establish convergence guarantees for $p$-th order Runge-Kutta schemes applied to the probability flow ODE at the discretized time level under the following mild assumptions:
\begin{itemize} \item \cref{assumption:secon-moment} assumes that the target density has compact support,
\item \cref{a:score-estimate} assumes that the $L^2$ discrete time score matching error over time is bounded by $\varepsilon^2_{\rm score}$.
\item \cref{a:bounds on RK} assumes that the coefficients involved in the Runge-Kutta schemes and their derivatives are bounded, and both the time steps and substeps lie within the discrete time sequence used for score matching.
\item \cref{a:score-derivative} assumes that the first and second derivatives of the estimated score are bounded.
\end{itemize}
Under these assumptions, we prove in \cref{theorem: main L^1 theorem} that for each $p \geq 1$, the total variation distance between the target distribution and the generated data distribution is bounded above by  
\begin{align}\label{e:bound}
   O(d^{\frac{7}{4}} \varepsilon_{\rm score}^{\frac{1}{2}} + d(d H_{\max})^p),
\end{align}
where $d$ is the data dimension and $H_{\max}$ is the maximum time step size used in the solver.

This result improves upon our previous work in \cite{huang2024convergence} in three key aspects:  
\begin{enumerate}
    \item In \cite{huang2024convergence}, we required boundedness of the first $(p+1)$-th  derivatives of the approximate score. Here, we only assume boundedness of the first and second derivatives. 
    \item We generalize the analysis of the reverse process to allow for non-uniform time step sizes, as long as the maximum time step size $H_{\max}$ remains sufficiently small.
    \item In \cref{theorem: main L^1 theorem general}, we extend our results beyond the Ornstein–Uhlenbeck forward process to more general forward processes.
\end{enumerate}

\item In \cref{s:general}, we develop $p$-th order (exponential) Runge–Kutta schemes with provable high-order convergence guarantees, applicable to  arbitrary forward processes with general variance schedules.

\item In \cref{sec:num}, we present numerical experiments. We examine the gradient and Hessian of approximate scores in practical diffusion models, and our results show that both quantities remain well-bounded, thereby supporting our theoretical assumptions. Additionally, we demonstrate high-order convergence for the designed (exponential) Runge–Kutta schemes using Gaussian mixture target densities. Finally, we explore sharp error bounds in this setting.

\end{itemize}

\subsection{Notations}

\begin{itemize}
\item Diacritics: $\widehat{\square}$ denotes quantities involve score error, $\widetilde{\square}$ denotes quantities involve both score and time discretization errors.
\item Time steps: $0 = t_0 < t_1< \cdots < t_N  = T-\tau$, where $\tau>0$ is a small parameter.  
    \item Distributions on $\mathbb{R}^d$: $q_t, \widehat q_t, \widetilde q_t$ denote forward process, $\varrho_t = q_{T-t}, \widehat \varrho_t = \widehat q_{T-t}, \dvarrho_t = \widetilde q_{T-t}$ reverse process.
    \item Vector fields from $\mathbb{R}^d$ to $\mathbb{R}^d$:
    Forward process: $\qU_t(x) \coloneq x +  \nabla \log q_t(x)$, $\sU_t \coloneq x + s_{T -t}(x)$,  $\dU_t \coloneq x + \widetilde s_{T -t}(x)$; Reversal process: $V_t \coloneq U_{T-t}$, $\sV_t \coloneq \sU_{T-t}$, $\dV_t \coloneq \dU_{T-t}$; Other vector fields:
    $Z_t(x)$.
    \item $\alpha = (\alpha_1, \alpha_2, \ldots, \alpha_d)$ is a multi-index with nonnegative integers $\alpha_i$'s, $|\alpha| \coloneq \alpha_1 + \alpha_2 + \ldots +\alpha_d$, and we define $\partial_x ^\alpha  \coloneq \partial_{x_1} ^{\alpha_1} \partial_{x_2} ^{\alpha_2} \cdots \partial_{x_d} ^{\alpha_d}$. We also use $\partial_i \coloneq \partial_{x_i}$ for simplicity.
    \item Constants: We use $C_u$ to denote universal constants like $10, 50, 100, 200$, i.e., $C_u$ is independent of the dimension $d$ and other parameters in this paper. Also, $C_u$ may vary by lines.
    \item Norms: For a vector $x = (x_1, x_2, \ldots, x_d)\in \mathbb{R}^d$, we use $\|x\|=\|x\|_2 \coloneq {(x_1 ^2 + x_2 ^2 + \ldots + x_d ^2)}^{\frac{1}{2}}$, $\|x\|_{\infty} \coloneq \sup_{1 \leq i \leq d} |x_i|$, $\|x\|_{1} \coloneq |x_1| + |x_2| + \ldots+|x_d|$, $\|x\|_{p} \coloneq {({|x_1|}^p + {|x_2|}^p + \ldots+{|x_d|}^p)}^{\frac{1}{p}}$. We similarly define $\|\cdot \|_2$, $\| \cdot \|_{\infty}$, $\| \cdot \|_{1}$, $\| \cdot \|_{p}$,  for matrices or even more general tensors, because we can view them as vectors and forget their tensor structures.
    For a vector-valued function $F(x) = (f_1(x), f_2(x), \ldots, f_m(x)): \mathbb{R}^d \mapsto \mathbb{R}^m$, where $m$ is a positive integer, we regard $F(x)$ as a vector in $\mathbb{R}^m$ and similarly use the notations $\|F(x)\|$, $\|F(x)\|_{1}$, $\|F(x)\|_{\infty}$. In particular, if $A$ is a $d \times d$ matrix, we use $\|A\|_{op}$ to denote its operator norm, i.e., $\|A\|_{op} \coloneq \sup_{\|x\|_2 =1} \|Ax\|_2$. 
\item Function class: We say a vector-valued function $F(x) = (f_1(x), f_2(x), \ldots, f_m(x)): \bR^d\mapsto \bR^m$ as being in $C^r$, if each of its components $f_i(x)$ has continuous first $r$-th derivatives. We say $F(x)$ is in the $L^p$-space $L^{p} (\mathbb{R}^d)$ if for each of its components $f_i(x)$, its $L^p(\mathbb{R}^d)$-norm defined as $\|f_i\|_{L^p(\mathbb{R}^d)} \coloneq {(\int_{\mathbb{R}^d} {|f_i(x)|}^p \de x)}^{\frac{1}{p}}$ is finite. We say $F(x)$ is in the Sobolev space $W^{r,p} (\mathbb{R}^d)$ if for each of its components $f_i(x)$, $\partial_x ^{\alpha} f_i \in L^p(\mathbb{R}^d)$ for each $\alpha$ with $|\alpha| \leq r$. We define the $W^{r,p} (\mathbb{R}^d)$-norm of $f_i$ as $\|f_i\|_{W^{r,p}(\mathbb{R}^d)} \coloneq (\sum_{|\alpha| \leq r} \|\partial_x ^{\alpha} f_i\|_{L^p(\mathbb{R}^d)} ^p )^{\frac{1}{p}}$. 

\item Coordinates: For a vector $v \in \mathbb{R}^d$, we sometimes denote its $q$-th coordinate as $v^{(q)}$. We either use the notation $v = (v_1, v_2, \ldots, v_d)$ or the notation $v = (v^{(1)}, v^{(2)}, \ldots, v^{(d)})$.
\end{itemize}

\section{Score-based Generative Model}\label{Section: Preliminaries}
\subsection{Forward and Reverse Process}

Score-based generative models begin with $d$ dimensional true data samples $\{X_0\}$ following an unknown target distribution $\muast$\footnote{Even if $\muast$ does not have a density, the forward process produces a density $q_t$ for all $t > 0$ by the added noise. The reverse process then aims to sample new data from $q_\tau$ for sufficiently small $\tau > 0$.
}. The objective is to sample new data from the target distribution.
Typically, the score-based generative models usually involve two processes---the forward and reverse processes. 

In the forward process, we start with data samples from $\muast$, and progressively transform the data into noise. This process is often based on the canonical Ornstein-Uhlenbeck~(OU) process given by 
\begin{align}
\label{eq:OU-0}
	\de X_t = -X_t \de t + \sqrt{2} \de B_t, \qquad X_0 \sim \muast, \qquad 0 \leq t \leq T, 
\end{align}
where $(B_t)_{t \in [0, T]}$ is a standard Brownian motion in $\RR^d$. 
The OU process has an analytical solution 
\begin{align}
\label{eq:OU}
	X_t \overset{d}{=} \lambda_t X_0 + \sigma_t W, \qquad W \sim \cN(0, \id_d),  \quad \textrm{ with } \quad \lambda_t = e^{-t} \quad \textrm{ and } \quad\sigma_t = \sqrt{1 - e^{-2t}}.
\end{align}
Here $\overset{d}{=}$ denotes that the random variables have the same distribution.

The OU process exponentially converges to its stationary distribution, the standard Gaussian distribution $\cN(0, \id_d)$.
Let $q_t$ denote the density of $X_t$, which evolves according to  the following Fokker--Planck equation: 
\begin{align*}
    \del_t q_t= \nabla \cdot \bigl((x +  \nabla \log q_t(x)) q_t\bigr) = \nabla \cdot (\qU_t q_t), \qquad 0 < t \leq T,
\end{align*}
with $\qU_t(x) \coloneq x +  \nabla \log q_t(x)$.

By denoting
    \begin{align*}
        \varrho_t  \coloneq  q_{T - t},
    \end{align*}
as the time reversal process from time $T$ to $0$, $\varrho_t$ satisfies the following partial differential equation (PDE):
\begin{align}\begin{split}
\label{e:reversep}
    \del_t \varrho_t =-\nabla \cdot \bigl((x + \nabla \log q_{T - t}) \varrho_t\bigr) = -\nabla \cdot (V_t \varrho_t), \qquad 0 \leq t < T,
\end{split}\end{align}
with $V_t(x) \coloneq x +  \nabla \log q_{T - t}(x)$.
The score function $\nabla\log q_{T- t}(x)$ is typically learned by a neural network trained using score matching techniques with progressively corrupted trajectories $\{X_t\}$ from \eqref{eq:OU-0}. Subsequently,  the reverse PDE can be solved from $\varrho_0 = q_T$ to sample new data from $\muast$.

The reverse  PDE \eqref{e:reversep} is often reformulated into a mean field equation for sampling instead of being directly solved. 
This mean field equation can manifest as stochastic dynamics 
\begin{align}\label{eq:reverse}
	\de Y_t = \left( Y_t + 2 \nabla \log q_{T - t} (Y_t) \right) \de t + \sqrt{2} \de B_t', \qquad Y_0 \sim q_T, \qquad 0 \leq t < T,
\end{align}
where $(B_t')_{0 \leq t \leq T}$ is a Brownian motion in $\RR^d$.
This formulation is commonly referred to as the denoising diffusion probabilistic model (DDPM). Alternatively, the mean-field equation can adopt a deterministic dynamics framework in terms of an ordinary differential equation with velocity field $V_t$:
\begin{align}\label{eq:reverse-ode}
	\del_t Y_t =  Y_t +  \nabla \log q_{T - t} (Y_t)  = V_t(Y_t), \qquad Y_0 \sim q_T, \qquad 0 \leq t < T,
\end{align}
known as the probability flow ODE. 
Additionally, when $\nabla \log q_{T - t} (x)$ is represented by the learned score function $s_t(x)$, the probability flow ODE \eqref{eq:reverse-ode} becomes 
\begin{align}\label{eq:reverse-ode-score}
	\del_t \sY_t =  \sY_t +  s_t(\sY_t) = \sV_t(\sY_t), \qquad \sY_0 \sim \svarrho_0, 
\end{align}
where the velocity field becomes $\sV_t(x) \coloneq x + s_t(x)$, $\sY_0$ is sampled from $\svarrho_0$. Since the density $\varrho_0=q_T$ is unknown, $\svarrho_0$ is commonly approximated by the standard Gaussian distribution $\cN(0, \id_d)$, which serves as a reliable approximation of $q_T$ for sufficiently large $T$ as shown in \Cref{lemma: qt close to Gaussian}.
The associated density of $\sY_t$ is denoted as $\svarrho_{t}$, which differs from $\varrho_t$ that describes the density of $Y_t$, due to the score matching error.

\subsection{Score Matching}\label{s:score_matching}
The goal of score matching is to learn the score function 
$$s_t(x) \approx \nabla \log q_{T-t}(x) \quad \textrm{for} \quad 0\leq t < T,$$ 
from corrupted trajectories $\{X_t\}$ generated by the forward process in \eqref{eq:OU}. By the variational formulation introduced in \cite{vincent2011connection}, the $L^2$-error at reverse time $t$ becomes tractable and can be expressed as
\begin{align*}
    \mathbb{E}_{x\sim q_{T-t}} \left[  \left\| s_{t}(x) -  \nabla\log q_{T-t}(x)\right\|_2^2 \right]
&= \mathbb{E}_{x_0\sim q_0}\mathbb{E}_{x\sim q_{T-t}(x|x_0)}\left[  \left\| s_{t}(x) -\nabla \log q_{T-t}(x|x_0)\right\|_2^2 \right] + \textrm{const}\\
&= \mathbb{E}_{x_0\sim q_0}\mathbb{E}_{w \sim \cN(0,\id_d)}\left[  \left\| s_{t}(\lambda_{T-t} x_0 + \sigma_{T-t} w) + w/\sigma_{T -t}\right\|_2^2 \right] +\textrm{const},
\end{align*}
where $q_{t}(x|x_0)$ denotes the conditional distribution of $X_{t}$ given $X_0 = x_0$, which is  Gaussian by \eqref{eq:OU}. The additive constant may differ between lines and is independent of the score function

To recover the score function over the entire time interval $0\leq t\leq T$, the discrete time score matching approach~\cite{ho2020denoising,vincent2011connection,song2020score} minimizes a weighted $L^2$-error over a discrete sequence of reverse times 
\begin{align}\label{e:defbT}
   \quad \bT:=\{j\Delta t : 0 \leq j \leq N \},\quad \text{where } \Delta t = \frac{T}{N}.
\end{align}
In practice, the discrete time sequence $\bT$ is typically chosen as a uniform grid~\cite{rombach2022high,karras2022elucidating} with fixed step size $\Delta t = T/N$. The optimization objective uses weights $\sigma_{T-t}^2$ \cite{ho2020denoising}, resulting in the following formulation: 
\begin{equation}
\label{eq:score-matching}
\begin{split}
s_t(x) 
&= \arg\min_{s: \bR \times \bR^d\mapsto \bR^d} 
\sum_{t\in \mathbb T} \sigma_{T-t}^2 \mathbb{E}_{x\sim q_{T-t}} \left[  \left\| s_{t}(x) -  \nabla\log q_{T-t}(x)\right\|_2^2 \right]\\
&= \arg\min_{s: \bR \times \bR^d\mapsto \bR^d} 
\sum_{t\in \mathbb T} \sigma_{T-t}^2 \mathbb{E}_{x_0\sim q_0}\mathbb{E}_{w \sim \cN(0,\id_d)}\left[  \left\| s_{t}(\lambda_{T-t} x + \sigma_{T-t} w) + w/\sigma_{T -t}\right\|_2^2 \right]\\
&= \arg\min_{s: \bR \times \bR^d\mapsto \bR^d} 
\sum_{t\in \mathbb T} \mathbb{E}_{x_0\sim q_0}\mathbb{E}_{w \sim \cN(0,\id_d)}\left[  \left\|  \sigma_{T-t} s_{t}(\lambda_{T-t} x + \sigma_{T-t} w) + w\right\|_2^2 \right].
\end{split}
\end{equation}

Our theoretical analysis is based on this discrete time score matching framework, which ensures that the score matching error is controlled only at the discrete times in $\bT$. Accordingly, we design probability flow ODE solvers that require score evaluations only at these discrete times.
Furthermore, rather than learning the score function $s_t(x)$ directly via neural networks, a reparameterized version $\rscore_t(x) = \sigma_{T-t} s_t(x)$  is typically learned \cite{ho2020denoising}. This reparameterization helps to mitigate the potential singularity of $s_t(x) \to \infty$ as $t \to T$, since $\sigma_{0} = 0$.\footnote{When the target distribution $\muast$ is a Dirac delta function $\delta_0$,  the corresponding score function becomes $s_t(x) \approx \nabla\log q_{T-t}(x) = -\frac{2}{\sigma_{T-t}^2}$, which diverges as $t\to T$. This behavior commonly arises when the target distribution is supported on a low-dimensional manifold, in which case it does not admit a well-defined density in the ambient space. \label{fn:st}} Under this reparameterization, the score matching objective in \eqref{eq:score-matching}
is reformulated as an optimization problem over $\rscore_t$:
\begin{equation*}
\begin{split}
\rscore_t(x) 
&= \arg\min_{\rscore: \bR \times \bR^d\mapsto \bR^d} 
\sum_{t\in \mathbb T} \mathbb{E}_{x_0\sim q_0}\mathbb{E}_{w \sim \cN(0,\id_d)}\left[  \left\|  \rscore_{t}(\lambda_{T-t} x + \sigma_{T-t} w) + w\right\|_2^2 \right].
\end{split}
\end{equation*}

\subsection{ODE Solver}
\label{s:time_integrator}

This work focuses on solving the probability flow ODE \eqref{eq:reverse-ode-score} to generate new data. We begin by randomly initializing $\widehat Y_{0}$ from the standard Gaussian distribution $\cN(0, \id_d)$, which approximates the terminal distribution $q_T$.
We then numerically integrate the ODE over the interval $[0,T-\tau]$
using a Runge-Kutta method, where $\tau$ is introduced to avoid potential singularities caused by the score function (see footnote \ref{fn:st}). Typically, $\tau$ is chosen as a multiple of $\Delta t$, the step size of the discrete time sequence $\bT$.
The time domain is discretized into $M$ time steps: $0 = t_0 < t_1< \cdots < t_M  = T-\tau$. 
Although the Runge-Kutta method is a single step method, it generally involves multiple internal stages, that may require evaluating $\sV_t$ at intermediate times $t \in (t_i , t_{i+1})$.
Since the score function is only trained at these discrete times in $\bT$ (see \cref{s:score_matching}), we assume that all Runge-Kutta evaluation times, including both primary time steps and intermediate stage points, are contained in $\bT$ (see \cref{a:bounds on RK}). This assumption ensures compatibility with the available score evaluations and is essential for achieving high-order accuracy, especially when the learned score lacks smoothness.

In the following, we introduce standard Runge–Kutta schemes constructed to fulfill this assumption for solving the probability flow ODE \eqref{eq:reverse-ode-score}. The extension to exponential Runge-Kutta schemes and arbitrary forward processes with general variance schedules is presented in \cref{s:GeneralProcess}.

An explicit Runge-Kutta  method is a single step, multistage time integrator. Let $H_i = t_{i+1} - t_i$ denote the step size, and  let $\dY_{t_i}$ represent the numerical approximation of $\sY_{t_i}$ in the probability flow ODE \eqref{eq:reverse-ode-score}. 
The family of explicit $s$-stage, $p$-th order Runge-Kutta schemes update the solution $\{\dY_{t_i}\}$ according to:
\begin{align}
\label{eq:RK-update}
\dY_{t_{i+1}} = \dY_{t_i} + H_i \sum_{j=1}^s b_j \widetilde k_j(\dY_{t_i}),
\end{align}
where the stage evaluations $\widetilde{k}_j$ are computed as
\begin{equation}
\label{eq:RK-update-k}
\begin{split}
    &\widetilde k_1(x) = \sV_{t_i + c_1 H_i}(x),\\
    &\widetilde k_2(x) = \sV_{t_i + c_2 H_i}\bigl(x + (a_{21}\widetilde k_1(x)) H_i\bigr),\\
    &\widetilde k_3(x) = \sV_{t_i + c_3 H_i}\bigl(x + (a_{31}\widetilde k_1(x) + a_{32}\widetilde k_2(x)) H_i\bigr),\\
    &\qquad\qquad\qquad\vdots
    \\
    &\widetilde k_s(x) = \sV_{t_i + c_s H_i}\bigl(x + (a_{s1}\widetilde k_1(x) + a_{s2}\widetilde k_2(x) + \cdots + a_{s,s-1}\widetilde k_{s-1}(x)) H_i\bigr).
\end{split}
\end{equation}
The lower triangular matrix $[a_{jk}]$ is called the Runge-Kutta matrix, while the coefficients $[b_j]$ and $[c_j]$ are referred to as the weights and nodes. The stage number $s$ and the parameters are chosen such that the local truncation error of \eqref{eq:RK-update} is $\cO(H_i^{p+1})$. Here and below, the scalar $s$ denotes the number of Runge-Kutta stages and should not be confused with the learned score $s_t$.
In this work, we focus on standard Runge-Kutta schemes of order up to four, as summarized by the Butcher tableaus in \cref{tab:butcher-specific}. Each node $c_j$ determines the evaluation time $t_i+c_jH_i$. The nodes need not be distinct or uniformly spaced, but we require that all stage times lie in $\mathbb{T}$.

\begin{table}[h]
\centering
\caption{Butcher tableaus for standard Runge-Kutta schemes of orders 1 through 4: forward Euler, explicit trapezoidal rule, Kutta's third-order method, and the classical four-stage fourth-order Runge-Kutta method.}
\renewcommand{\arraystretch}{1.15}
\begin{tabular}[b]{c|ccccc}
$0$ & $0$ & $0$ & $\cdots$ & $0$ & $0$ \\
$c_2$ & $a_{21}$ & $0$ & $\cdots$ & $0$ \\
$c_3$ & $a_{31}$ & $a_{32}$ & $\cdots$ & $0$& $0$ \\
$\vdots$ & $\vdots$ & $\vdots$ & $\ddots$ & $\vdots$ \\
$c_s$ & $a_{s1}$ & $a_{s2}$ & $\cdots$ & $a_{ss-1}$ &$0$ \\
\hline
       & $b_1$   & $b_2$   & $\cdots$ & $b_{s-1}$ & $b_s$
\end{tabular}
\quad
\begin{tabular}[b]{c|c}
$0$ & $0$  \\
\hline
& $1$
\end{tabular}
\quad
\begin{tabular}[b]{c|cc}
$0$ & $0$ & $0$ \\
$1$ & $1$ & $0$ \\
\hline
& $\frac{1}{2}$   & $\frac{1}{2}$
\end{tabular}
\quad
\begin{tabular}[b]{c|ccc}
$0$ & $0$ & $0$   & $0$ \\
$\frac{1}{3}$ & $\frac{1}{3}$  & $0$ & $0$    \\
$\frac{2}{3}$ & $0$ & $\frac{2}{3}$ &  $0$  \\
\hline
       & $\frac{1}{4}$   & $0$  & $\frac{3}{4}$
\end{tabular}
\quad
\begin{tabular}[b]{c|cccc}
$0$ & $0$ & $0$   & $0$ & $0$ \\
$\frac{1}{2}$ & $\frac{1}{2}$  & $0$ & $0$  & $0$  \\
$\frac{1}{2}$ & $0$ & $\frac{1}{2}$  & $0$  & $0$ \\
$1$ & $0$ & $0$ &  $1$ & $0$ \\
\hline
       & $\frac{1}{6}$   & $\frac{1}{3}$   & $\frac{1}{3}$ & $\frac{1}{6}$
\end{tabular}
\label{tab:butcher-specific}
\end{table}

\section{Convergence Analysis and Main Results}
In this section, we analyze the total variation error between the reference distribution $\varrho_{t}$ and the distribution $\dvarrho_{t}$ of the numerical solution $\dY_t$, produced by the standard Runge–Kutta schemes in \eqref{eq:RK-update}. The extension to exponential Runge-Kutta schemes and to forward processes with general variance schedules is present in \cref{s:GeneralProcess}. 
We begin by stating our assumptions, followed by the main result. 

\begin{assumption}\label{assumption:secon-moment}
	The data distribution $\muast$ is a probability distribution that is positive and compactly supported on a compact set $K_{\ast}$ in $\bR^d$, and we also define $D \coloneq 1+\max_{x \in K_{\ast}} \|x\|_{\infty} $.
\end{assumption}

\begin{assumption}\label{a:score-estimate}
Recall the discrete-time sequence $\bT$ defined in \eqref{e:defbT} and consider the integration domain $[0, T - \tau]$ for a small $\tau > 0$.
For the backward process, the time domain is discretized into $M$ steps:
\[
0 = t_0 < t_1 < \dots < t_M = T - \tau.
\]
with step sizes $H_i = t_{i+1} - t_i$ for every $0 \leq i \leq  M - 1$. We assume  
there exists a small parameter $\varepsilon_{\mathrm{score}} > 0$ such that the weighted $L^2$ score matching error satisfies
 \begin{align*}
		 \sum_{i=0} ^{M-1}H_i\cdot \sum_{t\in [t_i,t_{i+1}]\cap \bT}  \sigma_{T-t}^2 \int_{\bR^d}   \left\| s_{{t}}(x) -  \nabla \log\varrho_{{t}}(x)\right\|_2^2 \varrho_{{t}}(x) \rd x\leq \varepsilon_{\rm score} ^2 .
	\end{align*}
\end{assumption}

\begin{remark}
If each time interval contains a uniformly bounded number of points from $\bT$, especially if $\max_{0\leq i\leq M-1}H_i/\Delta t=\cO(1)$, then \Cref{a:score-estimate} is equivalent, up to a constant factor, to bounding the time-averaged $L^2$ score matching error
\begin{align*}
		 \sum_{t\in [0,T-\tau]\cap \bT} \Delta t\cdot \sigma_{T-t}^2 \int_{\bR^d}   \left\| s_{{t}}(x) -  \nabla \log\varrho_{{t}}(x)\right\|_2^2 \varrho_{{t}}(x) \rd x\leq \varepsilon_{\rm score} ^2 .
	\end{align*}
\end{remark}


Although for standard Runge-Kutta schemes \eqref{eq:RK-update} and \eqref{eq:RK-update-k} (as in \Cref{tab:butcher-specific}), the coefficients $[a_{jk}]$ and $[b_j]$ do not depend on the step size $H_i$, in the general case they may vary with $H_i$. This is particularly true for exponential Runge-Kutta schemes; see \Cref{s:GeneralProcess} for details.
We therefore allow the coefficients $[a_{jk}(H_i)]$ and $[b_j(H_i)]$ to depend on the step size $H_i$. The following assumption ensures that these coefficient functions, along with their derivatives, are appropriately bounded.

\begin{assumption}\label{a:bounds on RK}
For the $s$-stage, $p$-th order Runge-Kutta schemes \eqref{eq:RK-update} and \eqref{eq:RK-update-k}, we assume there exist constants $A_p, B_p \geq 1$ such that the coefficient functions $a_{jk}(h)$ and $b_j(h)$, along with their derivatives up to order $p+1$, satisfy the following bounds for all time steps $t_i$ and all $h \in [0,H_i]$:
\begin{align}\label{e:RKbound}
    \max_{1 \leq j,k \leq s} \max_{0 \leq l \leq p+1} \sigma_{\tau} ^{2l} \left| \partial_h^l a_{jk}(h) \right| \leq A_p, \quad 
    \max_{1 \leq j \leq s} \max_{0 \leq l \leq p+1} \sigma_{\tau} ^{2l} \left| \partial_h^l b_j(h) \right| \leq B_p.
\end{align}
For the standard Runge-Kutta schemes listed in \Cref{tab:butcher-specific}, \eqref{e:RKbound} holds trivially because the coefficients $[a_{jk}]$ and $[b_j]$ are constants.
Moreover, when applying the Runge-Kutta scheme, the time domain is discretized into $M$ steps. We assume that all time nodes $t_i$, as well as all intermediate evaluation points $t_i +c_jH_i$ (as in \eqref{eq:RK-update-k}), lie within the discrete time sequence $\bT$.
\end{assumption}

\begin{assumption}\label{a:score-derivative}
We assume that the approximate score $s_t(x)$ is $C^2$ with respect to $x$. More precisely, there is a positive constant $\widetilde K$, such that for any $t \in [0,T-\tau]\cap \bT$, 
    \begin{align*}
        \sup_{x \in \mathbb{R}^d} \max_{1\leq l, j \leq d} |\partial_{l} s_t ^{(j)}(x)| \leq \widetilde K \sigma_{\tau} ^{-4} , \quad \sup_{x \in \mathbb{R}^d} \max_{1\leq l,k, j \leq d} |\partial^2 _{lk} s_t ^{(j)}(x)| \leq \widetilde K \sigma_{\tau} ^{-6}  .
    \end{align*}
Here, we write $s_t(x) = (s_t ^{(1)}(x),s_t ^{(2)}(x),\ldots,s_t ^{(d)}(x))$, and we denote $\partial_l \coloneq \partial_{x_l}, \partial^2 _{lk} \coloneq \partial^2 _{x_l x_k}$.
We also assume that for any $t \in [0,T-\tau]\cap \bT$, there is a constant $\widetilde W$ such that  $\|s_t(x)\|_2 \leq \tau^{-1}\widetilde W(\sqrt{d}   +\|x\|_2) $.
\end{assumption}

\begin{remark}\label{remark:order of sigma_t in s_t}
    
    We remark that the assumptions on $s_t(x)$ in both \Cref{a:score-estimate} and \Cref{a:score-derivative} are of the correct order in $\sigma_\tau$. Indeed, according to \Cref{Lemma: Hessian estimates},
        \begin{align*}
            \|\nabla \log\varrho_{{t}}(x)\|_2 \leq 2 \sigma_{T-t} ^{-2} (\sqrt{d} D + \|x\|_2), \ \|\nabla^2 \log\varrho_{{t}}(x)\|_{\infty} \leq 3\sigma_{T-t}^{-4} D^2, \ \|\nabla^3 \log\varrho_{{t}}(x)\|_{\infty} \leq 24\sigma_{T-t}^{-6} D^3,
        \end{align*}
    and we know that $\sigma_{T-t} \geq \sigma_{\tau}$. Also, similar to~\cite{huang2024convergence}, we can replace the pointwise second-derivative assumption on $s_t$ with the integrated version
        \begin{align*}
            \max_{1\leq l, k, j \leq d} \int_{\mathbb{R}^d}|\partial^2 _{lk} s_t ^{(j)}(x)| \varrho_t(x) \de x \leq \widetilde K \sigma_{\tau} ^{-6} ,
        \end{align*}
    which suffices for the proof of \Cref{theorem: main L^1 theorem}. For simplicity we retain the stronger pointwise condition \Cref{a:score-derivative} and verify it numerically in \Cref{sec:num}.
\end{remark}

Our main result is stated below, with the detailed proof deferred to \cref{Section: Discrete ODE}.
\begin{theorem}\label{theorem: main L^1 theorem}
    Let $\dvarrho_{t}$ denote the distribution of the numerical solution $\dY_t$, produced by the $s$-stage, $p$-th order standard Runge–Kutta schemes in \eqref{eq:RK-update} and \eqref{eq:RK-update-k}.
    We assume \Cref{assumption:secon-moment,a:score-estimate,a:bounds on RK,a:score-derivative}. Then there exist constants $C_{\rm score} >0$, $C_{\rm RK}>0$, and $\Delta_{\rm disc} \in (0,1)$, depending on the constants $A_p, B_p, s, p, D,\widetilde K, \widetilde W$ in these assumptions, as well as exponents $\gamma_1,\gamma_2$ depending only on $p$, such that the following holds. Let $H_{\rm max} \coloneq \max_{0\leq i \leq M -1} (t_{i+1}-t_i) $ and assume that $H_{\rm max}\leq (d \log (2d))^{-1} \cdot \sigma_{\tau} ^{6} \cdot \Delta_{\rm disc}$. Then
    \begin{align}\label{e: main L^1 error}
        \begin{split}
        \TV(\varrho_{T-\tau}, \dvarrho_{T-\tau}) &\leq \TV(\varrho_{0}, \dvarrho_{0}) + C_{\rm score}  \sigma_{\tau} ^{-4}      T^{\frac{3}{4}}      d^{\frac{7}{4}}   \cdot \varepsilon_{\rm score} ^{\frac{1}{2}}   
         +  C_{\rm RK}  T H_{\mathrm{max}} ^p d^{p+1}  \sigma_{\tau} ^{-2\gamma_1} {\big(D+\sqrt{ \log d}\big)}^{\gamma_2}.
        \end{split}
    \end{align}
If the initialization is chosen as the standard normal distribution $ \dvarrho_0=\cN(0, \id_d)$, then $\TV(\varrho_0, \dvarrho_0)\leq C_u e^{-T} \sqrt d D$, which is exponentially small in $T$.
\end{theorem}

\subsection{Proof Mechanism Under $C^2$ Score Regularity}
\label{s:low-regularity-mechanism}

We emphasize the mechanism that allows \Cref{theorem: main L^1 theorem} to obtain a $p$-th order temporal error while assuming only two spatial derivatives of the learned score.
The proof does not apply classical $p$-th order local truncation error estimates directly to the learned score function.
Instead, on each interval $[t,t+H]$, we control the one-step error by comparing three pushforward maps: the exact flow map $\phi_h$ associated with \cref{eq:reverse-ode}, the oracle Runge-Kutta map $\baphi_h$, where the true score $V_t(x) = x + \nabla \log \varrho_t(x)$ is used, and the learned Runge-Kutta map $\dphi_h$ defined in \cref{eq:RK-update,eq:RK-update-k}.
\begin{table}[h]
\begin{center}
\begin{tabular}{@{}p{0.3\textwidth}p{0.45\textwidth}p{0.2\textwidth}@{}}
\toprule
map & meaning & regularity used \\
\midrule
exact flow map $\phi_h$ & $\partial_h \phi_h(x)=V_{t+h}(\phi_h(x)) \quad \phi_0(x)=x$ & diffeomorphism \\
\addlinespace
oracle Runge-Kutta \newline  map $\baphi_h$ & $\baphi_h(x)=x+h\sum_{j=1}^s b_j\overline k_j(x)  \newline
    \overline k_j(x)=V_{t+c_jH}\Bigl(x+\sum_{\ell<j}a_{j\ell}\overline k_\ell(x)H\Bigr)$  & true score, up to $p+1$-st derivatives\\
    \addlinespace
learned Runge-Kutta \newline  map $\dphi_h$ & $\dphi_h(x)=x+h\sum_{j=1}^s b_j\widetilde k_j(x)  \newline
    \widetilde k_j(x)=\sV_{t+c_jH}\Bigl(x+\sum_{\ell<j}a_{j\ell}\widetilde k_\ell(x)H\Bigr)$  & learned score, up to 2nd derivatives\\
\bottomrule
\end{tabular}
\caption{Three maps used in the one-step error decomposition.}
\label{tab:RK-maps}
\end{center}
\end{table}

The one-step error is decomposed as
\begin{equation*}
\begin{split}
    \|\dvarrho_{t+H}-\varrho_{t+H}\|_{L^1} &= \|(\dphi_H)_\#\dvarrho_t-(\phi_H)_\#\varrho_t\|_{L^1} \\
    &\leq \underbrace{\|(\dphi_H)_\#\dvarrho_t-(\dphi_H)_\#\varrho_t\|_{L^1}}_{I_1}+
    \underbrace{\|(\dphi_H)_\#\varrho_t-(\baphi_H)_\#\varrho_t\|_{L^1}}_{I_2} +
    \underbrace{\|(\baphi_H)_\#\varrho_t-(\phi_H)_\#\varrho_t\|_{L^1}}_{I_3}.
\end{split}   
\end{equation*}
Here, $I_1 = \|\dvarrho_t-\varrho_t\|_{L^1}$ (when $\dphi_H$ is a diffeomorphism) represents the incoming distributional error. The term $I_2$ compares the learned map $\dphi_H$ with the oracle map $\baphi_H$. Their difference arises from the use of the learned score instead of the true score, and hence  $I_2$ captures the effect of the score-matching error. The term $I_3$ compares the oracle Runge-Kutta map with the exact flow map and represents the numerical discretization error.
The term $I_3$ is the only place where the $p$-th order Runge-Kutta structure is used. It depends only on the true score, whose higher-order derivatives are controlled via the Gaussian convolution representation of the forward process; see \Cref{s:prel-estim-q}. This yields a one-step temporal error of order $H^{p+1}$.
By contrast, $I_2$ is the only term in which the learned score $s_t$ enters. Since $\dphi_H$ and $\baphi_H$ share the same Runge-Kutta structure, their difference is estimated by interpolating between the two pushforward flows. The resulting interpolation error is then controlled by the score-matching error in \Cref{a:score-estimate}, together with Gagliardo--Nirenberg interpolation estimates in \Cref{section:score estimation error}. The Gagliardo--Nirenberg inequality requires only a $C^2$ bound on the learned score, as stated in \Cref{a:score-derivative}.
In \Cref{s:RK2-main-proof}, we give a detailed proof for \Cref{theorem: main L^1 theorem} in the case when $p=2$.

\begin{remark}
The total variation bound in \cref{theorem: main L^1 theorem} simplifies to $\mathcal{O}( d^{7/4}\varepsilon_{\rm score}^{1/2} + d\cdot(dH_{\max})^p)$, where we treat quantities such as the total time $T$, the target distribution diameter  parameter $D$, and the stopping time parameter $\tau$ as constants. In the absence of score error ($\varepsilon_{\rm score}$), this bound further simplifies as $\mathcal{O}(d\cdot(dH_{\max})^p)$, which implies an iteration complexity of $\cO(d^{1+1/p} \varepsilon^{-1/p})$ to achieve a total variation accuracy of $\varepsilon$.

\end{remark}

\begin{remark}
  Although \eqref{e: main L^1 error} is established between the numerical solution $\dvarrho_{T-\tau}$ and reference density $\varrho_{T-\tau}$ at an earlier time, the Wasserstein 2-distance between $\varrho_{T-\tau}$ and the target distribution $\muast$, $W_2(\muast,\varrho_{T-\tau})$, tends to $0$ of order $\tau d$ as $\tau \to 0^+$ (see Remark 3.7 in \cite{huang2024convergence}). It is also standard that $W_2(\varrho_{T-\tau}, \dvarrho_{T-\tau})$ can be quantitatively estimated by $\TV(\varrho_{T-\tau}, \dvarrho_{T-\tau})$. We can use, for example, Proposition 7.10 in \cite{MR1964483} together with the facts like \Cref{lemma: qt close to Gaussian} and \Cref{lemma: moments of L^2 and infinity norms} to quantitatively estimate $W_2(\varrho_{T-\tau}, \dvarrho_{T-\tau})$ by $\TV(\varrho_{T-\tau}, \dvarrho_{T-\tau})$. Hence, we can also obtain a quantitative estimate for $W_2(\muast, \dvarrho_{T-\tau})$ by \Cref{theorem: main L^1 theorem}.
\end{remark}

\begin{remark}\label{rem:score_error}
For the score matching error term $C_{\rm score}  \sigma_{\tau} ^{-4}     T^{\frac{3}{4}}      d^{\frac{7}{4}}   \cdot \varepsilon_{\rm score} ^{\frac{1}{2}}$ in \eqref{e: main L^1 error}, as discussed in Remark 3.6 in our previous work in \cite{huang2024convergence}, if we make assumptions on higher order derivatives of $s_t(x)$ like \Cref{a:score-derivative}, i.e., $s_t(x)$ is $C^k$ with controlled $C^k$-norms, we can improve $(\varepsilon_{\rm score}) ^{\frac{1}{2}}$ to $(\varepsilon_{\rm score}) ^{1-\frac{1}{k}}$. 
In \Cref{sec:num}, we numerically explore the sharpness of this exponent, and we find that the total variation error $\TV(\varrho_{T-\tau}, \dvarrho_{T-\tau})$ scales linearly with $\varepsilon_{\rm score}$, with a potentially more  favorable dependence on the dimension $d$ for the Gaussian mixture target density.
\end{remark}

\begin{remark}
    The technical assumption $H_{\rm max}\leq (d \log (2d))^{-1} \cdot \sigma_{\tau} ^{6} \cdot \Delta_{\rm disc}$ in \Cref{theorem: main L^1 theorem} is to ensure that the maps in the table above, when viewed as pushforward maps, are invertible. This invertibility is essential for estimating the distribution of $\dY_{t_{i+1}}$ from the distribution of $\dY_{t_i}$ via the map \eqref{eq:RK-update} reformulated as \eqref{e:phi_h_define}. 
\end{remark}

\begin{remark}
The powers of the stopping parameter $\tau$ in \eqref{e: main L^1 error} and in the small-step condition are not optimized. They arise from a worst-case treatment of the terminal-time singularity: the proof uses first- and second-derivative bounds on the learned score, the corresponding Gaussian-smoothing bounds for the true score, and an invertibility argument for the Runge-Kutta pushforward maps. These estimates lead to polynomial factors such as $\sigma_\tau^{-4}$ and $\sigma_\tau^{-6}$ in the intermediate bounds. Previous analyses obtain milder dependence using time-localized estimates and endpoint-adaptive noise schedules~\cite{chen2023improved,benton2023linear,li2023towards,li2024sharp}, albeit under different settings or additional assumptions. Our main focus is the high-order dependence on $H_{\max}$ under only $C^2$ regularity of the learned score. Sharper dependence on $\sigma_\tau$ is left for future work.
\end{remark}

\begin{remark}
Compared to several previous studies on probability flow ODEs \cite{li2023towards, li2024unified, li2024sharp, li2024accelerating}, which require control over the difference between the derivatives of the true and estimated scores, our result in \Cref{a:score-derivative} only requires the boundedness of the first two derivatives of the estimated score. In \Cref{sec:num}, we numerically validate \Cref{a:score-derivative}, showing that both quantities remain well-bounded in practice, which supports our theoretical assumptions.
\end{remark}

In the next section, we extend \Cref{theorem: main L^1 theorem} to (exponential) Runge–Kutta schemes applied to arbitrary forward processes with general variance schedules.

\section{Results for General Forward Process}
\label{s:GeneralProcess}
\subsection{General Forward Process}
In practical settings \cite{ho2020denoising,song2020score}, the following variance preserving stochastic process is commonly used: 
\begin{align}
\label{eq:OU-time}
	\de X_t = -\frac{1}{2}\beta_t X_t \de t + \sqrt{\beta_t} \de B_t, \qquad X_0 \sim \muast, \qquad 0 \leq t \leq T, 
\end{align}
where $\beta_t:\bR\rightarrow\bR$ is the variance schedule that controls the level of noise injected at each time step.
Similar to the OU process, this process \eqref{eq:OU-time} admits an analytical solution:
\begin{align*}
	X_t \overset{d}{=} \lambda_t X_0 + \sigma_t W, \qquad W \sim \cN(0, \id_d), 
\end{align*}
where  $\lambda_t$ and $\sigma_t$ satisfy
\begin{align*}
\frac{\dd \log \lambda_t}{\dd t} = -\frac{1}{2}\beta_t \qquad \textrm{and} \qquad \frac{\dd \sigma_t^2}{\dd t} - 2\frac{\dd \log \lambda_t}{\dd t } \sigma_t^2 = \beta_t.
\end{align*}
These yield the explicit expressions:
\begin{align}\label{e:general_sigma_t}
\lambda_t = e^{-\frac{1}{2} \int_0^t \beta_s \de s} \qquad \textrm{and} \qquad  \sigma_t  = \sqrt{1 - \lambda_t^2}.
\end{align}
The parameters play a key role in the score matching procedure described in~\cref{s:score_matching}.
A widely adopted choice for $\beta_t$ is the linear variance schedule: $\beta_t = \beta_{\rm min} + \frac{t}{T} (\beta_{\rm max} - \beta_{\rm min})$, where $\beta_{\rm min}$ and $\beta_{\rm max}$ represent the minimum and maximum noise levels. Under this schedule, the decay factor becomes $\lambda_t = e^{-\frac{1}{2}\Bigl(t \beta_{\rm min} + \frac{t^2}{2T}(\beta_{\rm max} - \beta_{\rm min})\Bigr)}$. The OU process \eqref{eq:OU} is a special case with $\beta_t = 2$.

The density $q_t$ of $X_t$ as in \eqref{eq:OU-time} satisfies
\begin{align*}
    \del_t q_t =\frac{\beta_t}{2} \nabla \cdot \left((x+\nabla \log q_t)q_t\right)=\frac{\beta_t}{2}  \nabla \cdot (\qU_t q_t), \quad 0 <  t\leq T,
\end{align*}
with $\qU_t(x) \coloneq x +  \nabla \log q_t(x)$. 
By denoting $\varrho_t  = q_{T - t}$, the time reversal process from time $0$ to $T$, satisfies the following partial differential equation (PDE):
\begin{align}\begin{split}
\label{e:reversep_general}
    \del_t \varrho_t =-\frac{\beta_{T-t}}{2}\nabla \cdot \bigl((x + \nabla \log q_{T - t}) \varrho_t\bigr) = -\frac{\beta_{T-t}}{2} \nabla \cdot (V_t \varrho_t), \quad 0 \leq t < T,
\end{split}\end{align}
with $V_t(x) \coloneq x +  \nabla \log q_{T - t}(x)$.

Then we can reformulate the reverse  PDE \eqref{e:reversep_general} into 
a deterministic dynamics in terms of an ordinary differential equation with velocity field $V_t$:
\begin{align}\label{eq:reverse-ode_general}
	\del_t Y_t = \frac{\beta_{T-t}}{2}\left( Y_t +  \nabla \log q_{T - t} (Y_t) \right) = \frac{\beta_{T-t}}{2} V_t(Y_t), \qquad Y_0 \sim q_T, \qquad 0 \leq t < T.
\end{align}
When $\nabla \log q_{T - t} (x)$, is approximated by the learned score function $s_t(x)$, the probability flow ODE \eqref{eq:reverse-ode_general} becomes 
\begin{align}\label{eq:reverse-ode-score_general}
	\del_t \sY_t = \frac{\beta_{T-t}}{2}\left( \sY_t +  s_t(\sY_t)\right) = \frac{\beta_{T-t}}{2} \sV_t(\sY_t), \qquad \sY_0 \sim \svarrho_0, 
\end{align}
where the velocity field becomes $\sV_t(x) \coloneq x + s_t(x)$.

\subsection{ODE Solver for General Forward Process}\label{s:general}
To solve the differential equation \eqref{eq:reverse-ode-score_general}, one can again apply the standard Runge-Kutta schemes introduced in \cref{s:time_integrator}.
\begin{align}
\label{eq:RK-update-beta}
\dY_{t_{i+1}} = \dY_{t_i} +  H_i \sum_{j=1}^s  b_j  \wt k_j(\dY_{t_i}),
\end{align}
where $ H_i=t_{i+1}-t_i$, and
\begin{equation}
\label{eq:RK-update-beta-k}
\begin{split}
    & \wt k_1(x) =\frac{\beta_{T-t_i - c_1  H_i}}{2} \sV_{t_i +  c_1  H_i}(x),\\
    & \wt k_2(x) = \frac{\beta_{T-t_i -  c_2  H_i}}{2}\sV_{t_i +  c_2  H_i}\bigl(x + ( a_{21} \wt k_1(x)) H_i\bigr),\\
    & \wt k_3(x) = \frac{\beta_{T-t_i - c_3  H_i}}{2}\sV_{t_i + c_3 H_i}\bigl(x + ( a_{31} \wt k_1(x) +  a_{32} \wt k_2(x))  H_i\bigr),\\
    &\qquad\qquad\qquad\vdots
    \\
   & \wt k_s(x) = \frac{\beta_{T-t_i -  c_s  H_i}}{2}\sV_{t_i +  c_s H_i}\bigl(x + ( a_{s1}\wt k_1(x) +  a_{s2} \wt k_2(x) + \cdots +  a_{s,s-1} \wt k_{s-1}(x))  H_i\bigr).
\end{split}
\end{equation}

However, standard Runge-Kutta schemes do not leverage the natural separation between the linear and nonlinear terms in the probability flow ODE~\eqref{eq:reverse-ode-score_general}. Exploiting this structure by integrating the linear term analytically is the central idea behind exponential Runge–Kutta schemes~\cite{cox2002exponential,hochbruck2005explicit,hochbruck2010exponential}, which yields more robust time integrators. This approach has been adopted in recent works~\cite{zhang2022fast, lu2022dpm} for solving probability flow ODEs and has demonstrated improved performance.  In practice, applying exponential Runge–Kutta schemes in this setting requires several important modifications.
First, classical exponential Runge–Kutta schemes~\cite{cox2002exponential,hochbruck2005explicit,hochbruck2010exponential} are typically designed for time-independent linear terms. In our case~\eqref{eq:reverse-ode-score_general}, the linear term is time-dependent due to the variance schedule, necessitating additional adjustments to handle this dependency. 
Second, to mitigate the potential singularity of the score function as $t\to T$, the reparameterization $\rscore_t(x) \coloneq s_t(x)\,\sigma_{T - t}$, used in score matching, is adopted.
The rescaled quantity $\rscore_t$ is typically smoother than $s_t$, especially near the terminal time, making it a more robust target for deriving order conditions.
Finally, all time steps, including intermediate stages,  should lie within the discrete time sequence $\bT$ used in score matching. 

To derive the enhanced exponential Runge-Kutta schemes for solving \eqref{eq:reverse-ode-score_general}, we begin by introducing a time change variable
\begin{align}\label{eq:t-zeta change of variable2}
\zeta(t)= \frac{1}{2}\int_{T-t}^{T} \beta_s \, \rd s ,
\end{align}
and we rewrite \eqref{eq:reverse-ode-score_general} as $
 \del_t \bigl(e^{-\zeta(t)}\sY_t\bigr) = e^{-\zeta(t)} \frac{\beta_{T-t}}{2} s_t(\sY_t).$ 
 This leads to the corresponding integral form for $\sY_{t}$:
\begin{align} 
\label{eq:exp-rk-intg}
     \sY_{t_{i+1}} &= e^{\zeta(t_{i+1}) - \zeta(t_i)}\sY_{t_{i}} + \int_{t_i}^{t_{i+1}} \frac{\beta_{T-t}}{2} e^{\zeta(t_{i+1})-\zeta(t)} s_t(\sY_t) \rd t \nonumber \\
     &= e^{\zeta(t_{i+1}) - \zeta(t_i)}\sY_{t_{i}} - \sigma_{T - t_{i+1}}e^{\alpha_{t_{i+1}}}\int_{t_i}^{t_{i+1}} \frac{\rd}{\rd t} e^{-\alpha_{t}} \rscore_t(\sY_t) \rd t,
\end{align}
where, in the last equation, we introduce the reparameterized score $\rscore_t(x) \coloneq \sigma_{T-t} s_{t}(x)$ and apply the change of variables from \cite{lu2022dpm}, defined by
\begin{equation}
\label{eq:t-alpha change of variable}
    \alpha_{t} = \zeta(t) - \log\sigma_{T-t}, \quad \textrm{satisfying} \quad \frac{\rd \alpha_t}{\rd t} = -\frac{\beta_{T-t}}{2\sigma_{T - t}^2}, 
    \quad \frac{\rd}{\rd t} e^{-\alpha_{t} } = -\frac{\beta_{T-t}}{2\sigma_{T - t}} e^{-\zeta(t)}.
\end{equation} 
The enhanced exponential Runge-Kutta schemes update from $t_i$ to $t_{i+1}$ with step size $H_i=t_{i+1}-t_i$ according to (see detailed derivations in \cref{section:rk scheme})
\begin{align}
\label{eq:Exp-RK-update-beta}
\dY_{t_{i+1}} = e^{\zeta(t_{i+1}) - \zeta(t_i)}\dY_{t_i} + H_i \sum_{j=1}^s b_j(H_i) k_j(\dY_{t_i}),
\end{align}
where the stage evaluations are defined using the rescaled score $\rscore_t(x) = s_t(x)\,\sigma_{T - t}$ as follows:
\begin{equation}
\label{eq:Exp-RK-update-beta-k}
\begin{split}
    &k_1(x) = \rscore_{t_i + c_1 H_i}(e^{\zeta(t_{i} + c_1H_i) - \zeta(t_i)}x),\\
    &k_2(x) = \rscore_{t_i + c_2 H_i}\bigl(e^{\zeta(t_{i} + c_2H_i) - \zeta(t_i)}x + (a_{21}(H_i) k_1(x)) H_i\bigr),\\
    &k_3(x) = \rscore_{t_i + c_3 H_i}\bigl(e^{\zeta(t_{i} + c_3H_i) - \zeta(t_i)}x + (a_{31}(H_i) k_1(x) + a_{32} (H_i)k_2(x)) H_i\bigr),\\
    &\qquad\qquad\qquad\vdots
    \\
    &k_s(x) = \rscore_{t_i + c_s H_i}\bigl(e^{\zeta(t_{i} + c_sH_i) - \zeta(t_i)}x + (a_{s1}(H_i) k_1(x) + a_{s2}(H_i) k_2(x) + \cdots + a_{s,s-1}(H_i)k_{s-1}(x)) H_i\bigr).
\end{split}
\end{equation}
As in standard Runge-Kutta schemes, the lower triangular matrix $[a_{jk}]$, along with the coefficients $[b_j]$ and $[c_j]$, represent the Runge-Kutta matrix, weights and nodes. Typically, the nodes $c_j$ are constants, but the coefficients $a_{jk}(H_i)$ and $b_j(H_i)$ are functions of the step size $H_i$. The detailed coefficients for schemes up to third order are provided in \cref{tab:exp-butcher-specific}. Note that a fourth-order scheme requires at least five stages \cite{hochbruck2010exponential}, which we do not include here.
Exponential-type Runge-Kutta schemes of up to third order are widely used for solving the probability flow ODE, forming the backbone of the DPM solver~\cite{lu2022dpm}. However, in \cite{lu2022dpm}, the evaluation time points do not necessarily lie within the discrete time sequence $\bT$. In contrast, our construction~\cref{tab:exp-butcher-specific} leaves the nodes $c_j$ as free parameters. To make the scheme compatible with the score-matching grid, we choose rational nodes; if their common denominator is $m$, then taking $H_i$ to be a multiple of $m\Delta t$ guarantees that all evaluation times $t_i+c_jH_i$ lie within $\bT$.

\begin{table}[h]

\centering
\small
\caption{Butcher tableaus for the enhanced exponential Runge--Kutta schemes. Here $\varphi_k(h)=\int_0^1 \frac{x^{k-1}}{(k-1)!}e^{h(1-x)}\,\rd x$; in particular, $\varphi_1(h)=\frac{e^h-1}{h}$ and $\varphi_2(h)=\frac{e^h-h-1}{h^2}$.}

\subcaptionbox{First- and second-order schemes, with $b_2 = \sigma_{T - t_{i+1}}\frac{(\alpha_{t_{i+1}} - \alpha_{t_i})^2}{(\alpha_{t_{i}+c_2H_i} - \alpha_{t_i})H_i} \varphi_2(\alpha_{t_{i+1}} - \alpha_{t_i})$. The stage node $c_2$ is a free parameter; we set $c_2=1$ in the numerical experiments.}[0.9\textwidth]{
\begin{tabular}{c|c}
$0$ & $0$  \\
\hline
&  $\sigma_{T - t_{i+1}} \frac{\alpha_{t_{i+1}} - \alpha_{t_i}}{H_i} \varphi_1(\alpha_{t_{i+1}} - \alpha_{t_i})$
\end{tabular}
\qquad
\begin{tabular}{c|cc}
$0$ & $0$ & $0$ \\
$c_2$ & $\sigma_{T - t_{i} - c_2H_i} \frac{\alpha_{t_{i}+ c_2H_{i}} - \alpha_{t_i}}{H_i}\varphi_1\bigl(\alpha_{t_{i} + c_2H_{i}} - \alpha_{t_i}\bigr)$ & $0$ \\
\hline
&  $\sigma_{T - t_{i+1}}\frac{\alpha_{t_{i+1}} - \alpha_{t_i}}{H_i}\varphi_1\bigl(\alpha_{t_{i+1}} - \alpha_{t_i}\bigr) - b_2$ & $b_2$  
\end{tabular}
}
\\[2em] 
\subcaptionbox{Third-order scheme, with \footnotesize $\gamma = \frac{3c_3^2 - 2c_3}{3c_2^2-2c_2}$, $b_3 = \frac{\sigma_{T - t_{i+1}}
     (\alpha_{t_{i+1}} - \alpha_{t_i})^2\varphi_2(\alpha_{t_{i+1}} - \alpha_{t_i})
      }{\bigl(\gamma(\alpha_{t_{i}+ c_2H_{i}} - \alpha_{t_{i}}) + (\alpha_{t_{i}+ c_3H_{i}} - \alpha_{t_{i}})\bigr)H_i}$, 
      and  $a_{32} = \frac{\gamma\sigma_{T - t_{i} - c_2H_i}(\alpha_{t_{i}+ c_2H_{i}} - \alpha_{t_i}\bigr)^2\varphi_2\bigl(\alpha_{t_{i} + c_2H_{i}} - \alpha_{t_i}\bigr) + \sigma_{T - t_{i} - c_3H_i}\bigl(\alpha_{t_{i}+ c_3H_{i}} - \alpha_{t_i}\bigr)^2\varphi_2\bigl(\alpha_{t_{i} + c_3H_{i}} - \alpha_{t_i}\bigr)}{(\alpha_{t_{i}+ c_2H_{i}} - \alpha_{t_{i}})H_i}$. The stage nodes $c_2$ and $c_3$ are free parameters; we set $c_2=1/3$ and $c_3=2/3$ in the numerical experiments.}[0.9\textwidth]{
\begin{tabular}{c|ccc}
$0$ & $0$ & $0$   & $0$ \\
$c_2$ & $\sigma_{T - t_{i} - c_2H_i}\frac{\alpha_{t_{i}+ c_2H_{i}} - \alpha_{t_i}}{H_i}\varphi_1\bigl(\alpha_{t_{i} + c_2H_{i}} - \alpha_{t_i}\bigr)$ & $0$ & $0$    \\
$c_3$ & $\sigma_{T - t_{i} - c_3H_i}\frac{\alpha_{t_{i}+ c_3H_{i}} - \alpha_{t_i}}{H_i}\varphi_1\bigl(\alpha_{t_{i} + c_3H_{i}} - \alpha_{t_i}\bigr) - a_{32}$ & $a_{32}$ &  $0$  \\
\hline
       & $\sigma_{T - t_{i+1}}\frac{\alpha_{t_{i+1}} - \alpha_{t_i}}{H_i}\varphi_1\bigl(\alpha_{t_{i+1}} - \alpha_{t_i}\bigr) - \gamma b_3 - b_3$  & $\gamma b_3$  & $b_3$
\end{tabular}
}
\label{tab:exp-butcher-specific}

\end{table}

\subsection{Convergence Analysis} 
In this subsection, we analyze the convergence of the aforementioned (exponential) Runge–Kutta schemes applied to arbitrary forward processes with general variance schedules. 
We begin by presenting the necessary assumptions, followed by the main convergence result.

\begin{assumption}\label{a:bounds on beta_t}
There exists a large constant $C_\beta>0$, such that uniformly for $0\leq t\leq T$, $\beta_t$ is lower and upper bounded, and its derivatives are upper bounded: 
    \begin{align}\label{e:betatbound}
       C_\beta^{-1}\leq \beta_t\leq C_\beta,\quad  |\del_t^\ell \beta_t|\leq C_{\beta},\quad  1\leq \ell\leq p+1.
    \end{align}

\end{assumption}

\begin{theorem}\label{theorem: main L^1 theorem general}
    Let $\dvarrho_{t}$ denote the distribution of the numerical solution $\dY_t$, produced by the $s$-stage, $p$-th order Runge–Kutta schemes in \eqref{eq:RK-update-beta} or in  \eqref{eq:Exp-RK-update-beta}. We assume \Cref{assumption:secon-moment,a:score-estimate,a:bounds on RK,a:bounds on beta_t,a:score-derivative}. Then there exist constants  $C_{\rm score} >0$, $C_{\rm RK}>0$,  $\Delta_{\rm disc} \in (0,1)$, depending on constants $C_\beta, s, p, D,\widetilde K, \widetilde W$ in these assumptions, as well as exponents $\gamma_1,\gamma_2$ depending only on $p$, such that the following holds. Let $H_{\rm max} \coloneq \max_{0\leq i \leq M -1} (t_{i+1}-t_i) $ and assume that $H_{\rm max}\leq (d\log (2d))^{-1} \cdot \sigma_{\tau} ^{6} \cdot \Delta_{\rm disc}$. Then 
    \begin{align}
        \begin{split}
        \TV(\varrho_{T-\tau}, \dvarrho_{T-\tau}) &\leq \TV(\varrho_{0}, \dvarrho_{0}) + C_{\rm score}  \sigma_{\tau} ^{-4} T^{\frac{3}{4}}      d^{\frac{7}{4}}   \cdot \varepsilon_{\rm score} ^{\frac{1}{2}}
        +  C_{\rm RK}  T H_{\mathrm{max}} ^p d^{p+1} \sigma_{\tau} ^{-2\gamma_1}   {\big(D+\sqrt{ \log d}\big)}^{\gamma_2}.
        \end{split}
    \end{align}
If the initialization is chosen as the standard normal distribution $\dvarrho_0=\cN(0, \id_d)$, then $\TV(\varrho_0, \dvarrho_0)\leq C_u e^{-T} \sqrt d D$, which is exponentially small in $T$.
\end{theorem}

The proof is a reduction to \Cref{theorem: main L^1 theorem}. 
The key step is to rewrite both the standard Runge-Kutta schemes \eqref{eq:RK-update-beta} and the enhanced exponential Runge-Kutta schemes \eqref{eq:Exp-RK-update-beta} in the discrete-flow form \eqref{eq:RK-update}--\eqref{eq:RK-update-k}. 
After this rewriting, the stage evaluations involve $s_t(x)$, as in \eqref{eq:Exp-RK-update-beta-k-rf}, rather than $\sV_t(x)=x+s_t(x)$, because the linear part of the general probability flow is handled analytically.
Both $s_t(x)$ and $\sV_t(x)$ satisfy the same learned-score regularity assumptions, and the forward density for the general process satisfies the same derivative and moment estimates used in \Cref{s:prel-estim-q}. 
The remaining point is to check that the rewritten coefficients satisfy \Cref{a:bounds on RK}; this is the content of \Cref{lem:coefficients}. 
The proof of \Cref{theorem: main L^1 theorem general} is then given at the end of \Cref{Section: Discrete ODE}.


\section{Discrete ODE Flows and Main Proofs}
\label{Section: Discrete ODE}

In this section, we prove \Cref{theorem: main L^1 theorem} for the $p$-th order standard Runge-Kutta schemes introduced in \eqref{eq:RK-update} and \eqref{eq:RK-update-k}. The key ingredient is a one-step error estimate in \Cref{t:RK_1step}, whose mechanism was described in \Cref{s:low-regularity-mechanism}. Summing this one-step estimate over the time partition yields \Cref{theorem: main L^1 theorem}.

To make the low-regularity mechanism transparent, we first present the detailed one-step proof for the second-order Runge-Kutta scheme in \Cref{tab:butcher-specific}. This argument already contains the essential idea: the error introduced by the learned score is controlled by the score-matching error and requires only the $C^2$ bounds in \Cref{a:score-derivative}, whereas the higher-order temporal consistency is established only for the oracle Runge-Kutta scheme driven by the true score. The extension from the second-order case to general order $p$ involves standard, although lengthy, bookkeeping of order conditions and derivative bounds; these details are collected in the appendix.

At the end of the section, we sketch the proof of \Cref{theorem: main L^1 theorem general} for the general forward process, using the standard and exponential Runge--Kutta schemes introduced in \eqref{eq:RK-update-beta} and \eqref{eq:Exp-RK-update-beta}.



\begin{theorem}\label{t:RK_1step}
Consider the $s$-stage, $p$-th order Runge–Kutta schemes in \eqref{eq:RK-update} and \eqref{eq:RK-update-k},
and assume \Cref{assumption:secon-moment,a:bounds on RK,a:score-derivative}. Fix $t=t_i$ and set $H=t_{i+1}-t_i$ with $t+H\leq T-\tau$. We define the 
 $L^2$ score matching error at stage times $\{t+c_1H, t+c_2 H, t+c_3 H,\cdots, t+c_s H\}$ by
\begin{align}\label{e:score_bound}
		 \varepsilon_{\rm score} ^2(t) \coloneq \sum_{j=1}^s \sigma_{T-(t+c_j H)}^2 \int_{\bR^d}   \left\| s_{{t+c_j H}}(x) -  \nabla \log\varrho_{{t+c_j H}}(x)\right\|_2^2 \varrho_{{t+c_j H}}(x) \rd x.
\end{align}
Here $\varrho_t$ denotes the law of the exact solution $Y_t$ generated by the flow map \eqref{eq:reverse-ode}.
Then there exists a constant $\Delta_{\rm disc} \in (0,1)$ depending on the constants $A_p, B_p, s, p, D,\widetilde K, \widetilde W$ in these assumptions, such that the following holds. If  $H \leq (d\log (2d))^{-1} \cdot \sigma_{\tau} ^6  \cdot \Delta_{\rm disc}$,
then the one-step update satisfies:
 \begin{align}\label{eq:main local time estimate}
        \begin{split}
        &\int_{\mathbb{R}^d} |\dvarrho_{t+H}(x)-\varrho_{t+H}(x)| \de x
        \leq\int_{\mathbb{R}^d} |\dvarrho_{t}(x)-\varrho_{t}(x)| \de x+ \cE_{\rm score}+\cE_{\rm RK} ,
        \end{split}
    \end{align}
    Here $\dvarrho_t$ denotes the law of the numerical solution $\widetilde Y_t$ generated by the learned Runge--Kutta scheme \eqref{eq:RK-update}-\eqref{eq:RK-update-k}.
   The error term $\cE_{\rm score}$ comes from the score approximation error in \eqref{e:score_bound}, whereas $\cE_{\rm RK}$ is the local discretization error of the $p$-th order Runge--Kutta scheme on the interval $[t,t+H]$. More precisely, by \Cref{Prop: Conclusion second error term,Prop: Conclusion third error term}, \begin{align}\label{e:two_error} \cE_{\rm score} &= C_{\rm score}\,\sigma_{\tau} ^{-4} H d^{\frac{7}{4}} \bigl(\varepsilon_{\rm score}(t)\bigr)^{\frac12}, \quad \cE_{\rm RK} = C_{\rm RK}\,H^{p+1}d^{p+1} \sigma_{\tau} ^{-2\gamma_1} \bigl(D+\sqrt{\log d}\bigr)^{\gamma_2}. \end{align} The constants $C_{\rm score}>0$ and $C_{\rm RK}>0$ depend on $A_p,B_p,s,p,D,\widetilde K,\widetilde W$, while the exponents $\gamma_1,\gamma_2$ depend only on $p$.
\end{theorem}


\subsection{One-Step Maps: Exact Flow, Oracle Runge-Kutta, and Learned Runge-Kutta}\label{s:p-th-RK}
In the remaining subsections, we adopt notations from \Cref{t:RK_1step}, namely $t=t_i$ and $H=t_{i+1}-t_i$. 
For one step of the $p$-th order Runge-Kutta scheme in \eqref{eq:RK-update} and \eqref{eq:RK-update-k}, we can write it in the form
\begin{align}\label{e:deftMt}
    \dY_{t+H}=\dY_t +H \dM_t(\widetilde Y_t;H),\quad \dM_t(x;H)=\sum_{j=1}^s b_j  \widetilde k_j(x).
\end{align}
Here $\{b_j, \widetilde k_j(x)\}_{1\leq j\leq s}$ are defined in \eqref{eq:RK-update-k}. The function $\dM_t(x;H)$ depends on the learned score at stage times $\{t+c_1 H, t+c_2 H,\cdots, t+c_s H\}$ and on coefficients $a_{jk}$ of the $p$-th order Runge-Kutta method. We remark that in the definition of $\dM_t(x;H)$ by \eqref{eq:RK-update} and \eqref{eq:RK-update-k}, because $H=t_{i+1}-t_i$ is fixed in \Cref{t:RK_1step}, we can view the coefficients $b_j = b_j(H)$ and $a_{jk} = a_{jk}(H)$ as constants. 

The learned Runge-Kutta map is the straight-line interpolation between the identity and the numerical update: for $0\leq h\leq H$ 
\begin{align}\label{e:phi_h_define}
  \dphi_h(x) \coloneq x + h \dM_t(x;H).
\end{align}
Let us first assume that for all $0\leq h\leq H$, $\dphi_h$ is invertible, which will be verified later in \Cref{lemma: phi_h invertible}. With the above notations \eqref{e:phi_h_define}, we can rewrite \eqref{e:deftMt} as $\dY_{t+H}=\dphi_H(\dY_t)$.
Given the law $\dvarrho_{t}$ of $\dY_t$,  we then define $\dvarrho_{t+h}$ as the pushforward law by the map $\dphi_h$:
    \begin{align}\label{e:trhoth}
        \dvarrho_{t+h} (x) = (\dphi_h)_\#\dvarrho_{t} (x)  \coloneq \dvarrho_{t} (\dphi_h ^{-1}(x)) \cdot | \det[ \nabla (\dphi_h ^{-1}(x))]|.
    \end{align}
Then the law of $\dY_{t+H}=\dphi_H(\dY_t)$ is $ \dvarrho_{t+H} (x)$.

We denote by $\phi_h(x)$ the exact flow map associated with the continuous dynamics \eqref{eq:reverse-ode},
\begin{align}
\label{e:exact-flow-map}
    \del_h \phi_h(x)=V_{t+h}(\phi_h(x))  \quad \textrm{with} \quad \phi_{0}(x) = x \in \mathbb{R}^d.
\end{align} 
Given the law $\varrho_t$ of $Y_t$,  the law $\varrho_{t+h}$ of $Y_{t+h}$ is the pushforward law by the map $\phi_h$:
    \begin{align}\label{e:rhoth}
    \varrho_{t+h} (x) = (\phi_h)_\# \varrho_{t} (x) \coloneq \varrho_{t} (\phi_h ^{-1}(x)) \cdot|  \det[ \nabla (\phi_h ^{-1}(x))]|.
    \end{align}
In particular, the law of $Y_{t+H}$ is given by $\varrho_{t+H} (x) = \varrho_{t} (\phi_H ^{-1}(x)) \cdot  |\det[ \nabla (\phi_H ^{-1}(x))]|$. Then we can formally define $M_t (Y_t;H)$, so that the solution $Y_{t+H}$ of the continuous flow \eqref{eq:reverse-ode} can also be written in the form of \eqref{e:deftMt}
\begin{align*}
    Y_{t+H}=Y_t+H M_t (Y_t;H),\quad M_t(x;H)=\frac{\phi_H(x)-x}{H}.
\end{align*}

Analogously to \eqref{e:deftMt}, we introduce  $\baM_t(x;H)$, which is defined in the same way as $\dM_t(x;H)$, but using  $V_t(x)=x+\nabla \log \varrho_t(x)$ from the true score instead of  $\sV_t(x)$. We let 
\begin{align}\label{e: score error p RK method}
\baM_{t}(x;H) \coloneq \sum_{j=1}^s b_j \overline k_j(x) , 
\end{align}
with 
\begin{equation}\label{e: p RK interpolation terms}
\begin{split}
    &\overline k_1(x) = V_{t + c_1 H}\bigl(x\bigr),\\
    & \overline k_2(x)  = V_{t + c_2 H}\bigl(x + (a_{21 } \overline k_1(x) ) H\bigr),\\
    & \overline k_3 (x) = V_{t + c_3 H}\bigl(x + (a_{31} \overline k_1(x)  + a_{32} \overline k_2(x) ) H\bigr),\\
    &\qquad\qquad\qquad\vdots
    \\
    &\overline k_s(x)  = V_{t + c_s H}\bigl(x + (a_{s1} \overline k_1(x)  + a_{s2} \overline k_2(x)  + \cdots + a_{s,s-1} \overline k_{s-1}(x) ) H\bigr).
\end{split}
\end{equation}
In the above \eqref{e: score error p RK method} and \eqref{e: p RK interpolation terms}, the coefficients $b_j = b_j(H)$ and $a_{jk} = a_{jk}(H)$ are the same as those we used to define $\dM_t(x;H)$ in \eqref{e:deftMt}.
We can similarly define the oracle Runge-Kutta map: 
    \begin{align}
    \label{e:oracle-RK-flow-map}
        \baphi_h(x) \coloneq x + h \baM_t(x;H).
    \end{align}
By \Cref{lemma: phi_h invertible}, let us also assume that for all $0\leq h\leq H$, $\baphi_h$ is invertible. Similar to \eqref{e:rhoth}, given the law $\varrho_t$ of $Y_t$,  we then define $\overline \varrho_{t+H}$ as the pushforward law by the map $\baphi_H$:
    \begin{align}\label{e:brhoth}
      \overline \varrho_{t+H} (x)  = (\baphi_H)_\# \varrho_{t} (x)\coloneq \varrho_{t} (\baphi_H ^{-1}(x)) \cdot | \det[ \nabla (\baphi_H ^{-1}(x))]|.
    \end{align}

\subsection{Proof Sketch of \cref{t:RK_1step}: The Second-Order RK Case}\label{s:RK2-main-proof}
We now spell out the one-step argument for the second-order ($p=2$) scheme in \Cref{tab:butcher-specific}, namely
\[
    c_1=0,\qquad c_2=1,\qquad a_{21}=1,\qquad b_1=b_2=\frac12 .
\]
This is the first genuinely higher-order case, and it already shows why the learned score needs only two spatial derivatives.
The learned map~\eqref{e:phi_h_define} and the  oracle map~\eqref{e:oracle-RK-flow-map} are explicit:
\begin{align*}
    \dphi_H(x)
    &=x+\frac{H}{2}\sV_t(x)
      +\frac{H}{2}\sV_{t+H}\bigl(x+H\sV_t(x)\bigr)\quad
    \baphi_H(x)
    &=x+\frac{H}{2}V_t(x)
      +\frac{H}{2}V_{t+H}\bigl(x+HV_t(x)\bigr).
\end{align*}
The exact flow map $\phi_H$ is defined in \cref{e:exact-flow-map}.
For $H$ satisfying the small-step condition, these maps are diffeomorphisms; this follows from the growth estimates in \Cref{p:ABCD} and the inverse-map argument in \Cref{lemma: phi_h invertible}.
Using the pushforward formulas for $\dphi_H$, $\baphi_H$, and $\phi_H$, we decompose
\begin{align}\label{e:RK2-I123}
    \|\dvarrho_{t+H}-\varrho_{t+H}\|_{L^1}
    \leq&  \underbrace{\|(\dphi_H)_\#\dvarrho_t-(\dphi_H)_\#\varrho_t\|_{L^1}}_{I_1}
     \nonumber\\
    &+
    \underbrace{\|(\dphi_H)_\#\varrho_t-(\baphi_H)_\#\varrho_t\|_{L^1}}_{I_2} +
    \underbrace{\|(\baphi_H)_\#\varrho_t-(\phi_H)_\#\varrho_t\|_{L^1}}_{I_3}.
\end{align}
Next we estimate the three terms separately.

\textbf{Estimate for $I_1$.} 
The first term in \cref{e:RK2-I123} is exactly the incoming distribution error:  
\begin{align}\begin{split}\label{e:defI1 intro}
    I_1 &\coloneq \int_{\mathbb{R}^d} \left|\dvarrho_{t} (\dphi_H ^{-1}(x)) \cdot | \det[ \nabla (\dphi_H ^{-1}(x))]| -  \varrho_{t} (\dphi_H ^{-1}(x)) \cdot | \det[ \nabla (\dphi_H ^{-1}(x))]| \right| \de x\\
    &=\int_{\mathbb{R}^d} |\dvarrho_{t}(x)-\varrho_{t}(x)|\de x,
\end{split}\end{align}
where the last line follows from the change of variables $y = \dphi_H ^{-1}(x)$. 

\textbf{Estimate for $I_2$.} 
The term $I_2$ compares the learned map $\dphi_H$ with the oracle map $\baphi_H$, and it is the only term in which the learned score enters. By the change-of-variables formula,
\begin{align}\label{e:defI2 intro}
    I_2 &\coloneq \int_{\mathbb{R}^d} \left| \varrho_{t} (\dphi_H ^{-1}(x)) \cdot | \det[ \nabla (\dphi_H ^{-1}(x))]| -  \varrho_{t} (\baphi_H ^{-1}(x)) \cdot | \det[ \nabla (\baphi_H ^{-1}(x))]| \right| \de x
\end{align}
The goal is to bound this quantity in terms of the map difference $\baphi_H - \dphi_H$, and hence in terms of the score-matching error; the detailed proof is given in \cref{Section: second term conclusion}. Since \cref{e:defI2 intro} is expressed in terms of the inverse maps $\dphi_H^{-1}$ and $\baphi_H^{-1}$, we use the interpolation 
$\varphi_s = (1 - s)\dphi_H^{-1} + s\baphi_H^{-1}$. Applying the interpolation estimate \eqref{e: interpolation error in second term}, and decomposing the resulting expression into the two contributions estimated in \cref{ssec:J1-proof,ssec:J_2_proof}, yields 
\begin{align}\label{e:I2 decomposition}
\begin{split}
I_2 \lesssim_d &
\int_{\mathbb{R}^d} \varrho_t(x) \big\| \baphi_H (x) - \dphi_H(x) \big\|_{2} \de x + \int_{\mathbb{R}^d} \varrho_t(x) \|x\|_2 \big\| \baphi_H (x) - \dphi_H(x) \big\|_{2} \de x \\
&+ \sum_{1\leq i,j \leq d} \int_{\mathbb{R}^d} \big| \partial_i \bigl(\varrho_t(x) (\dphi_H^{(j)}(x) - \baphi_H^{(j)}(x))\bigr) \big| \de x + \mathcal{O}(e^{-1/(\xi H)}).
\end{split}
\end{align}
Here $\dphi_H = (\dphi_H^{(1)}, \dots,\dphi_H^{(d)})$ denotes the coordinate representation of $\dphi_H$. In \cref{e:I2 decomposition},  the notation $\lesssim_d$ suppresses constants depending on $d$ and on the parameters $A_p, B_p, s, p, D,\widetilde K, \widetilde W$, while $\xi$ depends only on  $D,\widetilde K,\widetilde W$. The detailed version of \cref{e:I2 decomposition} with explicit constants is presented in \cref{e:I2bb}.

We first bound the first two terms in \cref{e:I2 decomposition} by the score-matching error. Using  $\sV_t(x)-V_t(x) = s_t(x)-\nabla\log\varrho_t(x)$, we estimate 
\begin{align}\label{e:RK2-map-difference}
\begin{split}
    &\|\dphi_H(x)-\baphi_H(x)\|_2 = H \|\dM_t(x;H) - \baM_t(x;H) \|_2
    \\  &
    \leq \frac{H}{2}\left(\|\sV_t(x)-V_t(x)\|_2 +\|\sV_{t+H}(x+H\sV_t(x))
      -V_{t+H}(x+HV_t(x))\|_2 \right)\\
    &\leq \frac{H}{2}\bigg(\|\sV_t(x)-V_t(x)\|_2 +\|\sV_{t+H}(x+H\sV_t(x))
      -\sV_{t+H}(x+H V_t(x))\|_2 \\
    & \quad +\|\sV_{t+H}(x+H V_t(x))
      -V_{t+H}(x+HV_t(x))\|_2\bigg)\\
    &\lesssim H \left(\|\sV_t(x)-V_t(x)\|_2
      + \|\sV_{t+H}(\psi_H(x))-V_{t+H}(\psi_H(x))\|_2 \right),
\end{split}
\end{align}
where $\psi_H(x)=x+HV_t(x)$. In the last line we use the Lipschitz bound on $\sV_{t+H}$ from \Cref{a:score-derivative}, together with the smallness of $H$. For general $p \geq 2$-order scheme, analogous estimates are collected in \Cref{lemma: pth order score error}.
Substituting \cref{e:RK2-map-difference} into the first two terms of \cref{e:I2 decomposition} reduces these contributions to the score-matching error. The second term on the right hand side of \cref{e:RK2-map-difference} is evaluated at $\psi_H(x)$ rather than at $x$. This change of variables produces an additional exponentially small error, due to the comparison between $\varrho_t(x)$ and  $\varrho_t(\psi_H^{-1}(x))$. This comparison is provided by the cutoff argument in \Cref{lemma: score error cut-off general scheme}. In particular,
    \begin{align*}
        \begin{split}
            &\int_{\mathbb{R}^d} \varrho_t(x) \|\sV_{t+H}(\psi_H(x))-V_{t+H}(\psi_H(x))\|_2 \de x  \lesssim \int_{\mathbb{R}^d} \varrho_t(\psi_H ^{-1}(x)) \|\sV_{t+H}(x)-V_{t+H}(x)\|_2 \de x 
            \\  & \lesssim \int_{\mathbb{R}^d} \varrho_t(\phi_H ^{-1}(x)) \|\sV_{t+H}(x)-V_{t+H}(x)\|_2 \de x + \mathcal{O}(e^{-1/(\xi H)})
            \\  & \lesssim \int_{\mathbb{R}^d} \varrho_{t+H}(x) \|\sV_{t+H}(x)-V_{t+H}(x)\|_2 \de x + \mathcal{O}(e^{-1/(\xi H)}).
        \end{split}
    \end{align*}
It remains to control the derivative term in \cref{e:I2 decomposition}. This term is bounded by the first term in \cref{e:I2 decomposition} through the Gagliardo-Nirenberg interpolation inequality in \Cref{lemma: Gagliardo-Nirenberg}, which controls the gradient of a function by the function itself and its second derivatives. This is the only step that uses the second-derivative bound on the learned score from \Cref{a:score-derivative}. See details for this Gagliardo-Nirenberg interpolation in \Cref{lemma: Second derivative bound on errors}. Consequently, with the detailed constants collected in \Cref{Prop: Conclusion second error term}, we obtain
\begin{align}\label{e:RK2-I2-bound}
    I_2\leq C_{\rm score}\sigma_{\tau} ^{-4} H d^{7/4}\varepsilon_{\rm score}(t)^{1/2}
    +C H\sigma_{\tau} ^{-10}\pi^{-d/10}e^{-1/(\xi H)}.
\end{align}

\textbf{Estimate for $I_3$.} The term $I_3$ compares the oracle Runge-Kutta map with the exact flow map and no longer involves the learned score. 
Introduce the oracle Runge-Kutta interpolation
\[
    F_h(x) \coloneq x+\frac{h}{2}V_t(x)+\frac{h}{2}V_{t+h}\bigl(x+hV_t(x)\bigr),
    \qquad 0\leq h\leq H,
\]
so that $F_H=\baphi_H$.
Define $\baV_{t+h}=(\partial_hF_h)\circ F_h^{-1}$, then $\partial_hF_h(x) = \baV_{t+h}(F_h(x))$. Hence, the oracle Runge-Kutta 2 pushforward, $(F_h)_\# \varrho_{t}$, solves the continuity equation with velocity $\overline V_{t+h}$, while the exact density solves the continuity equation with velocity $V_{t+h}$.
The stability estimate for continuity equations in \Cref{theorem: L^1 error} gives
\begin{align}\label{e:RK2-I3-continuity}
    \begin{split}
        I_3 &\coloneq \int_{\mathbb{R}^d} \left| \overline\varrho_{t+H}(x)  -   \varrho_{t+H} (x)\right| \de x = \int_{\mathbb{R}^d} \left| (F_H)_\# \varrho_{t}  -   (\phi_H)_\# \varrho_{t} (x)\right| \de x
        \\  &\leq \int_0^H\int_{\mathbb R^d}
    \left|\nabla\cdot\bigl(\varrho_{t+h}(x)(V_{t+h}(x)-\overline V_{t+h}(x))\bigr)\right|
    \,\de x\,\de h .
    \end{split}
\end{align}
We further claim that 
\begin{align}\label{eq:I3 V-baV bound}
    \|V_{t+h}(x)-\overline V_{t+h}(x)\|_\infty
    +\|\nabla(V_{t+h}-\overline V_{t+h})(x)\|_\infty
    \leq C H^2 d^2 \sigma_{\tau} ^{-2\gamma_1} (D+\|x\|_\infty)^{\gamma_2}.
\end{align}
This estimate reflects the second-order consistency of the Runge-Kutta2 scheme: the  Taylor expansions of the exact characteristic flow $\phi_h$ and of $F_h$ agree through order $h^2$ at $h=0$. 
We illustrate this estimate for $\|V_{t+h}(F_h(x))-\overline V_{t+h}(F_h(x))\|_\infty$, and then \eqref{eq:I3 V-baV bound} can follow from a change of variable argument, and one can similarly obtain the estimate for $\|\nabla(V_{t+h}-\overline V_{t+h})(x)\|_\infty$ following \cref{e:tV-V}--\cref{e:VFdiff}. First, we have that
    \begin{align}\label{eq:I3 order 2 Taylor phi_h F_h}
        \begin{split}
            &F_0(x) = \phi_0(x) = x, \ \left.\frac{\rd F_h(x)}{\rd h}\right|_{h=0} =\left.\frac{\rd \phi_h(x)}{\rd h}\right|_{h=0}= V_t(x), 
            \\  &\left.\frac{\rd^2 \phi_h(x)}{\rd^2 h}\right|_{h=0}=\left.\frac{\rd^2 F_h(x)}{\rd^2 h}\right|_{h=0} = \partial_t V_t(x) + (\nabla V_t(x)) V_t(x).
        \end{split}
    \end{align}
Thus, by the chain rule, we have for $0\leq m\leq 1$,
\begin{align*}
    &\left.\frac{\rd^m  V_{t+h}(F_h(x))}{\rd^m h}\right|_{h=0} = \left.\frac{\rd^m  V_{t+h}(\phi_h(x))}{\rd^m h}\right|_{h=0}= \left.\frac{\rd^{m+1} \phi_h(x)}{\rd^{m+1} h}\right|_{h=0} = \left.\frac{\rd^{m+1} F_h(x)}{\rd^{m+1} h}\right|_{h=0} = \left.\frac{\rd^m  \overline V_{t+h}(F_h(x))}{\rd^m h}\right|_{h=0}.
\end{align*}
By the Taylor expansion up to order $2$, we conclude that
    \begin{align*}
        V_{t+h}(F_h(x))-\overline V_{t+h}(F_h(x)) = \int_0^{h}(h-\tau) \left( \frac{\rd^2 V_{t+\tau }(F_\tau(x))}{\rd \tau^2} - \frac{\rd^2 \overline V_{t+\tau }(F_\tau(x))}{\rd \tau^2} \right) \de \tau,
    \end{align*}
whose $\|\cdot\|_{\infty}$-norm is bounded by
    \begin{align}\label{e:sup second derivative V in I3}
         H^2 \sup_{0 \leq \tau \leq H}\left(\left\| \frac{\rd^2 V_{t+\tau }(F_\tau(x))}{\rd \tau^2} \right\|_{\infty} +\left\| \frac{\rd^2 \overline V_{t+\tau }(F_\tau(x))}{\rd \tau^2} \right\|_{\infty} \right),
    \end{align}
which is further bounded by the right hand side of \eqref{eq:I3 V-baV bound} using estimates in \Cref{lemma: time-space derivatives log qt}, because terms in \cref{e:sup second derivative V in I3} only involve the space and time derivatives of the true score $V_t(x) = x+\nabla \log \varrho_t(x)$. The full details of estimating \cref{e:sup second derivative V in I3} are included in \Cref{lemma:V derivative bound}.

In this argument~\eqref{eq:I3 V-baV bound}, we only used the derivatives of the true score $V_t(x)$, which are controlled by the Gaussian smoothing estimates in \Cref{s:prel-estim-q}; no derivative of the learned score $\sV_t(x)$ appears in the estimate for $I_3$.
The corresponding estimates for general $p$-th order schemes are proved in \Cref{lemma: high_order_error,lemma:V derivative bound}.

Thus, substituting \eqref{eq:I3 V-baV bound} into \eqref{e:RK2-I3-continuity} and using the Gaussian moment bounds from \Cref{lemma: moments of L^2 and infinity norms} yields
\begin{align}\label{e:RK2-I3-bound}
    I_3 \leq C_{\rm RK}H^3d^3\sigma_{\tau} ^{-2\gamma_1}(D+\sqrt{\log d})^{\gamma_2}.
\end{align}

\textbf{Estimate for \eqref{e:RK2-I123} and conclusion.} Finally, the exponentially small term in \eqref{e:RK2-I2-bound} is absorbed into the $H^3$ term by using $e^{-1/(\xi H)}\leq C_\xi H^3$ for sufficiently small $H$.
Combining \eqref{e:defI1 intro}, \eqref{e:RK2-I2-bound}, \eqref{e:RK2-I3-bound}, and \eqref{e:RK2-I123} proves \cref{t:RK_1step} with $p=2$.

\subsection{Extension to the General $p$th-Order Case}\label{s:Proof Outline of RK_1step}
The proof for a general $p$-th order explicit Runge-Kutta scheme in \eqref{eq:RK-update} and \eqref{eq:RK-update-k} uses the same three-map decomposition as the Runge-Kutta2 proof above.
The score perturbation term $I_2$ is unchanged in nature: it compares two Runge-Kutta maps with identical coefficients and different score evaluations, and it uses only the $C^2$ learned-score bound through the same interpolation and Gagliardo-Nirenberg argument.
The only part that depends on $p$ is $I_3$, the oracle Runge-Kutta error.
For $I_3$, one must differentiate the oracle interpolation map $F_h$ up to order $p$ and use the Runge-Kutta order conditions to show that the exact characteristic flow $\phi_h$ and $F_h$ agree through order $p$ at $h=0$.
The derivative bookkeeping for this general step is technical and is given in \Cref{Section: third term conclusion}.
We keep the full general-order proof of \Cref{t:RK_1step}, together with the auxiliary estimates for $I_2$, $I_3$, and invertibility of the Runge-Kutta maps, in \Cref{app:general-order-one-step}.
Thus, the main text contains the complete RK2 argument and the following global iteration argument, while the appendix records the general $p$ estimates needed for the stated theorem.


\subsection{Proof of the Main Total Variation \Cref{theorem: main L^1 theorem}}
\label{s:proof_main}
The main theory follows by iterating the one-step estimate in \cref{t:RK_1step}.
Recall the partition $0=t_0<t_1<\cdots<t_M=T-\tau$ and write $H_i=t_{i+1}-t_i$. 
Set
\[
    e_i\coloneq \int_{\mathbb{R}^d}|\dvarrho_{t_i}(x)-\varrho_{t_i}(x)|\,\de x .
\]
By \Cref{t:RK_1step}, and by choosing the constant $\Delta_{\rm disc}$ small enough to satisfy the auxiliary small-step requirements in \Cref{p:ABCD}, the condition
\[
    H_{\mathrm{max}}\coloneq \max_{0\leq i\leq M-1}H_i
    \leq (d\log(2d))^{-1}\sigma_{\tau} ^{6} \Delta_{\rm disc}
\]
implies, for every $0\leq i\leq M-1$,
\begin{align}\label{e:one_step_recursion}
    e_{i+1}\leq e_i+\cE^i_{\rm score}+\cE^i_{\rm RK},
\end{align}
where
\[
    \cE^i_{\rm score}
    =C_{\rm score}\sigma_{\tau} ^{-4}H_i d^{7/4}\bigl(\varepsilon_{\rm score}(t_i)\bigr)^{1/2},
    \qquad
    \cE^i_{\rm RK}
    =C_{\rm RK}H_i^{p+1}d^{p+1}\sigma_{\tau} ^{-2\gamma_1}
      {\bigl(D+\sqrt{\log d}\bigr)}^{\gamma_2}.
\]
Iterating \eqref{e:one_step_recursion} gives
\begin{align}\label{e:trho_diff}
    e_M\leq e_0+\sum_{i=0}^{M-1}\bigl(\cE^i_{\rm score}+\cE^i_{\rm RK}\bigr).
\end{align}

We first estimate the accumulated score error. By H{\"o}lder's inequality,
    \begin{align*}\begin{split}
        \sum_{i=0} ^{M-1} H_i  (\varepsilon_{\rm score}(t_i)) ^{\frac{1}{2}} 
        &\leq \left( \sum_{i=0} ^{M-1}  H_i \varepsilon_{\rm score} ^{2} (t_i) \right)^{\frac{1}{4}} \left( \sum_{i=0} ^{M-1} H_i \right)^{\frac{3}{4}} 
        \leq  \varepsilon_{\rm score} ^{\frac{1}{2}}  T^{\frac{3}{4}},
   \end{split} \end{align*}
where $\varepsilon^2_{\rm score}\geq  \sum_{i=0}^{M-1} H_i \varepsilon^2_{\rm score}(t_i)$ follows from \Cref{a:score-estimate} and the definition of $\varepsilon^2_{\rm score}(t_i)$ in \eqref{e:score_bound}, and $\sum_{i=0}^{M-1}H_i=T-\tau$. Therefore,
    \begin{align}\label{e:E_score_bound}
        \sum_{i=0} ^{M-1}\cE^i _{\rm score} \leq C_{\rm score}  T^{\frac{3}{4}} \sigma_{\tau} ^{-4}    d^{\frac{7}{4}}   \cdot \varepsilon_{\rm score} ^{\frac{1}{2}}.
    \end{align}
For the Runge-Kutta truncation error, $\sum_{i=0}^{M-1}H_i^{p+1}\leq H_{\mathrm{max}}^pT$, and hence
    \begin{align}\label{e:ERK_bound}
        \sum_{i=0} ^{M-1} \cE^i _{\rm RK} \leq C_{\rm RK}  T H_{\mathrm{max}} ^p  d^{p+1}  \sigma_{\tau} ^{-2\gamma_1} {\big(D+\sqrt{ \log d}\big)}^{\gamma_2}.
    \end{align}
Combining \eqref{e:E_score_bound} and \eqref{e:ERK_bound} with \eqref{e:trho_diff} completes the proof of \Cref{theorem: main L^1 theorem}.

 \subsection{Proof of \Cref{theorem: main L^1 theorem general} for the General Forward Process}
\label{s:proof_general}
Finally, for the general forward process, we prove \Cref{theorem: main L^1 theorem general} by rewriting both the standard and exponential Runge-Kutta schemes with a general variance schedule in the discrete-flow form analyzed above.

For the standard Runge-Kutta schemes \eqref{eq:RK-update-beta}, we introduce the modified coefficients
\begin{align*}
   \check b_j(H_i):=\frac{\beta_{T-t_i-c_j H_i}b_j}{2},\quad \check a_{jk}(H_i)= \frac{\beta_{T-t_i-c_kH_i} a_{jk}}{2},\quad 1\leq k<j\leq s.
\end{align*}
Although \eqref{eq:RK-update-beta} and \eqref{eq:RK-update-beta-k} depend on $\beta_t$, we can rewrite them in the following form:
\begin{align*}
\dY_{t_{i+1}} = \dY_{t_i} + H_i \sum_{j=1}^s  \check b_j \check k_j(\dY_{t_i}),
\end{align*}
with 
\begin{equation}\label{eq:RK-update-beta-k-rf}
\begin{split}
    &\check k_1(x) = \widehat V_{t_i + c_1 H_i}(x),\\
    &\check k_2(x) = {\widehat V}_{t_i + c_2 H_i}\bigl(x + (\check a_{21}(H_i)\check k_1(x)) H_i\bigr),\\
    &\check k_3(x) = {\widehat V}_{t_i + c_3 H_i}\bigl(x + (\check a_{31}(H_i)\check k_1(x) + \check a_{32}(H_i) \check k_2(x)) H_i\bigr),\\
    &\qquad\qquad\qquad\vdots
    \\
    &\check k_s(x) = {\widehat V}_{t_i + c_s H_i}\bigl(x + (\check a_{s1}(H_i) \check k_1(x) + \check a_{s2}(H_i)\check k_2(x) + \cdots + \check a_{s,s-1}(H_i)\check k_{s-1}(x)) H_i\bigr).
\end{split}
\end{equation}
In this reformulated version, the above scheme has the same structure as the standard Runge-Kutta schemes \eqref{eq:RK-update-k}.

Similarly, for the enhanced exponential Runge-Kutta scheme \eqref{eq:Exp-RK-update-beta} and \eqref{eq:Exp-RK-update-beta-k}, we introduce
\begin{align*}
\begin{split}
    &\check k_0(x)=x, \quad \check b_0(H_i)=\frac{e^{\zeta(t_{i}+H_i) - \zeta(t_i)}-1}{H_i},\quad 
      \check a_{j0}(H_i)=\frac{e^{\zeta(t_{i} + c_jH_i) - \zeta(t_i)}-1}{H_i},\\
      &\check b_{j}(H_i)=\sigma_{T-t_i-c_j H_i} b_j(H_i), \quad \check a_{jk}(H_i)=\sigma_{T-t_i-c_k H_i} a_{jk}(H_i), \quad 1\leq k<j\leq s.
\end{split}\end{align*}
With these definitions, the enhanced exponential Runge-Kutta scheme can be written as: 
\begin{align*}
\dY_{t_{i+1}} = \dY_{t_i} + H_i \sum_{j=0}^s \check b_j (H_i)\check k_j(\dY_{t_i}),
\end{align*}
and 
\begin{equation}\label{eq:Exp-RK-update-beta-k-rf}
\begin{split}
    &\check k_1(x) = s_{t_i+c_1H_i}(x+\check a_{10}(H_i) \check k_0(x) H_i),\\
    &\check k_2(x) = s_{t_i + c_2 H_i}\bigl(x + (\check a_{20}(H_i) \check k_0(x)+\check a_{21}(H_i) \check k_1(x)) H_i\bigr),\\
    &\check k_3(x) = s_{t_i + c_3 H_i}\bigl(x + (\check a_{30}(H_i) \check k_0(x)+\check a_{31}(H_i) \check k_1(x) + \check a_{32}(H_i) \check k_2(x)) H_i\bigr),\\
    &\qquad\qquad\qquad\vdots
    \\
    &\check k_s(x) = s_{t_i + c_s H_i}\bigl(x + (\check a_{s0}(H_i) \check k_0(x)+\check a_{s1}(H_i) \check k_1(x) + \check a_{s2}(H_i) \check k_2(x) + \cdots + \check a_{s,s-1}(H_i)\check k_{s-1}(x)) H_i\bigr).
\end{split}
\end{equation}
Again, in this new form, the enhanced exponential Runge-Kutta schemes have the same structure as the standard schemes \eqref{eq:RK-update-k} with $\sV_t$ replaced by $s_t$. 
Importantly, both  \( s_t(x) \) and \( \widehat V_t(x) = x + s_t(x) \) satisfy the same regularity assumptions (see \Cref{a:score-derivative}). Therefore, the bounds collected in \Cref{p:ABCD} apply with the same proof after replacing $\sV_t$ by $s_t$. Moreover, the explicit formula \eqref{e:general_sigma_t} for $\sigma_t$ gives the corresponding derivative and moment estimates for the forward density of \eqref{eq:OU-time}, exactly as used in \Cref{s:prel-estim-q}.

The following lemma states that the coefficients of both rewritten standard and exponential schemes \eqref{eq:RK-update-beta-k-rf}  and \eqref{eq:Exp-RK-update-beta-k-rf}  satisfy the regularity condition in \Cref{a:bounds on RK}.  The proof of \cref{lem:coefficients} depends only on the explicit coefficient formulas and is given in \Cref{s:coefficient_est}. 
\begin{lemma}
\label{lem:coefficients}
We consider the $i$-th time step for the Runge-Kutta schemes over the time interval $[t_i, t_{i}+H_i]\subset[0, T-\tau]$, such that $H_{i} \leq \tau$. Under \cref{a:bounds on beta_t} for the variance schedule $\beta_t$, the Runge-Kutta coefficients, rewritten in the forms \eqref{eq:RK-update-beta-k-rf} and \eqref{eq:Exp-RK-update-beta-k-rf} and treated as  functions of the time step $h$, satisfy the following derivative bounds.
\begin{align}
\label{e:derbound}
    \max_{0\leq j,k\leq s}\max_{0\leq l\leq p+1} \sigma_{\tau} ^{2l} |\del_{h}^l \check{a}_{jk}(h)|\leq A_p,\quad \max_{0\leq j\leq s} \max_{0\leq l\leq p+1} \sigma_{\tau} ^{2l} |\del_{h}^l \check{b}_j(h)|\leq B_p.
\end{align}
\end{lemma}
With \Cref{lem:coefficients} in hand, the proof of
\Cref{theorem: main L^1 theorem} applies directly to the rewritten
standard and exponential schemes
\eqref{eq:RK-update-beta-k-rf} and
\eqref{eq:Exp-RK-update-beta-k-rf}.
In both cases, the one-step decomposition into $I_1$, $I_2$, and $I_3$ in \Cref{t:RK_1step} remains valid. The only difference is that the
exponential scheme uses \(s_t\), rather than \(x+s_t\), in its stage
evaluations. The linear term $x$ is
identical in the exact, oracle, and learned maps $\phi_h$, $\baphi_h$ and $\widetilde{\phi}_h$; therefore, it neither changes the required regularity nor contributes to the map differences estimated in 
$I_2$ and $I_3$.
Summing the resulting one-step errors exactly as in \Cref{s:proof_main} completes the proof of \Cref{theorem: main L^1 theorem general}.


%


\section{Numerical Experiments}
\label{sec:num}
In this section, we present numerical experiments to validate our theoretical assumptions, verify the derived convergence rates, and investigate sharp bounds for the probability flow ODE approach.

\subsection{Assumption Verification}
In this subsection, we numerically verify our \cref{a:score-derivative}, which concerns the boundedness of the gradient and Hessian of the estimated score function. Specifically, we evaluate:
\begin{align*}
        \sup_{x \in \mathbb{R}^d} \max_{1\leq l \leq d} |\partial_l s_t ^{(j)}(x)| , \quad \sup_{x \in \mathbb{R}^d} \max_{1\leq l, k\leq d} |\partial^2 _{lk} s_t ^{(j)}(x)|, 
    \end{align*}
on both the MNIST~\cite{lecun2002gradient} and FashionMNIST~\cite{xiao2017fashion} datasets. 
For both datasets, the data dimension is $d=28^2$, where $x \in \bR^d$ and $s_t(x) \in \bR^d$; the superscript $1 \leq j \leq d$ denotes the $j$-th component of the score function.
The diffusion model is trained using the open-source library \url{https://github.com/lucidrains/denoising-diffusion-pytorch}. Its neural network architecture follows the PixelCNN++ backbone~\cite{salimans2017pixelcnn++}, incorporating convolutional residual blocks and self-attention mechanisms. 
The forward process \eqref{eq:OU-time} employs a linear variance schedule from $\beta_{\min} = 10^{-4}$ to $\beta_{\max} = 0.02$ with $T=1000$.   
Generated samples are visualized in \cref{fig:generated_samples}. 
To assess the gradient and Hessian, we visualize them at
$t=0.1T$, $0.5T$, $0.9T$ and $0.99T$ each as a $28\times 28$ matrix in \cref{fig:MNIST_results} and \cref{fig:FashionMNIST_results}. 
Specifically,  for the gradient, the figures visualize $\sigma_{\tau} ^{4}\sup_{x \in \mathbb{R}^d} \max_{1\leq l \leq d} |\partial_l s_t ^{(j)}(x)|$, where each pixel corresponds to a different $j$; similarly, for the Hessian, the figures visualize $\sigma_{\tau} ^{6}\sup_{x \in \mathbb{R}^d} \max_{1\leq l, k\leq d} |\partial^2 _{lk} s_t ^{(j)}(x)|$, where each pixel also corresponds to a different $j$.
In both cases, the maximum values of the gradient and Hessian increase as $t$ approaches the terminal time $T$. However, after scaling by the corresponding powers of $\sigma_{\tau}$ with $\tau = 10$, these quantities remain bounded and small,  which empirically verifies our \Cref{a:score-derivative}:
    \begin{align*}
        \sigma_{\tau} ^{4}\sup_{x \in \mathbb{R}^d} \max_{1\leq l, j \leq d} |\partial_{l} s_t ^{(j)}(x)| \leq \widetilde K , \quad \sigma_{\tau} ^{6}\sup_{x \in \mathbb{R}^d} \max_{1\leq l,k, j \leq d} |\partial^2 _{lk} s_t ^{(j)}(x)| \leq \widetilde K ,
    \end{align*}
for any $t \in [0,T-\tau]\cap \bT$.

\begin{figure}[ht]
\centering
    \includegraphics[width=0.3\textwidth]{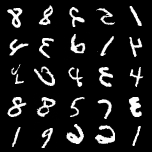}~~~~~
    \includegraphics[width=0.3\textwidth]{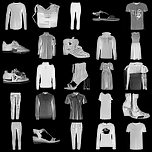}
    \caption{Generated samples from MNIST and FashionMNIST.
     }
    \label{fig:generated_samples}
\end{figure}

\begin{figure}[ht]

\centering
    \includegraphics[width=0.9\textwidth]{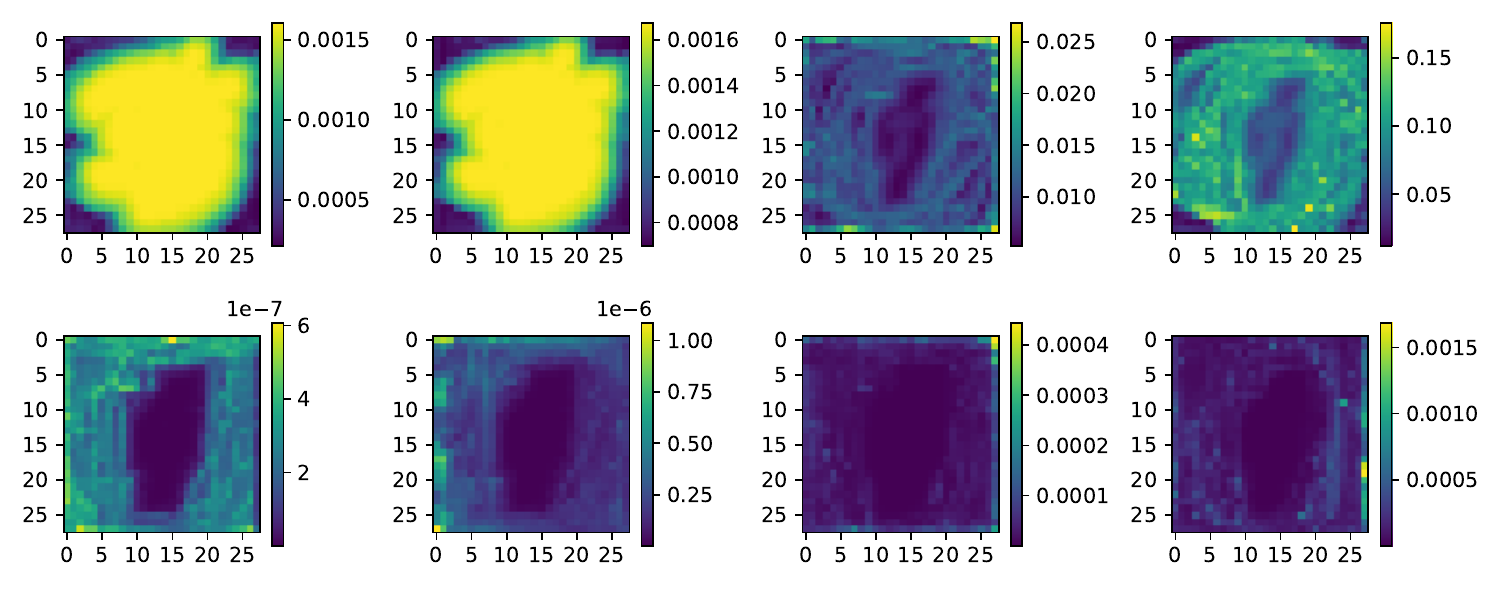}
    \caption{MNIST dataset: maximum of gradient $\sigma_{\tau} ^4 \sup_{x \in \mathbb{R}^d} \max_{1\leq l \leq d} |\partial_l s_t ^{(j)}(x)|$ (top) and Hessian $\sigma_{\tau} ^6 \sup_{x \in \mathbb{R}^d} \max_{1\leq l, k\leq d} |\partial^2 _{lk} s_t ^{(j)}(x)|$ (bottom) at $t=0.1T$, $0.5T$, $0.9T$, and $0.99T$ (from left to right).
     }
    \label{fig:MNIST_results}
\end{figure}

\begin{figure}[ht]
\centering
    \includegraphics[width=0.9\textwidth]{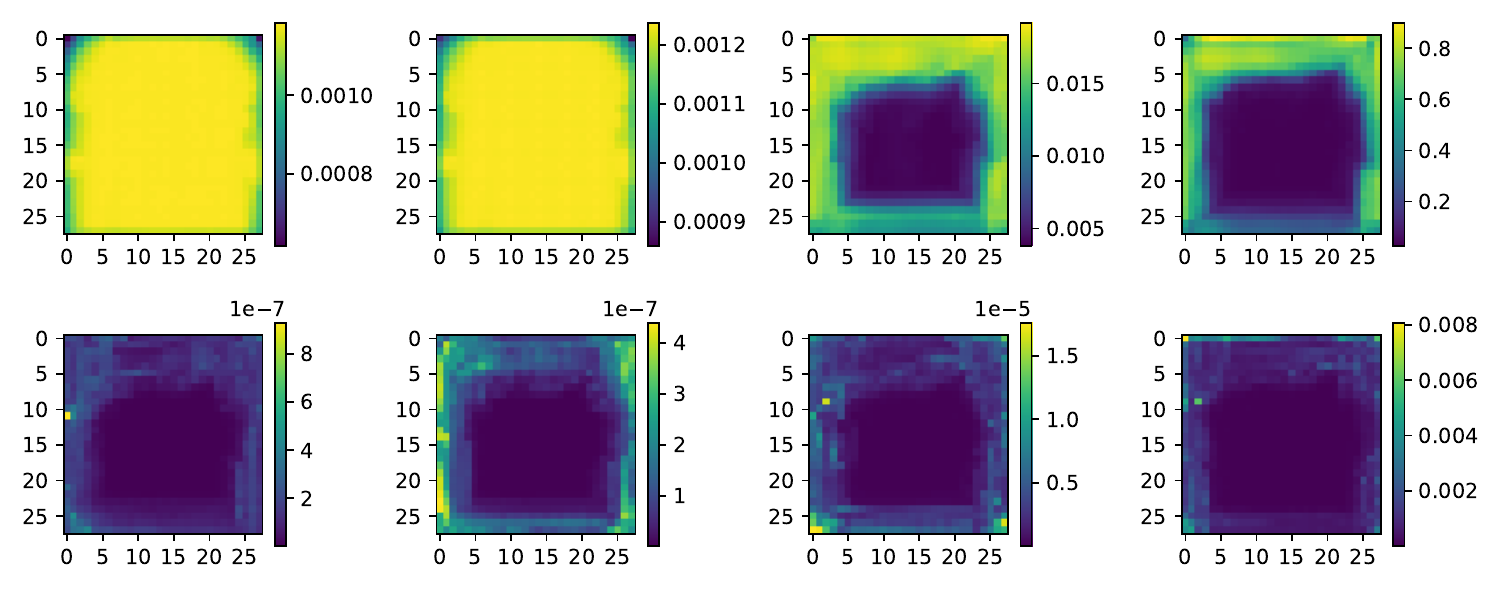}
    \caption{FashionMNIST dataset: maximum of gradient $\sigma_{\tau} ^4 \sup_{x \in \mathbb{R}^d} \max_{1\leq l \leq d} |\partial_l s_t ^{(j)}(x)|$ (top) and Hessian $\sigma_{\tau} ^6 \sup_{x \in \mathbb{R}^d} \max_{1\leq l, k\leq d} |\partial^2 _{lk} s_t ^{(j)}(x)|$ (bottom) at $t=0.1T$, $0.5T$, $0.9T$, and $0.99T$ (from left to right).
     }
    \label{fig:FashionMNIST_results}
\end{figure}

\subsection{Convergence Verification and Sharp Bounds} 
\label{ssec:sharp_bounds}
In this subsection, we numerically analyze the convergence behavior of various schemes for solving the probability flow ODE. Our focus is on a $K$-mode Gaussian mixture target distribution, given by 
\begin{equation}
\label{eq:GM}
    q_0(x) = \sum_{k=1}^K w_k \N(x; m_k, C_k).
\end{equation}
The motivation for using a Gaussian mixture is its widespread application in classification tasks and its analytical form. Although its density does not have compact support, it decays exponentially fast.
The forward process, such as the OU process denoted in \eqref{eq:OU} or more practical process denoted in \eqref{eq:OU-time}, yields
\begin{align*}
	q_t(y) = &  \int_{\bR^d} \frac{1}{(\sqrt{2\pi} \sigma_t)^d} \cdot \exp \left( -\frac{\|y - \lambda_t x\|_2^2}{2\sigma_t^2} \right) q_0(x) \de x 
	=  \sum_{k=1}^K w_k \N(y; \lambda_t m_k, \lambda_t^2 C_k + \sigma_t^2 I).
\end{align*}
The score function takes the following analytical form 
\begin{align}
	\nabla_x \log q_t(x)
	= & -\sum_{k=1}^K\frac{w_k\N(x; \lambda_t m_k, \lambda_t^2 C_k + \sigma_t^2 I)}{q_t(x)} (\lambda_t^2C_k + \sigma_t^2I)^{-1}(x - \lambda_t m_k).
\end{align}

\begin{figure}[ht]
\centering
    \includegraphics[width=0.5\textwidth]{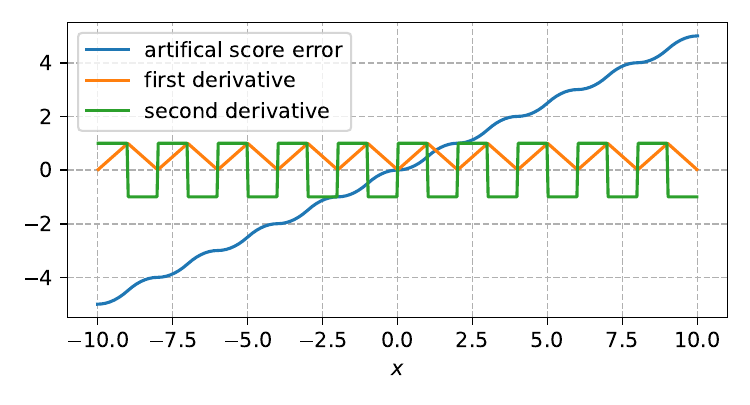}
    \caption{Artificial score error and its first and second derivatives used in \cref{ssec:sharp_bounds}.
     }
    \label{fig:score-error-visualization}
\end{figure}

As discussed previously, the score function is typically represented by a neural network with inputs $t$ and $x$, trained through score matching with sequentially corrupted training data (see \cref{s:score_matching}). However, for the purpose of this test, we circumvent the score matching step. Instead, we assume access to an imperfect score function, characterized by the following artificial score error (visualized in \cref{fig:score-error-visualization}):
\begin{align*}
    s_{t}(x) - \nabla \log \varrho_t(x) =  \varepsilon_{\rm score}\frac{\mathbbm{1} \delta(x_1)}{\sqrt{d}} \quad\textrm{where} \quad\delta(x) = \begin{cases} n+\frac{(x - \lfloor{x\rfloor})^2}{2} & \lfloor{x\rfloor} = 2n  \\ n+1 - \frac{(1 + \lfloor{x\rfloor} -x)^2}{2} & \lfloor{x\rfloor} = 2n+1.
\end{cases}
\end{align*}
Here, $\varepsilon_{\rm score}$ denotes the magnitude of the score error, and $d$ is the dimension of $x$.  The vector $\mathbbm{1} \in \bR^d$ is an all-ones vector, scaled by the factor $\sqrt{d}$. The function $\delta:\bR \rightarrow \bR$ has discontinuous second order derivatives $ \delta''(x) = \begin{cases} 1 & \lfloor{x\rfloor} = 2n  \\ -1 & \lfloor{x\rfloor} = 2n+1
\end{cases}$.

For the following numerical study, we consider various combinations
\begin{itemize}
    \item Standard Runge-Kutta schemes up to fourth order applied to the OU forward process with $T=16$, denoted as RK1, RK2, RK4.
    \item Standard Runge-Kutta schemes up to fourth order applied to a general forward process with a linear variance schedule from $\beta_{\min} = 10^{-4}$ to $\beta_{\max} = 0.02$ with $T=2000$, denoted as RK1($\beta$), RK2($\beta$), RK4($\beta$).
    \item Enhanced exponential Runge-Kutta schemes up to third order applied to the same general forward process with the aforementioned linear variance schedule and $T=2000$, denoted as ExpRK1($\beta$), ExpRK2($\beta$), ExpRK3($\beta$). 
\end{itemize}
To initialize the Runge-Kutta solvers, we sample $J=10^7$ particles from the standard Gaussian distribution $\mathcal{N}(x; 0, \id_d)$.
Convergence is assessed based on the total variation error, relative mean error, and relative covariance error.
For the total variation error, we consider the marginal density along the first dimension and estimate it using kernel density estimation, with bandwidth selected according to Silverman’s rule \cite{silverman2018density}. All code used to produce the numerical results and figures in this paper are available at
    \url{https://github.com/PKU-CMEGroup/InverseProblems.jl/blob/master/Diffusion/Gaussian-mixture-density-RK-include.ipynb}\,.

\begin{figure}[ht]
\centering
    \includegraphics[width=0.9\textwidth]{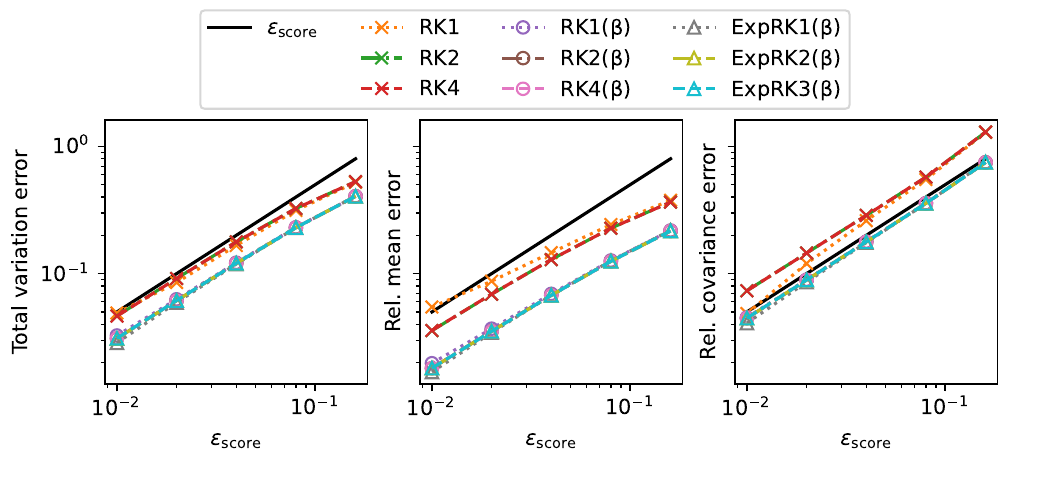}
    \caption{One dimensional test: convergence of density estimates obtained by solving the probability flow ODE using different Runge-Kutta schemes (with 512 time steps, rendering temporal error negligible) under varying levels of artificial score error $\varepsilon_{\rm score}$. This observed linear convergence rate surpasses our worst-case bound of $\varepsilon_{\rm score}^{1/2}$.
     }
    \label{fig:score-error}
\end{figure}

\begin{figure}[ht]
\centering
    \includegraphics[trim=0 16pt 0 54pt, 
    clip, width=0.9\textwidth]{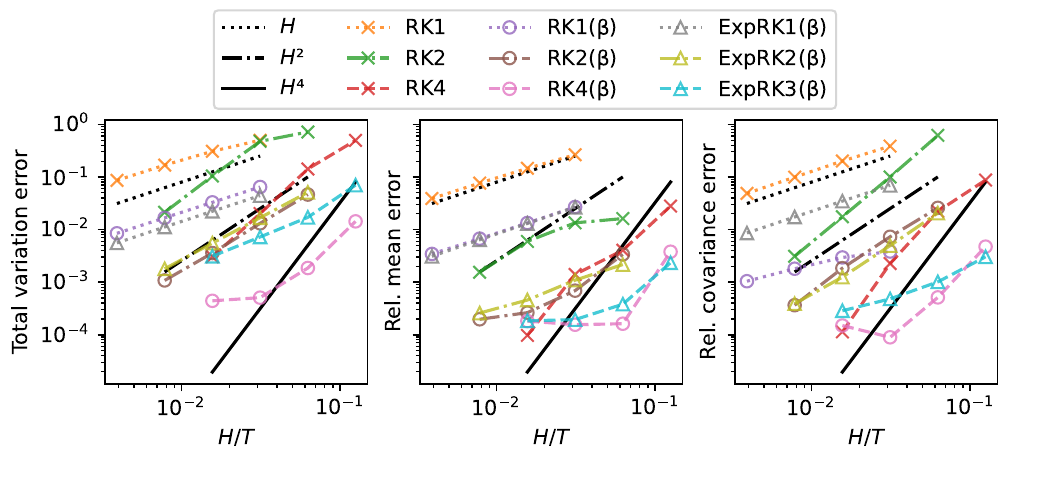}
    \includegraphics[width=0.9\textwidth]{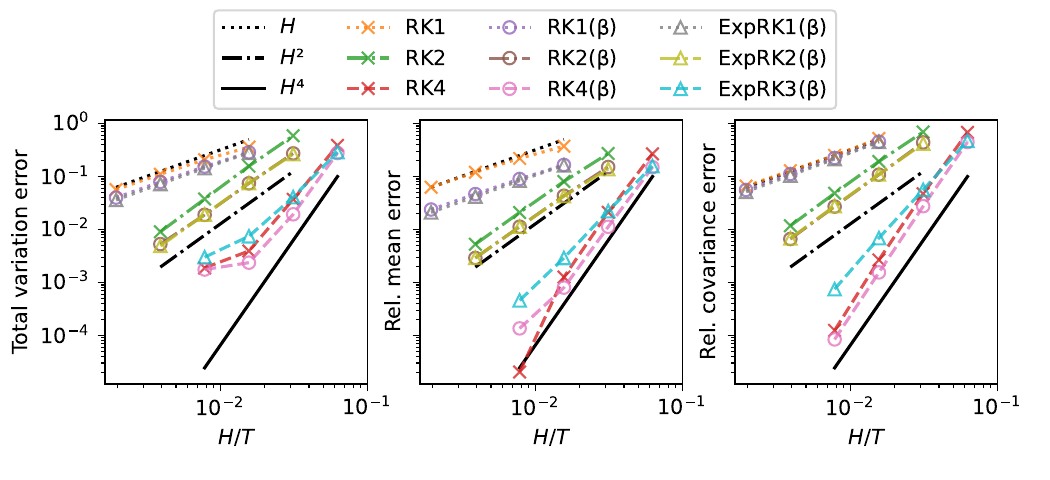}
    \caption{One dimensional test: convergence of density estimates obtained by solving the probability flow ODE using different Runge-Kutta schemes, with varying time step sizes $H$ and a fixed negligible score error magnitude $\varepsilon_{\rm score} = 10^{-6}$ (top), as well as with order-matching score error $\varepsilon_{\rm score} = \mathcal{O}(H^p)$ (bottom).
     }
    \label{fig:discretization-error}
\end{figure}

We begin with a one-dimensional simulation ($d=1$), where the target distribution is a three-component Gaussian mixture. 
First, we study the effect of score matching error on convergence. We keep the time step small (with 512 time steps, rendering temporal error negligible) and vary the magnitude of the score error $\varepsilon_{\rm score}$. The convergence results are depicted in  \cref{fig:score-error}.
A clear linear relationship is evident between the errors and $\varepsilon_{\rm score}$.  Interestingly, the empirical errors are smaller than the theoretical upper bound predicted by \cref{theorem: main L^1 theorem}, which scales as $\varepsilon_{\rm score}^{1/2}$. This theoretical rate stems from the use of the Gagliardo–Nirenberg interpolation inequality to control the contribution of the score error, as explained in \cref{rem:score_error}. When the score function exhibits greater smoothness, this exponent can be improved towards $1$, explaining the sharper empirical convergence observed.

Next, we study the effect of temporal discretization error. 
We fix a small score error magnitude $\varepsilon_{\rm score} = 10^{-6}$ and vary the time step size $H$. 
The observed convergence rates (see \cref{fig:discretization-error}-top) for different Runge-Kutta schemes align with the theoretical predictions. The plateau observed for very small step sizes is attributed to the finite sample size of particles, which imposes an error floor around $10^{-4}$. For the total variation error, it also suffers from the kernel density estimate error.
Additionally, we observe that schemes applied to forward processes with a linear variance schedule yield significantly lower errors compared to those using the OU process. In this case, the standard and exponential Runge-Kutta schemes perform similarly, which can be attributed to the absence of singularities in the score function as $t$ approaches the terminal time $T$. Furthermore, high-order schemes consistently achieve lower errors while maintaining similar computational cost. Among all the schemes tested, RK4 and ExpRK3 achieve the best performance and exhibit superior stability, accurately generating samples even with as few as 8 large time steps.
Next, we test the case where the score error matches the temporal discretization error, i.e., $\varepsilon_{\rm score} = \mathcal{O}(H^p)$.  The results, shown in  \cref{fig:discretization-error}-bottom,  clearly confirm the expected high-order convergence with respect to the time step size $H$. As predicted by \cref{theorem: main L^1 theorem}, the two sources of error do not compound or amplify one another. 

Finally, we study the effect of the dimension $d$. We consider a randomly generated 5-mode Gaussian mixture as described in \cite{albergo2023stochastic,huang2024convergence} in dimensions $d=8$, $32$, and $128$, and set the score error to match the order of the Runge-Kutta scheme. 
The estimated marginal densities for $d=128$ using various Runge-Kutta schemes applied to the OU process\footnote{To keep the figure concise, we omit results for Runge-Kutta schemes applied to forward processes with a linear variance schedule, which perform even better than those presented. } are depicted in \cref{fig:high-D-error}-top. We use a relatively large score error $\varepsilon_{\rm score} = 0.1$, with 32 time steps for RK1 and RK2, and 16 steps for RK4. Notably, RK4, which allows for a larger time step (and hence lower computational cost), successfully captures all modes and achieves the best accuracy.
The convergence results, depicted in  \cref{fig:high-D-error}-bottom, clearly demonstrate a high-order relationship between the error and the step size $H$.  Surprisingly, the observed error exhibits  no dependence on the dimensionality $d$, which is significantly better than the worst-case theoretical bound \cref{theorem: main L^1 theorem}. Based on this series of tests, we conjecture that the error scales as $\mathcal{O}(\varepsilon_{\rm score} + H^p)$, at least for Gaussian mixture target distributions. A linear dependence on  $\varepsilon_{\rm score}$ can be justified, as previously discussed, or by assuming an error bound on the gradient of the score~\cite{li2023towards,li2024sharp,li2024accelerating}. However, deriving a sharp estimate for  the dependence on dimensionality $d$ may require additional analysis tools.

\begin{figure}[ht]
\centering
    \includegraphics[width=0.7\textwidth]{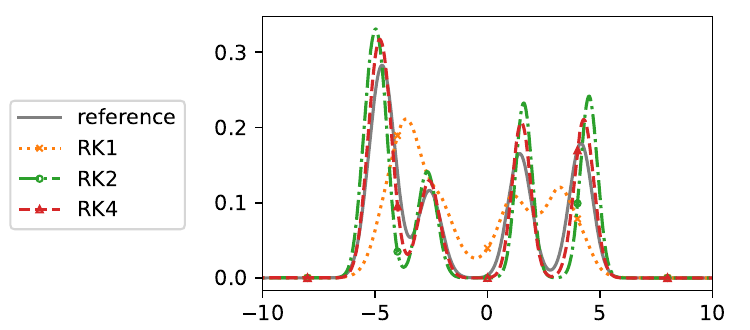}
    \includegraphics[width=0.9\textwidth]{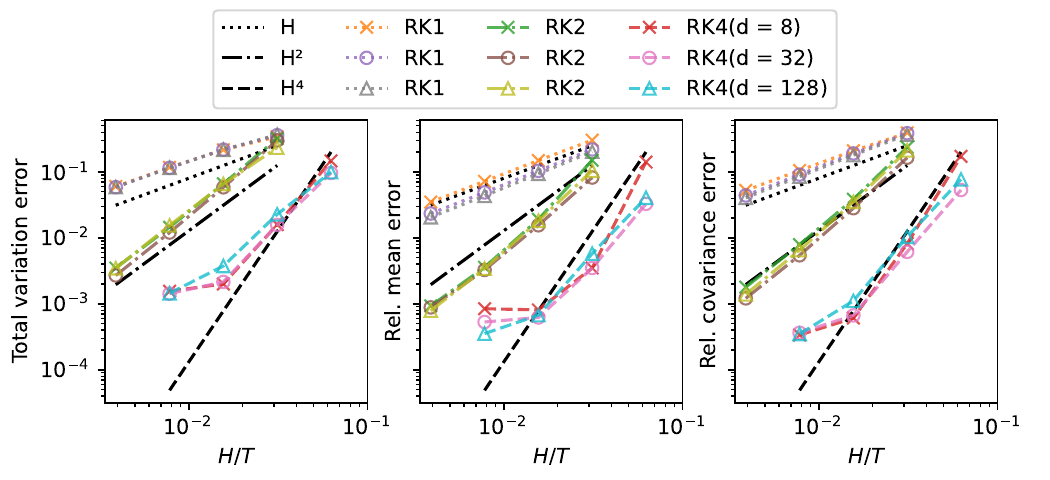}
    \caption{High dimensional test: estimated density for a $d=128$ Gaussian mixture using various Runge-Kutta schemes applied to the OU process,  with 32 time steps for RK1 and RK2, and 16 time steps for RK4 (top); convergence of density estimates obtained with varying step sizes $H$ and score errors set to match the scheme, $\varepsilon_{\rm score} = \mathcal{O}(H^p)$ (bottom).
     }
    \label{fig:high-D-error}
\end{figure}

\section{Conclusion}
In this study, we investigated the convergence of score-based generative models using the probability flow ODE framework, both theoretically and numerically. Our focus was on high-order (exponential) Runge-Kutta schemes, which are widely used in industry.
We provided theoretical convergence guarantees at the discrete level and, compared to prior work, our analysis relies on minimal assumptions. In addition to an $L^2$ score-matching error, we only assume boundedness of the first and second derivatives of the estimated score function—a condition that we verified numerically on several benchmark datasets.
Furthermore, our numerical experiments with Gaussian mixture target densities, conducted in dimensions up to $128$, empirically support our theoretical findings. The results suggest a sharp error bound of $\mathcal{O}(\varepsilon_{\rm score} + H^p)$, indicating the potential for improved estimates in terms of both dimension and score-matching error. 
Deriving tighter bounds on the dependence of error on dimensionality remains an open challenge and may require the development of new analytical techniques---an important direction for advancing the theoretical understanding of high-dimensional generative models in practical settings. Notably, dimension-free convergence results for stochastic dynamics in Gaussian mixtures have been established in \cite{li2025dimension}.  
In parallel, there have been recent significant developments in diffusion model theory that address low-dimensional stochastic dynamics \cite{li2024adapting, azangulov2024convergence, potaptchik2024linear, huang2024denoising,yakovlev2025generalization,tang2024adaptivity}, as well as in probability flow ODEs using first-order solvers \cite{liang2025low,tang2025adaptivity}. It would be interesting to investigate higher-order solvers in this setting. 
Additionally,  various sampling methods can be interpreted as gradient flows in the space of probability measures \cite{chen2023sampling, trillos2023optimization, jordan1998variational, chewi2024statistical}. For theoretical connections between the score-based approach and gradient flows, see \cite{cheng2024convergence, xie2025flow}, which represents another promising avenue for accelerating sampling methods.

\vspace{0.5in}
\noindent {\bf Acknowledgments} 
DZH is supported by National Natural Science Foundation of China (No. 62595771, 12471403, and 12288101) and the high-performance computing platform of Peking University.
JH is supported by NSF grants DMS-2331096 and DMS-2337795, and the Sloan research fellowship. 

\bibliography{References.bib}
\bibliographystyle{abbrv}

\appendix
\section{General-Order One-Step Estimates}\label{app:general-order-one-step}

This section gives the full one-step proof of \cref{t:RK_1step} for a $p$-th order standard Runge-Kutta scheme in \eqref{eq:RK-update} and \eqref{eq:RK-update-k}. Throughout the  appendix, we simplify the notation by using $C_u ^{-1} \tau \leq \sigma_{\tau} ^2 \leq C_u \tau$ for a universal constant $C_u$ when $\tau <1$, which follows directly from the definition of $\sigma_{\tau}$ for the OU process~\eqref{eq:OU}. For the general schemes \eqref{eq:RK-update-beta} and \eqref{eq:Exp-RK-update-beta}, the corresponding bounds are expressed directly in terms of the noise scale $\sigma_{\tau}$ defined in \eqref{e:general_sigma_t}.

Following \cref{s:p-th-RK}, we define the exact flow map $\phi_h(x) = x + hM_t(x;H) $, the oracle Runge-Kutta map $\baphi_h = x + h\baM_t(x;H)$, and the learned Runge-Kutta map $\dphi_h = x + h\dM_t(x;H)$. For the technical estimates below, when no confusion is possible, we abbreviate
\begin{align*}
    M_t=M_t(\cdot;H),\quad 
    \baM_t=\baM_t(\cdot;H),\quad \textrm{and}\quad
    \dM_t=\dM_t(\cdot;H) .
\end{align*}
The one-step error is decomposed into three components: the incoming distribution error $I_1$, the learned-score perturbation term $I_2$, and the oracle discretization term $I_3$:
\begin{align}
\label{e:RK2-I123-Appendix}
    \|\dvarrho_{t+H}-\varrho_{t+H}\|_{L^1}
    \leq&  \underbrace{\|(\dphi_H)_\#\dvarrho_t-(\dphi_H)_\#\varrho_t\|_{L^1}}_{I_1}
     \nonumber\\
    &+
    \underbrace{\|(\dphi_H)_\#\varrho_t-(\baphi_H)_\#\varrho_t\|_{L^1}}_{I_2} +
    \underbrace{\|(\baphi_H)_\#\varrho_t-(\phi_H)_\#\varrho_t\|_{L^1}}_{I_3}.
\end{align}
In \Cref{ssec:flow_map_properties}, we prove properties of these flow maps. In \Cref{s:proof_RK_1step}, we give detailed estimates of $I_2$ and $I_3$. In \Cref{s:proof_t:RK_1step}, we present the proof of \cref{t:RK_1step}.

\subsection{Flow Map Properties}
\label{ssec:flow_map_properties}

We first state a proposition that gathers several useful estimates for the Runge-Kutta map functions $V_{t+h}(x)$, $\sV_{t+h}(x)$, $\widetilde{M}_t(x)$, and $\overline{M}_t(x)$, 
and ensures that for all $0\leq h\leq H$, $\dphi_h$ and $\baphi_h$ are invertible.
These estimates will be repeatedly applied throughout the remaining proofs.
The proof of \Cref{p:ABCD} is deferred to \Cref{Section: Constant Convention}.

\begin{proposition}\label{p:ABCD}
Adopt \Cref{assumption:secon-moment}, \Cref{a:bounds on RK} and \Cref{a:score-derivative}, the following holds
\begin{enumerate}
    \item 
\label{i:Vbound}
We recall $D$ from \Cref{assumption:secon-moment} and $\widetilde W, \widetilde K$ from \Cref{a:score-derivative}, and introduce
\begin{equation}\begin{split}
\label{e:tbound1}
 &\widetilde W ^{(0)} 
  = \sqrt{d} \widetilde W \tau^{-1},  
  \widetilde W ^{(1)} 
  =  1+\widetilde W \tau^{-1}, 
  \\
 &\widetilde K ^{(1)}  = 1+\widetilde K \tau^{-2},  \widetilde K ^{(2)}  = \widetilde K \tau^{-3}, \widetilde L^{(1)}=1+d\widetilde K \tau^{-2},
\end{split}\end{equation}
and
\begin{equation}\begin{split}\label{e:tbound2}
            &\overline W ^{(0)} = 2\tau^{-1} \sqrt{d}D, \overline W ^{(1)} = 2\tau^{-1}, \\
            &\overline K ^{(1)}= 2D^2 \tau^{-2}, \overline K ^{(2)} = 24 D^3 \tau^{-3},\overline L ^{(1)} = 2dD^2 \tau^{-2}.
\end{split}\end{equation}
Then for any $x \in \mathbb{R}^d$ and any $h\in [0,H]$ such that $0\leq t+h\leq T-\tau$, the following holds: 
\begin{equation*}\begin{split}
&\|V_{t+h}(x)\|_2 \leq \overline W ^{(0)} + \overline W ^{(1)} \|x\|_2, \|\nabla V_{t+h}(x)\|_{op} \leq \overline L ^{(1)},
\\ 
&\|\nabla V_{t+h} (x)\|_{\infty} \leq \overline K ^{(1)}, \|\nabla ^2 V_{t+h} (x)\|_{\infty} \leq \overline K ^{(2)}.
\end{split}\end{equation*}
Similarly, we have that 
\begin{equation*}\begin{split}
&\|\sV_{t+h}(x)\|_2 \leq \widetilde W ^{(0)} + \widetilde W ^{(1)} \|x\|_2, \|\nabla \sV_{t+h}(x)\|_{op} \leq \widetilde L ^{(1)},
\\
&\|\nabla \sV_{t+h} (x)\|_{\infty} \leq \widetilde K ^{(1)},  \|\nabla ^2 \sV_{t+h} (x)\|_{\infty} \leq \widetilde K ^{(2)}.
\end{split}\end{equation*}

\item There exist $\widetilde A, \widetilde B, \widetilde C,  \widetilde D, \overline A, \overline B, \overline C, \overline D$, which depend on $\tau$, the constants $A_p,B_p$ in \Cref{a:bounds on RK}, and constants in the above \Cref{i:Vbound}, such that for any $x \in \mathbb{R}^d$ and $0\leq t\leq T-\tau$, the following holds: 
\begin{equation*}\begin{split}
\|\dM_t(x)\|_2  \leq \widetilde A + \widetilde B \|x\|_2, \|\nabla \dM_t(x) \|_{op} \leq  \widetilde C , \|\nabla \dM_t(x) \|_{\infty} \leq  \widetilde D;
\end{split}\end{equation*} 
similarly, 
\begin{equation*}\begin{split}
\|\baM_t(x)\|_2 \leq \overline A + \overline B \|x\|_2, \|\nabla \baM_t(x) \|_{op} \leq \overline C, \|\nabla \baM_t(x) \|_{\infty} \leq \overline D.
\end{split}\end{equation*}
We notice that the above statements on $\widetilde M_t$ and $\overline M_t$  can imply that 
\begin{equation*}\begin{split}\|\dphi_H(x) - x\|_2 \leq H(\widetilde A + \widetilde B \|x\|_2), \|\nabla \dphi_H(x) - \mathbb{I}_d \|_{op} \leq H \widetilde C, \|\nabla \dphi_H (x) - \mathbb{I}_d\|_{\infty} \leq H \widetilde D;
\end{split}\end{equation*}  
similarly, 
\begin{equation*}\begin{split}\|\baphi_H(x) - x\|_2 \leq H(\overline A + \overline B \|x\|_2), \|\nabla \baphi_H(x) - \mathbb{I}_d \|_{op} \leq H \overline C, \|\nabla \baphi_H (x) - \mathbb{I}_d\|_{\infty} \leq H \overline D.
\end{split}\end{equation*} 

\end{enumerate}
\end{proposition}
We emphasize that the constant $D$ is the diameter defined in \Cref{assumption:secon-moment} and is different from $\widetilde D, \overline D$ defined in \Cref{p:ABCD}. \Cref{p:ABCD} follows from a careful analysis of the true score $\nabla \log \varrho_t$ and its derivatives, and we postpone its proof to \Cref{Section: Constant Convention}.

To conclude this subsection, we state \cref{lemma: phi_h invertible}, which implies that $\dphi_h(y)$ and $\baphi_h(y)$ are diffeomorphisms on $\mathbb{R}^d$, and hence invertible. The proof relies on a matrix estimate provided in the following lemma.

\begin{lemma}\label{lemma: Growth of inverse matrix}
    Let $\gamma, \xi, H $ be positive 
  constants such that $\gamma H < 1/2$ and $2\xi H \leq 1/d$. Let $Q$ be a $d \times d $ matrix such that $\|Q - \mathbb{I}_d\|_{op} \leq \gamma H$ and $\|Q - \mathbb{I}_d\|_{\infty} \leq \xi H$, then 
        \begin{align*}
       \|Q^{-1} - \mathbb{I}_d \|_{op} \leq 2\gamma H, \quad \|Q^{-1} - \mathbb{I}_d \|_{\infty} \leq 2\xi H, 
       \quad \textrm{and} \quad
            e^{-4\xi H d} \leq  \det Q \leq e^{2\xi H d}.
    \end{align*}
\end{lemma}

\begin{proof}[Proof of \Cref{lemma: Growth of inverse matrix}]
    We denote $P = \mathbb{I}_d -Q$.
    It is easy to see that $Q^{-1} = \mathbb{I}_d + \sum_{j=1} ^{+\infty} P^j$. Hence, $\|Q^{-1} - \mathbb{I}_d \|_{op} \leq  \sum_{j=1} ^{+\infty} {(\gamma H)}^j \leq 2\gamma H$. Also, because $\xi H \leq 1/d$, we see that $\|P^2\|_{\infty} \leq d (\xi H)^2 \leq \xi H/2$, and similarly $\|P^j\|_{\infty} \leq \xi H/2^{j-1}$. It follows that $\|Q^{-1} - \mathbb{I}_d \|_{\infty} \leq 2\xi H$. 
    
    For the determinant, we see that $Q^t Q = (\mathbb{I}_d - P)^t (\mathbb{I}_d - P) = \mathbb{I}_d - (P^t + P - P^t P)$. Because $P^t + P - P^t P$ is a symmetric matrix, we assume that its eigenvalues are real numbers $\lambda_1,\lambda_2,\dots,\lambda_d$. Then,
    \begin{align*}
        \begin{split}
             | \det Q| ^2&=| \det[Q^tQ ]| = \prod_{i=1} ^d (1-\lambda_i) \\ &\leq e^{-\sum_{i=1} ^d \lambda_i} = e^{-\trace (P^t + P - P^t P)} \leq e^{2\xi d H + d^2 \xi^2 H^2} \leq e^{4\xi d H},
        \end{split}
    \end{align*}
    because $d\xi H <1$. So, $| \det Q| \leq e^{2\xi d H}$.
    Similarly, we see that $| \det[Q ^{-1}]| \leq e^{4\xi d H}$. Finally, because one can consider a continuous family of $Q_s \coloneq (1-s)\mathbb{I}_d + sQ$ for $s \in [0,1]$, and one can see that for any $s \in [0,1]$, $\|Q_s - \mathbb{I}_d\|_{\infty} \leq \xi H$. Hence, $| \det[Q_s]| \geq e^{-4\xi d H}$ for any $s \in [0,1]$. Because when $s =0$, $\det[Q_0] = 1 >0$, by continuity, we see that $\det[Q_s] >0$ for any $s \in [0,1]$. In particular, $\det Q >0$.
\end{proof}

\begin{lemma}\label{lemma: phi_h invertible}
  Assume the statements in \Cref{p:ABCD}. If $H>0$ is small enough such that $t+H<T-\tau$ and 
  \begin{align}\label{e:Hsmall}
      H (\widetilde 
    A + \widetilde 
    B + \widetilde C+\overline A + \overline B +\overline C) <1/10,
  \end{align}
  then for $0\leq h\leq H$, $\dphi_h(y)$ and $\baphi_h(y)$ are diffeomorphisms on $\mathbb{R}^d$.
\end{lemma}
\begin{proof}
    We notice that $\nabla \dphi_h (x) = {\mathbb I}_d + h \nabla \dM_t(x)$, and $\|\nabla \dM_t(x)\|_{op} \leq \widetilde C$ by \Cref{p:ABCD} for any $x \in \mathbb{R}^d$.
    Then, for any $h\leq H$, \Cref{lemma: Growth of inverse matrix}, with $Q=\nabla \wt \phi_h(x)$, implies $\|(\nabla \dphi_h (x) ) ^{-1}\|_{op} \leq 2$. 
    Again, by \eqref{e:Hsmall}, for $h\leq H$, we have that $\|\dphi_h(x)\| \geq \|x\|_2 - \|\dphi_h(x) - x\|_2 \geq \|x\|_2 - h(\widetilde A + \widetilde B \|x\|_2) \geq (\|x\|_2-1)/2$. We see that $\|\dphi_h(x)\|_2 \to + \infty$ if $\|x\|_2 \to + \infty$. Hence, by Hadamard-Cacciopoli theorem,  $\dphi_h$'s are global diffeomorphisms on $\mathbb{R}^d$.
    By the same argument, the same statement holds for $\baphi_h (x)$.
\end{proof}

\subsection{Estimates for $I_2$, $I_3$}
\label{s:proof_RK_1step}
Our main result, \Cref{t:RK_1step}, follows as a consequence of the following two propositions, which provide estimates for $I_2$  and $I_3$ on the right-hand side of \eqref{e:RK2-I123-Appendix}. 
In the following \Cref{Prop: Conclusion second error term}, we show that $I_2\leq \cE_{\rm score}$, and in \Cref{Prop: Conclusion third error term} we show that $I_3\leq \cE_{\rm RK}$ as given in \eqref{e:two_error}. The proof of \Cref{Prop: Conclusion second error term} is presented in \Cref{Section: second term conclusion}, and the proof of \Cref{Prop: Conclusion third error term} is given in \Cref{Section: third term conclusion}.

\begin{proposition}\label{Prop: Conclusion second error term} Adopt \Cref{a:bounds on RK}, the assumption \eqref{e:score_bound} in \Cref{t:RK_1step}, and assume the statements in \Cref{p:ABCD}. Denote $B:=1+A_p+B_p$. 
Assume that $H >0$ is small enough, such that $t+H\leq T-\tau$, and the following holds
\begin{align}
    \begin{split}\label{e:Hcondition}
        &\big[ (\widetilde L ^{(1)} +\overline L ^{(1)}) +  \sqrt{d}D(\widetilde W ^{(0)} + \overline W ^{(0)})+ (\widetilde W ^{(0)} + \overline W ^{(0)})^2+ dD^2(\widetilde W ^{(1)} + \overline W ^{(1)})  + 2d(\log 2d)(\widetilde W ^{(1)} + \overline W ^{(1)})^2
        \\&+ d (\widetilde K ^{(1)}+\overline K ^{(1)})   \big]\times  10^6  B^2  s^2 {4}^{s} \times  H\leq \sigma_{T-t} ^2 .
    \end{split}
\end{align}
Then, $I_2$ in \eqref{e:RK2-I123-Appendix} can be estimated by
     \begin{align*}
        \begin{split}
        I_2\leq C_{\rm disc}  \tau^{-2}  H\bigg[  d^{\frac{7}{4}}   \cdot (\varepsilon_{\rm score}(t)) ^{\frac{1}{2}} +     \tau^{-3}  {\pi}^{-\frac{d}{10}} e^{-\frac{1}{\xi H} }   \bigg] ,
        \end{split}
    \end{align*}
where $C_{\rm disc} $ is a constant depending on $B,  s, p, D,\widetilde K, \widetilde W$, and $\xi \coloneq 2\cdot 10^6 B^2s^2 4^s (\widetilde W ^{(1)} + \overline W ^{(1)})^2$.
   
\end{proposition}

\begin{proposition}\label{Prop: Conclusion third error term}
Adopt \Cref{a:bounds on RK}, and assume the statements in \Cref{p:ABCD}.
    Denote $B:=1+A_p+B_p$. There exists a large constant $C(p,s, B)>0$ depending on $p,s, B$,  and  positive constants $\gamma_1, \gamma_2 $ depending on $p$, such that if $t+H\leq T-\tau$ and $20 s 2^sBd D^2 \tau^{-2} H\leq 1$, then $I_3$ in \eqref{e:RK2-I123-Appendix} can be estimated by
        \begin{align}
            \begin{split}\label{e:third_term}
                I_3\leq C(p,s,B) H^{p+1} d^{p+1}  \tau^{-\gamma_1} {\big(D+\sqrt{ \log d}\big)}^{\gamma_2} .
            \end{split}
        \end{align}
        
\end{proposition}

\subsection{Proof of \Cref{t:RK_1step}}
\label{s:proof_t:RK_1step}
\begin{proof}[Proof of \Cref{t:RK_1step}]
    Plug the estimates in \Cref{Prop: Conclusion second error term} and \Cref{Prop: Conclusion third error term} into the decomposition \eqref{e:RK2-I123-Appendix}, we obtain that
         \begin{align}\label{e:simplify RK_1step}
        \begin{split}
        \int_{\mathbb{R}^d} |\dvarrho_{t+H}(x)-\varrho_{t+H}(x)| \de x
        \leq& 
         \int_{\mathbb{R}^d} |\dvarrho_{t}(x)-\varrho_{t}(x)| \de x 
         \\ &+ C_{\rm disc}  \tau^{-2}  H d^{\frac{7}{4}}   \cdot (\varepsilon_{\rm score}(t)) ^{\frac{1}{2}} 
        +   C_{\rm disc}  \tau^{-5} H   {\pi}^{-\frac{d}{10}} e^{-\frac{1}{\xi H} }    
        \\&+C(p,s,B) H^{p+1} d^{p+1}  \tau^{-\gamma_1} {\big(D+\sqrt{ \log d}\big)}^{\gamma_2}.
        \end{split}
    \end{align}
    Then, we further simplify the exponentially small term. Using the inequality that $e^\zeta \geq \frac{\zeta^{p}}{p!}$ when $\zeta \geq 0$, we have that $e^{-\frac{1}{\xi H} } = (e^{\frac{1}{\xi H} } )^{-1} \leq p! \times (\xi H)^{p}$. By \Cref{p:ABCD}, we know that $\xi \leq C(B,s,\widetilde W) \tau^{-2} $ for a constant $C(B,s,\widetilde W)$ depending on $B,s, \widetilde W$. Hence, we can combine the last two terms in \eqref{e:simplify RK_1step} to get a new term of the form $C(p,s,B,\widetilde W) H^{p+1} d^{p+1}  \tau^{-\gamma_1 ' } {\big(D+\sqrt{ \log d}\big)}^{\gamma_2}$, where we take $\gamma_1 ' = \max \{\gamma_1, 2p+5\}$, and $C(p, B,s,\widetilde W)$ is a new constant depending on $p, B,s, \widetilde W$. This completes the proof of \Cref{t:RK_1step}.
\end{proof}

\section{Proof of \Cref{p:ABCD}}\label{Section: Constant Convention}

This section proves the estimates for the Runge-Kutta map functions in \Cref{p:ABCD}. 

\begin{lemma}\label{lemma: Constant Convention1}
Adopt the assumptions in \Cref{p:ABCD}.  For $t+H \leq T- \tau$, \Cref{i:Vbound} of \Cref{p:ABCD} holds with the following choice of parameters 
\begin{align}\label{e:tbound3}
 \widetilde W ^{(0)} 
  = \sqrt{d} \widetilde W \tau^{-1},  \quad 
  \widetilde W ^{(1)} 
  =  1+\widetilde W \tau^{-1},\quad
   \widetilde K ^{(1)}  = 1+\widetilde K \tau^{-2}, \quad  \widetilde K ^{(2)}  = \widetilde K \tau^{-3}, \quad \widetilde L^{(1)}=1+d\widetilde K \tau^{-2},
\end{align}
and
\begin{align}\label{e:tbound4}
            \overline W ^{(0)} = 2\tau^{-1} \sqrt{d}D, \quad \overline W ^{(1)} = 2\tau^{-1}, \quad \overline K ^{(1)}= 2D^2 \tau^{-2}, \quad \overline K ^{(2)} = 24 D^3 \tau^{-3},\quad \overline L ^{(1)} = 2dD^2 \tau^{-2}.
        \end{align}
\end{lemma}
\begin{proof}
    The first claim \eqref{e:tbound3} is a consequence of \Cref{a:score-derivative}. For example $\|\nabla \widehat V_{t+h}(x)\|_{op}\leq 1+\|\nabla \widehat s_{t+h}(x)\|_{op}\leq 1+d \max_{1\leq l,j\leq d}|\partial_{x_l} s_t^{(j)}(x)|\leq 1+d\widetilde K/\tau^2$. 
    The claim \eqref{e:tbound4} follows from the estimates on derivatives of $\log q_t(y)$ in Lemma~\ref{Lemma: Hessian estimates}. 
    For example, we see that $|V^{(i)}_{t+h} (x)| \leq \tau^{-1} (|x_i| + D)$, so $\|V_{t+h} (x)\|_2 \leq 2\tau^{-1} (\sqrt{d}D+\|x\|_2)$, which gives that we can take $\overline W ^{(0)} = 2\tau^{-1} \sqrt{d}D$ and $\overline W ^{(1)} =2\tau^{-1}$. Similarly, $\overline K ^{(1)}= 2D^2 \tau^{-2}$, $\overline K ^{(2)} = 24D^3 \tau^{-3}$, and $\overline L ^{(1)}= 2dD^2 \tau^{-2}$.
\end{proof}

In the rest of this section we prove the second part of \Cref{p:ABCD}. We recall  $\dM_t(x)$ and $\baM_t(x)$ from \eqref{e:deftMt} and \eqref{e: score error p RK method} 
\begin{align}\label{e: score error p RK method copy}
\dM_{t}(x)= \sum_{j=1}^s b_j \widetilde k_j(x), \quad \dphi_H(x)=x+H \dM_{t}(x),\quad \baM_{t} \coloneq \sum_{j=1}^s b_j \overline k_j(x) , \quad \baphi_H(x)=x+H \baM_{t}(x),
\end{align}
with 
\begin{equation}\label{e: p RK interpolation terms copy}
\begin{split}
    & \widetilde k_1 = \sV_{t + c_1 H}\bigl(x\bigr), \quad \overline k_1 = V_{t + c_1 H}\bigl(x\bigr),\\
    &\widetilde k_2 = \sV_{t + c_2 H}\bigl(x + (a_{21} \widetilde k_1) H\bigr), \quad \overline k_2 = V_{t + c_2 H}\bigl(x + (a_{21 } \overline k_1) H\bigr),\\
    &\widetilde k_3 = \sV_{t + c_3 H}\bigl(x + (a_{31} \widetilde k_1 + a_{32} \widetilde k_2) H\bigr), \quad \overline k_3 = V_{t + c_3 H}\bigl(x + (a_{31} \overline k_1 + a_{32} \overline k_2) H\bigr),\\
    &\qquad\qquad\qquad\vdots
    \\
    &\widetilde k_s = \sV_{t + c_s H}\bigl(x + (a_{s1} \widetilde k_1 +  \cdots + a_{s,s-1} \widetilde k_{s-1}) H\bigr), \quad \overline k_s = V_{t + c_s H}\bigl(x + (a_{s1} \overline k_1 +  \cdots + a_{s,s-1} \overline k_{s-1}) H\bigr).
\end{split}
\end{equation}
The coefficients $b_j = b_j(H)$ and $a_{jk} = a_{jk}(H)$ are fixed, and have bounds 
    \begin{align*}
        \max_{1\leq j,k\leq s} | a_{jk}|\leq A_p,\quad \max_{1\leq j\leq s}  |b_j|\leq B_p,
    \end{align*}
according to \eqref{e:RKbound}.

In the following lemmas, we collect estimates for $\widetilde k_i$ and $\overline k_i$ in \eqref{e: p RK interpolation terms copy}, and their derivatives.

\begin{lemma}\label{lemma: infinity norm pth order growth in RK} 
Adopt the assumptions and notations in \Cref{p:ABCD} and \Cref{lemma: Constant Convention1}. 
 For those $\widetilde k_i(x)$'s and $\overline k_i(x)$'s as in \eqref{e: p RK interpolation terms copy}, we define 
        \begin{align*}
            \begin{split}
                &\|\nabla \widetilde k_i(x)\|_{\infty} := \sup_{1\leq l,q \leq d} |\partial_l \widetilde k_i ^{(q)} (x)|, \ \|\nabla ^2 \widetilde k_i(x)\|_{\infty} := \sup_{1\leq l,q,r \leq d} |\partial_{lr}^2 \widetilde k_i ^{(q)} (x)|,
                \\  &\|\nabla \overline k_i(x)\|_{\infty} := \sup_{1\leq l,q \leq d} |\partial_l \overline k_i ^{(q)} (x)|, \ \|\nabla ^2 \overline k_i(x)\|_{\infty}: = \sup_{1\leq l,q,r \leq d} |\partial_{lr}^2 \overline k_i ^{(q)} (x)| .
            \end{split}
        \end{align*}
        We have that, for any $x\in \mathbb{R}^d$, 
        \begin{align*}
                &\sum_{i=1} ^s \|\nabla \widetilde k_i(x) \|_{\infty} \leq s \widetilde K ^{(1)} {(1+ \widetilde K ^{(1)} A_p H d)}^{s-1}, \ 
                 \sum_{i=1} ^s \|\nabla \overline k_i(x) \|_{\infty} \leq s \overline K ^{(1)} {(1+ \overline K ^{(1)} A_p H d)}^{s-1},
        \end{align*}
    and 
         \begin{align*}
            \sum_{i=1} ^s \|\nabla ^2 \widetilde k_i(x) \|_{\infty} \leq 2 \widetilde K ^{(2)} s^3 {(1+ \widetilde K ^{(1)}A_p H d)}^{2s} ,
                \ \sum_{i=1} ^s \|\nabla ^2 \overline k_i(x) \|_{\infty} \leq 2 \overline K ^{(2)} s^3 {(1+ \overline K ^{(1)}A_p H d)}^{2s} .
        \end{align*}
As a corollary, we see that $\|\nabla ^2 \dphi_H \|_{\infty} \coloneq \sup_{x \in \mathbb{R}^d} \|\nabla ^2 \dphi_H (x) \|_{\infty}$ and $\|\nabla ^2 \baphi_H \|_{\infty} \coloneq \sup_{x \in \mathbb{R}^d} \|\nabla ^2 \baphi_H (x) \|_{\infty}$ satisfy that
    \begin{align*}
            \|\nabla ^2 \dphi_H \|_{\infty} \leq 2 H B_p\widetilde K ^{(2)} s^3 {(1+ \widetilde K ^{(1)}A_p H d)}^{2s} ,
                \ \|\nabla ^2 \baphi_H \|_{\infty} \leq 2 H B_p \overline K ^{(2)} s^3 {(1+ \overline K ^{(1)}A_p H d)}^{2s} .
        \end{align*}
Also, if $(\widetilde K ^{(1)} + \overline K ^{(1)}) A_p H d s 2^s \leq 1$, then for each $i \in \llbracket 1 , s \rrbracket$,
    \begin{align*}
        \|\nabla  \widetilde k_i(x) \|_{\infty} \leq 2 \widetilde K ^{(1)}, \quad \|\nabla  \overline k_i(x) \|_{\infty} \leq 2 \overline K ^{(1)}.
    \end{align*}
\end{lemma}

\begin{proof}[Proof of \Cref{lemma: pth order growth in RK}]
    We only prove the statement for $\widetilde k_i(x)$'s and the proof for $\overline k_i(x)$'s is similar. We first notice that, according to \Cref{i:Vbound} of \Cref{p:ABCD}, for any $x\in \bR^d$ and $0\leq h\leq H$, $\|\widetilde V_{t+h}(x)\|_2 \leq \widetilde W ^{(0)} + \widetilde W ^{(1)} \|x\|_2 $, and 
 $\| \nabla \widetilde V_{t+h}(x) \|_{op} \leq \widetilde L ^{(1)}$. Hence, $\|\widetilde k_1(x)\|_2 \leq \widetilde W ^{(0)} + \widetilde W ^{(1)} \|x\|_2$, $\|\nabla \widetilde k_1(x)\|_{op} \leq \widetilde L ^{(1)}$ according to \eqref{eq:RK-update-k}. By the relation that $\widetilde k_i = \sV_{t + c_i H}\bigl(x + (a_{i1} \widetilde k_1 + a_{i2} \widetilde k_2 + \cdots + a_{i,i-1} \widetilde k_{i-1}) H\bigr)$, we obtain that
        \begin{align*}
            \begin{split}
                &\|\widetilde k_i (x)\|_{2} \leq \widetilde W ^{(0)} + \widetilde W ^{(1)}\|x\|_{2}+ \widetilde W ^{(1)} A_p H (\|\widetilde k_1(x)\|_{2} + \|\widetilde k_2(x)\|_{2} + \cdots + \|\widetilde k_{i-1}(x)\|_{2}),
                \\  & \| \nabla \widetilde k_i (x)\|_{op} \leq \widetilde L ^{(1)} (1+ A_p H (\|\widetilde \nabla k_1(x)\|_{op} + \|\widetilde \nabla k_2(x)\|_{op} + \cdots + \|\widetilde \nabla k_{i-1}(x)\|_{op})).
            \end{split}
        \end{align*}
    If we define $T_j \coloneq \sum_{i=1} ^j \|\widetilde k_i(x) \|_{2}$, $T_j ' \coloneq \sum_{i=1} ^j \|\nabla \widetilde k_i(x) \|_{op}$, we see that
    \begin{align*}
        \begin{split}
            &T_j \leq (\widetilde W ^{(0)} + \widetilde W ^{(1)} \|x\|_{2}) + (1+ \widetilde W ^{(1)} A_p  H) T_{j-1}, \\ & T_j ' \leq \widetilde L ^{(1)}  + (1+ \widetilde L ^{(1)} A_p  H) T_{j-1} ' .
        \end{split}
    \end{align*} 
    Hence,
    \begin{align*}
        \begin{split}
            &T_s \leq {(1+ \widetilde W ^{(1)}A_p H)}^{s-1} \big[(s-1) {(\widetilde W ^{(0)} + \widetilde W ^{(1)}\|x\|_{2})} +  (\widetilde W ^{(0)} + \widetilde W ^{(1)} \|x\|_2) \big] ,
            \\  & T_s ' \leq {(1+ \widetilde L ^{(1)}A_p H)}^{s-1} \big[(s-1) {\widetilde L ^{(1)}} +  {\widetilde L ^{(1)}} \big].
        \end{split}
    \end{align*}
Also, if $\widetilde W ^{(1)} A_p H s 2^s \leq 1$, we see that for each $i \in \llbracket 1, s \rrbracket$,
    \begin{align*}
        \begin{split}
            \|\widetilde k_i (x)\|_{2} &\leq \widetilde W ^{(0)} + \widetilde W ^{(1)}\|x\|_{2}+ \widetilde  W ^{(1)} A_p H (\|\widetilde  k_1(x)\|_{2} + \|\widetilde  k_2(x)\|_{2} + \cdots + \|\widetilde  k_{i-1}(x)\|_{2})
            \\  &\leq \widetilde W ^{(0)} + \widetilde W ^{(1)}\|x\|_{2}+ \widetilde  W^{(1)} A_p H s2^s \big[\widetilde W ^{(0)} + \widetilde W ^{(1)} \|x\|_2 \big] \leq 2(\widetilde W ^{(0)} + \widetilde W ^{(1)}\|x\|_{2}),
        \end{split}
    \end{align*}
and similarly, if $\widetilde L ^{(1)} A_p H s 2^s \leq 1$, 
    \begin{align*}
    \|\nabla \widetilde k_i(x) \|_{op} \leq 2\widetilde L ^{(1)}.
    \end{align*}

\end{proof}

\begin{proof}[Proof of \Cref{lemma: infinity norm pth order growth in RK}]\label{s:prove_k}
    We only prove the statement for $\widetilde k_i(x)$'s and the proof for $\overline k_i(x)$'s are similar. 
    By the relation that $\widetilde k_i = \sV_{t + c_i H}\bigl(x + (a_{i1} \widetilde k_1 + a_{i2} \widetilde k_2 + \cdots + a_{i,i-1} \widetilde k_{i-1}) H\bigr)$, we obtain that
        \begin{align*}
    \partial_{l_1} \widetilde k_i ^{(q)} = \sum_{l_2 = 1} ^{d} (\partial_{l_2}\sV_{t + c_i H} ^{(q)})[\delta_{l_1 l_2} + H(a_{i1} \partial_{l_1}\widetilde k_1 ^{(l_2)} + a_{i2} \partial_{l_1} \widetilde k_2 ^{(l_2)} + \cdots + a_{i,i-1} \partial_{l_1} \widetilde k_{i-1} ^{(l_2)}) ],
        \end{align*}
and 
    \begin{align*}
        \begin{split}
            \partial_{l_1 l_3} ^2 \widetilde k_i ^{(q)} &= \bigg( \sum_{l_2  = 1} ^{d} \sum_{l_4  = 1} ^{d} (\partial_{l_2 l_4} ^2 \sV_{t + c_i H} ^{(q)})[\delta_{l_1 l_2} + H(a_{i1} \partial_{l_1}\widetilde k_1 ^{(l_2)} + \cdots + a_{i,i-1} \partial_{l_1} \widetilde k_{i-1} ^{(l_2)}) ]
            \\ &\quad \cdot [\delta_{l_3 l_4} + H(a_{i1} \partial_{l_3}\widetilde k_1 ^{(l_4)} + \cdots + a_{i,i-1} \partial_{l_3} \widetilde k_{i-1} ^{(l_4)}) ] \bigg)
           \\   &\quad  + \sum_{l_2 = 1} ^{d} (\partial_{l_2}\sV_{t + c_i H} ^{(q)})[ H(a_{i1} \partial_{l_1 l_3} ^2\widetilde k_1 ^{(l_2)} \cdots + a_{i,i-1} \partial_{l_1 l_3} ^2 \widetilde k_{i-1} ^{(l_2)})].
        \end{split}
    \end{align*}
Hence,
    \begin{align}\label{e:iteration growth}
        \|\nabla \widetilde k_i(x)\|_{\infty} \leq \widetilde K ^{(1)} [1+ A_p H d (\|\nabla \widetilde k_1(x)\|_{\infty} +\|\nabla \widetilde k_2(x)\|_{\infty} +\cdots+ \|\nabla \widetilde k_{i-1}(x)\|_{\infty}) ].
    \end{align}
If we define $T_j \coloneq \sum_{i=1} ^{j} \|\nabla \widetilde k_i(x)\|_{\infty}$, we see that 
    \begin{align*}
        T_j \leq \widetilde K ^{(1)} + (1+ \widetilde K ^{(1)} A_p H d) T_{j-1}.
    \end{align*}
Also, $\|\nabla \widetilde k_1(x)\|_{\infty} \leq \widetilde K ^{(1)}$.
So, 
    \begin{align}\label{e:iteration result}
        T_s \leq s \widetilde K ^{(1)} {(1+ \widetilde K ^{(1)} A_p H d)}^{s-1}.
    \end{align}
If $\widetilde K ^{(1)} A_p H d s 2^s \leq 1$, then the right hand side of \eqref{e:iteration result} is bounded by $s2^s \widetilde K^{(1)}$. Substituting this bound into the right hand side of \eqref{e:iteration growth}, we obtain that for each $i \in \llbracket 1 , s \rrbracket$, $ \|\nabla  \widetilde  k_i(x) \|_{\infty} \leq 2 \widetilde K ^{(1)}$.

Similarly,
    \begin{align*}
        \|\nabla^2 \widetilde k_i(x)\|_{\infty} \leq \widetilde K ^{(2)} {(1+ A_p H d T_{i-1})}^{2} + \widetilde K ^{(1)}A_p H d (\|\nabla ^2 \widetilde k_1(x)\|_{\infty} +\cdots+ \|\nabla  ^2 \widetilde k_{i-1}(x)\|_{\infty}),
    \end{align*}
and then we define $T_j ' \coloneq \sum_{i=1} ^{j} \|\nabla ^2 \widetilde k_i(x)\|_{\infty}$. We see that
    \begin{align*}
        T_j ' \leq \widetilde K ^{(2)} {(1+ A_p H d T_{j-1})}^{2} + (1+ \widetilde K ^{(1)}A_p H d) {T_{j-1} '}.
    \end{align*}
Also, $\|\nabla ^2 \widetilde k_1(x)\|_{\infty} \leq \widetilde K ^{(2)}$. So,
    \begin{align*}
        T_s ' \leq \widetilde K ^{(2)} \sum_{i=1} ^{s} {(1+ A_p H d T_{i-1})}^{2} {(1+ \widetilde K ^{(1)}A_p H d)}^{s-i} \leq 2 \widetilde K ^{(2)} s^3 {(1+ \widetilde K ^{(1)}A_p H d)}^{2s}.
    \end{align*}
For the last inequality, we regard $\xi \coloneq \widetilde K ^{(1)}A_p H d$. Then, 
    \begin{align*}
        \begin{split}
            &\sum_{i=1} ^s {(1+(i-1)\xi(1+\xi)^{i-1})}^2 {(1+\xi)}^{s-i} \leq 2 \sum_{i=1} ^s \big[{(1+\xi)}^{s-i} + {(i-1)}^2 \xi^2 {(1+\xi)}^{s+i-2}\big]
            \\  &\leq 2s\big[{(1+\xi)}^{s-1} + {(s-1)}^2 \xi^2 {(1+\xi)}^{2s-2}\big] \leq 2s {(1+\xi)}^{2s-2} (1+s^2 \xi^2) \leq 2 s^3 {(1+\xi)}^{2s} .
        \end{split}
    \end{align*}

\end{proof}

\begin{lemma}\label{lemma: pth order growth in RK}
   Adopt the assumptions and notations in \Cref{p:ABCD} and \Cref{lemma: Constant Convention1}.   For those $\widetilde k_i(x)$'s and $\overline k_i(x)$'s as in \eqref{e: p RK interpolation terms copy}, we have that 
        \begin{align*}
            \begin{split}
                &\sum_{i=1} ^s \|\widetilde k_i(x) \|_{2} \leq s{(1+ \widetilde W ^{(1)}A_p H)}^{s-1} \big[   \widetilde W ^{(0)} + \widetilde W ^{(1)} \|x\|_2 \big] ,
                \\ &\sum_{i=1} ^s \|\overline k_i(x) \|_{2} \leq s {(1+  \overline W ^{(1)}A_p H)}^{s-1} \big[    \overline W^{(0)} + \overline W ^{(1)} \|x\|_2 \big] ,
            \end{split}
        \end{align*}
    and 
         \begin{align*}
            \begin{split}
                &\sum_{i=1} ^s \|\nabla \widetilde k_i(x) \|_{op} \leq s{(1+ \widetilde L ^{(1)}A_p H)}^{s-1}   \widetilde L ^{(1)} ,
                \\ &\sum_{i=1} ^s \|\nabla \overline k_i(x) \|_{op} \leq s {(1+ \overline L ^{(1)}A_p H)}^{s-1} \overline L ^{(1)}.
            \end{split}
        \end{align*}
Moreover, if $(\widetilde W ^{(1)} + \overline W ^{(1)}) A_p H s 2^s \leq 1$, then for each $i \in \llbracket 1 , s \rrbracket$,
    \begin{align*}
        \| \widetilde k_i(x) \|_2 \leq 2(\widetilde W ^{(0)} + \widetilde W ^{(1)}\|x\|_{2}), \quad  \| \overline k_i(x) \|_2 \leq 2(\overline W ^{(0)} + \overline W ^{(1)}\|x\|_{2});
    \end{align*}
if $(\widetilde L ^{(1)} + \overline L ^{(1)}) A_p H s 2^s \leq 1$, then for each $i \in \llbracket 1 , s \rrbracket$,
    \begin{align*}
    \|\nabla \widetilde k_i(x) \|_{op} \leq 2\widetilde L ^{(1)}, \quad \|\nabla \overline k_i(x) \|_{op} \leq 2\overline L ^{(1)}.
    \end{align*}
\end{lemma}

On the other hand, using the explicit estimates for $V_t(x)$ as we did in \Cref{Lemma: Hessian estimates}, we have the following estimates.
\begin{lemma}\label{lemma: infinity norm for k_i(x)}
Adopt the notations in \Cref{p:ABCD}. For those $\overline k_i(x)$'s as in \eqref{e: p RK interpolation terms copy}, we have that     
    \begin{align*}
        \sum_{i=1} ^s \|\overline k_i (x)\|_{\infty} \leq s\tau^{-1}{(1+2\tau ^{-1}A_p H)}^{s-1} \big[D + 2\|x\|_{\infty} \big],
    \end{align*}
    and
    \begin{align*}
        \sum_{i=1} ^s \|\overline k_i (x)\|_{2} \leq 2s\tau^{-1}{(1+2\tau ^{-1}A_p H)}^{s-1} \big[\sqrt{d}D + \|x\|_{2} \big].
    \end{align*}
If we take that $H$ such that $H A_p 2^s s \tau^{-1} \leq 1$, then for each $i=1,2,\dots, s$,
    \begin{align*}
        \|\overline k_i (x)\|_{\infty} \leq 4 \tau^{-1}(\|x\|_{\infty} +D), \quad \|\overline k_i (x)\|_{2} \leq 4\tau ^{-1}(\|x\|_{2} + \sqrt{d} D).
    \end{align*}
\end{lemma}
\begin{proof}[Proof of \Cref{lemma: infinity norm for k_i(x)}]
    According to \Cref{Lemma: Hessian estimates}, we see that for any $t$ such that $T-t \geq \tau$, 
        \begin{align*}
            \|V_t(x)\|_{\infty} \leq (1+\sigma_{T-t} ^{-2})\|x\|_{\infty} + D\sigma_{T-t} ^{-2} \leq 2\tau ^{-1}\|x\|_{\infty} + D\tau ^{-1}.
        \end{align*}
    So, for any $i \in \llbracket 1, s \rrbracket$, by the relation that $\overline k_i = V_{t + c_i H}\bigl(x + (a_{i1} \overline k_1 + a_{i2} \overline k_2 + \cdots + a_{i,i-1} \overline k_{i-1}) H\bigr)$, we obtain that
        \begin{align*}
            \|\overline k_i (x)\|_{\infty} \leq \tau^{-1}(D + 2\|x\|_{\infty}) + 2\tau ^{-1}A_p H(\|\overline k_1 (x)\|_{\infty}+ \|\overline k_2 (x)\|_{\infty} + 
 \cdots + \|\overline k_{i-1} (x)\|_{\infty}).
        \end{align*}
By a similar strategy we used in \Cref{lemma: pth order growth in RK}, we see that
    \begin{align*}
        \sum_{i=1} ^s \|\overline k_i (x)\|_{\infty} \leq s\tau^{-1}{(1+2\tau ^{-1}A_p H)}^{s-1} \big[D + 2\|x\|_{\infty} \big].
    \end{align*}

Similarly, by \Cref{Lemma: Hessian estimates}, for any $t$ such that $T-t \geq \tau$, we see that
\begin{align*}
            \|V_t(x)\|_{2}  \leq 2\tau ^{-1}(\|x\|_{2} + \sqrt{d} D),
        \end{align*}
and then
    \begin{align*}
        \sum_{i=1} ^s \|\overline k_i (x)\|_{2} \leq 2s\tau^{-1}{(1+2\tau ^{-1}A_p H)}^{s-1} \big[\sqrt{d}D + \|x\|_{2} \big].
    \end{align*}
\end{proof}

We finally need to estimate those $\widetilde A, \widetilde B, \widetilde C, \widetilde D, \overline A, \overline B, \overline C, \overline D$ in \Cref{p:ABCD}. The second part of \Cref{p:ABCD} is a consequence of the following lemma. 
\begin{lemma}\label{lemma: Constant Convention}
Adopt the assumptions and notations in \Cref{p:ABCD} and \Cref{lemma: Constant Convention1}, and assume that $H >0$ satisfies that 
    \begin{align*}
        H \cdot A_p (\widetilde L ^{(1)} + \overline L^{(1)}+\widetilde W ^{(1)} + \overline W ^{(1)} + d(\widetilde K ^{(1)}  + \overline K ^{(1)} ) ) \leq 1,
    \end{align*}
then we have the following estimates:
    \begin{align*}
        \begin{split}
            &\widetilde A \leq B_p s 2^s \widetilde W ^{(0)}, \quad \overline A \leq B_p s 2^s \overline W ^{(0)}, \quad \widetilde B  \leq B_p s 2^s \widetilde W ^{(1)} , \quad \overline B  \leq B_p s 2^s \overline W ^{(1)} , \\ &\widetilde C  \leq B_p s 2^s \widetilde L ^{(1)} , \quad \overline C  \leq B_p s 2^s \overline L ^{(1)} , \quad \widetilde D \leq B_p s 2^s \widetilde K ^{(1)}, 
            \quad \overline D \leq B_p s 2^s \overline K ^{(1)}, 
            \\  &  \|\nabla ^2 \dphi_H \|_{\infty} \leq 2 H B_p s^3 4^s \widetilde K ^{(2)} , \quad \|\nabla ^2 \baphi_H \|_{\infty} \leq 2 H B_p s^3 4^s \overline K ^{(2)}.
        \end{split}
    \end{align*}
\end{lemma}
\begin{proof}
    According to Lemma~\ref{lemma: infinity norm pth order growth in RK} and  Lemma~\ref{lemma: pth order growth in RK}, we see that 
    \begin{align*}
        &\widetilde A = \|\dM_t(0)\|_2 \leq B_p s{(1+ \widetilde W ^{(1)}A_p H)}^{s-1} \widetilde W ^{(0)}  \leq B_p s 2^s \widetilde W ^{(0)},\\
        &\widetilde B \leq B_p s{(1+ \widetilde W ^{(1)}A_p H)}^{s-1} \widetilde W ^{(1)}  \leq B_p s 2^s \widetilde W ^{(1)},\\
    &
         \widetilde C = \|\nabla \dM_t(x)\|_{op} \leq B_p s 2^s \widetilde L ^{(1)},\\
    &
        \widetilde D = \|\nabla \dM_t(x)\|_{\infty} = B_p s \widetilde K ^{(1)} {(1+ \widetilde K ^{(1)} A_p H d)}^{s-1} \leq B_p s 2^s \widetilde K ^{(1)},\\
   &
        \|\nabla ^2 \dphi_H \|_{\infty} \leq 2 H B_p\widetilde K ^{(2)} s^3 {(1+ \widetilde K ^{(1)}A_p H d)}^{2s} \leq 2 H B_p s^3 4^s \widetilde K ^{(2)}. 
    \end{align*}
    We can similarly obtain the estimates for $\overline A, \overline B, \overline C, \overline D, \|\nabla ^2 \baphi_H \|_{\infty}$.
\end{proof}

\section{Proof of \Cref{Prop: Conclusion second error term}}\label{Section: second term conclusion}
This section proves the learned-score perturbation estimate for $I_2$. \Cref{ssec:J1-proof} estimates the density-gradient term $J_1$, and \Cref{ssec:J_2_proof} estimates the Jacobian term $J_2$. \Cref{ssec:I2-estimate} combines these two estimates and reduces the resulting map differences to score-matching quantities, while \Cref{ssec:second-error-proof} completes the proof of \Cref{Prop: Conclusion second error term}.

In our proofs, $H$ will also be chosen small such that for any $|h|<H$ and any $x \in \mathbb{R}^d$, as matrices, $\|(\nabla \baphi_h (x) ) ^{-1}\|_{op} \leq 2$ and $\|(\nabla \dphi_h (x) ) ^{-1}\|_{op} \leq 2$. We already illustrated how to choose suitably small $H$ in \Cref{lemma: phi_h invertible}. Similarly, we can assume that when $|h|<H$, as matrices, $\|(\nabla \baphi_h ^{-1})(y)\|_{op} \leq 2$ and $\|(\nabla \dphi_h ^{-1})(y)\|_{op} \leq 2$. This is because, for example, $\|(\nabla \dphi_h ^{-1})(y)\|_{op} = \|(\nabla \dphi_{h}  (\dphi_h ^{-1}(y) ) )^{-1}\|_{op} \leq 2$ for any $y \in \mathbb{R}^d$. Also, $\det[ \nabla (\baphi_h ^{-1}(x))]>0$ and $\det[ \nabla (\dphi_h ^{-1}(x))]>0$ for any $x \in \mathbb{R}^d$ when $|h| <H$.

We recall $I_2$ from \eqref{e:RK2-I123-Appendix}
    \begin{align}\label{e:defI2copy}
        I_2=\int_{\mathbb{R}^d} | \varrho_{t} (\dphi_H ^{-1}(x)) \cdot | \det[ \nabla (\dphi_H ^{-1}(x))]| -  \varrho_{t} (\baphi_H ^{-1}(x)) \cdot | \det[ \nabla (\baphi_H ^{-1}(x))]| | \de x.
    \end{align}
We consider the interpolation map, which interpolates $\dphi_H ^{-1}(x)$ and $\baphi_H ^{-1}(x)$
\begin{align}\label{eq:I2 interpolation phi_s}
    \varphi_s(x) \coloneq s\dphi_H ^{-1}(x) + (1-s)\baphi_H ^{-1}(x), \quad s\in [0,1].
\end{align} 
So, $\varphi_0(x) = \baphi_H ^{-1}(x)$ and $\varphi_1(x)=\dphi_H ^{-1}(x)$. We first need to collect some growth estimates for $\varphi_s(x)$ similar to the second part of  \Cref{p:ABCD}.
\begin{lemma}\label{lemma: Growth of interpolation maps}
    Adopt the assumptions and notations in \Cref{p:ABCD}. Take $H$ small enough such that
        \begin{align*}
            H  (\widetilde B + \overline B + \widetilde C + \overline C + d(\widetilde D + \overline D)) <1/4.
        \end{align*}
   Let $\bm\al = 2 (\widetilde A + \overline A)$, $\bm\beta=2(\widetilde B + \overline B)$, $\bm \gamma=2(\widetilde C + \overline C)$, and $\bm\xi = 2 (\widetilde D + \overline D)$.
    Then, for any $s\in [0,1]$, $\varphi_s(x)$ is a diffeomorphism on $\mathbb{R}^d$, and the following holds $ \| \varphi_s(x) -x\|_2 \leq   H (\bm\alpha+ \bm\beta\|x\|_2)$, $\|\nabla \varphi_s(x) -\mathbb{I}_d \|_{\op} \leq \bm\gamma H$, $\|\nabla \varphi_s(x) -\mathbb{I}_d \|_{\infty} \leq \bm\xi H$. 
\end{lemma} 
\begin{proof}
    We only need to find out suitable estimates for $\varphi_0$ and $\varphi_1$, because $\varphi_s = s\varphi_1 + (1-s)\varphi_0$. We do this for $\varphi_1(x) = \dphi_H ^{-1}(x)$ first and the computations for $\varphi_0(x) = \baphi_H ^{-1}(x)$ are similar.

     By \Cref{p:ABCD}, $\|\dphi_H(x) - x\|_2 \leq H(\widetilde A + \widetilde B \|x\|_2)$. So, $\| x - \dphi_H ^{-1}(x) \|_2 \leq H(\widetilde A + \widetilde B \|\dphi_H ^{-1}(x)\|_2) \leq H(\widetilde A + \widetilde B \|\dphi_H ^{-1}(x) - x\|_2 + \widetilde B \|x \|_2)$. If we choose $H \widetilde B < 1/2$ at the beginning, then $\| x - \dphi_H ^{-1}(x) \|_2 \leq 2H (\widetilde A + \widetilde B \|x\|_2)\leq H(\bm\alpha+\bm\beta\|x\|_2)$, for $\bm\alpha\geq 2 (\widetilde A + \overline A)$ and $\bm\beta\geq2(\widetilde B + \overline B)$. 

     By \Cref{lemma: Growth of inverse matrix}, we also notice the following estimates $\|\nabla \dphi_H ^{-1}(x) -\mathbb{I}_d \|_{\op} = \big\|{[(\nabla \dphi_H) (\dphi_H ^{-1}(x))]}^{-1} -\mathbb{I}_d \big\|_{\op} \leq 2 \big\|{[(\nabla \dphi_H) (\dphi_H ^{-1}(x))]} -\mathbb{I}_d \big\|_{\op}  \leq 2H \widetilde C\leq 2\bm\gamma H$, for $\bm\gamma \geq 2 (\widetilde C + \overline C) $. Similarly, $\|\nabla \dphi_H ^{-1}(x) -\mathbb{I}_d \|_{\infty} \leq 2H \widetilde D\leq \bm\xi H$ provided  $\bm\xi= 2 (\widetilde D + \overline D) $ and $2\bm\xi< 1/d$.
    
\end{proof}

Because we can choose $H$ small at the beginning such that all $\varphi_s$'s in \eqref{eq:I2 interpolation phi_s} are diffeomorphisms and $\det[ \nabla (\varphi_s(x))] >0$ for $s \in [0,1]$. So, $I_2$ (recall from \eqref{e:defI2copy}) becomes
    \begin{align}\label{e: interpolation error in second term}
        \begin{split}
            &\int_{\mathbb{R}^d} | \varrho_t(\varphi_0(x)) \cdot  \det[ \nabla (\varphi_0(x))] - \varrho_t(\varphi_1(x)) \cdot  \det[ \nabla (\varphi_1(x))]| \de x
            \\  &\leq \int_0 ^1 \int_{\mathbb{R}^d} \bigg| \frac{d}{ds}  \bigg(\varrho_t(\varphi_s(x)) \cdot  \det[ \nabla (\varphi_s(x))] \bigg)\bigg| \de x \de s
           \leq J_1+J_2.
        \end{split}
    \end{align}
where 
\begin{align}\begin{split}\label{e:defJ1J2}
    &J_1:=\int_0 ^1 \int_{\mathbb{R}^d} \bigg| (\nabla \varrho_t) (\varphi_s(x)) \cdot (\varphi_1(x) - \varphi_0(x)) \det[ \nabla (\varphi_s(x))]\bigg| \de x \de s. \\
    &J_2:=\int_0 ^1 \int_{\mathbb{R}^d}\bigg| \varrho_t(\varphi_s(x)) \cdot \trace \bigg( \bigg( \nabla [\varphi_1(x) - \varphi_0(x)] \bigg) \cdot {(\nabla \varphi_s(x) )}^{-1} \bigg) \det[ \nabla \varphi_s(x)] \bigg| \de x \de s .
\end{split}\end{align}

\subsection{Estimates of $J_1$}
\label{ssec:J1-proof}
To proceed, we need the following argument to compare two expectations with respect to $\varrho_t$. 

\begin{lemma}\label{lemma: score error cut-off general scheme}
    Assume that $\psi_1,\psi_2$ are two diffeomorphisms on $\mathbb{R}^d$, such that $ \| \psi_1(x) -x\|_2 \leq  H (\bm\al_1 + \bm\beta_1\|x\|_2)$, $\|\nabla\psi_1(x)-x\|_\infty\leq \bm\xi_1 H$,   $ \| \psi_2(x) -x\|_2 \leq  H (\bm\al_2+  \bm\beta_2\|x\|_2)$ and  $\|\nabla\psi_2(x)-x\|_\infty\leq \bm\xi_2 H$ with positive constants $\bm\al_1, \bm\al_2, \bm\beta_1, \bm\beta_2, \bm\xi_1,\bm\xi_2$. We also assume that $Z(x)$ is a map from $\mathbb{R}^d \to \mathbb{R}^d$ such that $\|Z(x)\|_2 \leq A + B\|x\|_2 ^m$, where $m \in \mathbb{Z}_+$. Let $\eta >1$ be a constant. Then, if $H$ satisfies that
        \begin{align}\label{e: H in cut-off scheme}
            100 H \big[(\bm\al_1+\bm\al_2)(\sqrt{d}D + \eta \sigma_{T-t})+(\bm\beta_1+\bm\beta_2)(\sqrt{d}D + \eta \sigma_{T-t})^2+d(\bm\xi_1+\bm\xi_2) \big]  \leq \sigma_{T-t} ^2,
        \end{align}
we have that
        \begin{align*}
            \int_{\mathbb{R}^d}    \varrho_{t}(\psi_1(x)) \big\| Z(x)\big\|_2  \de x
    &\leq   3\int_{\mathbb{R}^d}    \varrho_{t}(\psi_2(x)) \big\| Z(x) \big\|_2 \de x  + 8^m E(m,A,B,\eta),
        \end{align*}
    where $E(m,A,B,\eta)$ is an exponentially small term defined as
        \begin{align*}
            E(m,A,B,\eta) \coloneq \int_{\|z\|_2\geq 2\eta} \frac{1}{{(\sqrt{2\pi} )}^{d}} \cdot e^{-\frac{\|z\|^2}{2}} [(A+2B{(\sqrt{d}D)}^m) + B {(\|z\|_2)} ^m] \de z.
        \end{align*}
\end{lemma}

\begin{proof}[Proof of \Cref{lemma: score error cut-off general scheme}]
By our assumption $\psi_1(x)$ is a diffeomorphism, and 
$\|\psi_1^{-1}(x)-x\|_2\leq  H(\bm\al_1+\bm\beta_1\|\psi_1^{-1}(x)\|_2)$. By rearranging it, we get $\|\psi_1^{-1}(x)-x\|_2\leq 2 H(\bm\al_1+\bm\beta_1\|x\|_2)$, provided $\bm\beta_1 H<1/4$.

   Next we notice that for any $x\in K_{\ast}$, $\|x\|_2 \leq \sqrt{d} D$ by the Assumption~\ref{assumption:secon-moment}.
    Consider the domain $\cD_{4R} \coloneq \{x\in \bR^d: \|x- \lambda_{T-t} K_{\ast}\|_2\leq 4R\}$, with $R \coloneq \eta \sigma_{T-t}$. 
    Outside the $\cD_{4R}$, we have that
    \begin{align}
        \begin{split}
            & \int_{\mathbb{R}^d \setminus \cD_{4R}}    \varrho_{t}(\psi_1(x)) \big\| Z(x)\big\|_2  \de x \leq \int_{\mathbb{R}^d \setminus \cD_{4R}}    \varrho_{t}(\psi_1(x)) (A + B\|x\|_2 ^m)  \de x
            \\ &\leq \int_{\mathbb{R}^d \setminus \psi_1 (\cD_{4R})}    \varrho_{t}(x) (A+ B \|\psi^{-1}_1(x)\|_2 ^m) | | \det[ \nabla (\psi_1 ^{-1}(x))]| \de x
            \\  & \leq e^{4\bm\xi_1 H d}\int_{\mathbb{R}^d \setminus (\cD_{2R})}    \varrho_{t}(x) (A + B {(\|x\|_2+2H (\bm\al_1 +  \bm\beta_1 \|x\|_2))}^m)  \de x
            \\  & \leq  2 \int_{\mathbb{R}^d \setminus (\cD_{2R})}    \varrho_{t}(x) 2^m ((A+B{2(H\bm\alpha_1)}^m) + B{(1+2H\bm\beta_1)}^m \|x\|_2 ^m)  \de x.
        \end{split}
    \end{align}
where we used the fact that $\nabla (\psi_1 ^{-1}(x)) = {[(\nabla \psi_1)(\psi_1 ^{-1}(x))]}^{-1}$ and \Cref{lemma: Growth of inverse matrix}, and we chose $H$ small such that, $\cD_{2R} \subset \psi_1 (\cD_{4R})$. This can be done if we require that $ H(\bm\alpha_1 + \bm\beta_1\sqrt{d} D) \leq \min\{R,1\}/10$. This is because if $x \in \cD_{2R}$, then there is a $y \in K_{\ast} $, such that $\|x- \lambda_{T-t} y\|_2\leq 2R$. Then, 
    \begin{align}\label{e: psi_1 in D4R}
        \begin{split}
            &\|\psi_1 ^{-1}(x) - \lambda_{T-t} y \|_2 \leq 2R + \|\psi_1 ^{-1}(x) - x \|_{2} \leq 2R +  H(\bm\al_1 + \bm\beta_1\|\psi_1 ^{-1}(x) \|_2)
            \\  &\leq 2R +  H(\bm\al_1 + \bm\beta_1 \|\psi_1 ^{-1}(x) -\lambda_{T-t} y\|_2 + \bm\beta_1 \sqrt{d}D).
        \end{split}
    \end{align}
Because $\bm\beta_1 H  \leq \bm\beta_1 H  \sqrt{d} D \leq 1/10$ as $D \geq 1$, we obtain that $\|\psi_1 ^{-1}(x) - \lambda_{T-t} y \|_2 \leq 3R + 2 H(\bm\al_1+ \bm\beta_1\sqrt{d} D) < 4R$. Hence, $\psi_1 ^{-1} (\cD_{2R}) \subset \cD_{4R}$. Now, according to the definition of $q_t$ as in \eqref{e: definition of q_t}, 
\begin{align*}
            \varrho_t(y) = q_{T-t} (y) =\int_{\mathbb{R}^d} \frac{1}{{(\sqrt{2\pi} \sigma_{T-t})}^{d}} \cdot e^{-\frac{\|y-\lambda_{T-t} x\|_2 ^2}{2{(\sigma_{T-t})}^2}} \ d\muast(x),
\end{align*}
    with $\int_{\mathbb{R}^d} d\muast(x) = 1$ and $\mathrm{supp}(\muast) = K_{\ast}$, we see that
        \begin{align*}
            \begin{split}
                &\int_{\mathbb{R}^d \setminus (\cD_{2R})}    \varrho_{t}(y) 2^m ((A+B{(2H\bm\al_1)}^m) + B{(1+2H\bm\beta_1)}^m \|y\|_2 ^m)  \de y \\
                &\leq 4^m \int_{\mathbb{R}^d \setminus (\cD_{2R})}\varrho_t(y) ((A+B) + B\|y\|_2 ^m) \de y 
                \\  &= 4^m\int_{\mathbb{R}^d \setminus (\cD_{2R})}\int_{\mathbb{R}^d} \frac{1}{{(\sqrt{2\pi} \sigma_{T-t})}^{d}} \cdot e^{-\frac{\|y-\lambda_{T-t} x\|_2 ^2}{2{(\sigma_{T-t})}^2}} ((A+B) + B\|y\|_2 ^m) \ d\muast(x)  \de y
                \\  &\leq 8^m \int_{\mathbb{R}^d} 
                \int_{\mathbb{R}^d \setminus (\cD_{2R})} \frac{1}{{(\sqrt{2\pi} \sigma_{T-t})}^{d}} \cdot e^{-\frac{\|y-\lambda_{T-t} x\|_2 ^2}{2{(\sigma_{T-t})}^2}} [(A+2B{(\sqrt{d}D)}^m) + B{\|y-\lambda_{T-t} x\|_2 }^m] \de y\ d\muast(x)  \\  &\leq 8^m \int_{\mathbb{R}^d} 
                \int_{\|y-\lambda_{T-t} x\|_2\geq 2R} \frac{1}{{(\sqrt{2\pi} \sigma_{T-t})}^{d}} \cdot e^{-\frac{\|y-\lambda_{T-t} x\|_2 ^2}{2{(\sigma_{T-t})}^2}} [(A+2B{(\sqrt{d}D)}^m) + B{\|y-\lambda_{T-t} x\|_2}^m] \de y\ d\muast(x)
                \\  & =8^m \int_{\|z\|_2\geq (2R)/\sigma_{T-t}} \frac{1}{{(\sqrt{2\pi} )}^{d}} \cdot e^{-\frac{\|z\|_2 ^2}{2}} [(A+2B{(\sqrt{d}D)}^m) + B {(\sigma_{T-t}\|z\|_2)}^m] \de z,
            \end{split}
        \end{align*}
which is an exponentially small term if we choose $(2R)/\sigma_{T-t}$ large.

Inside the $\cD_{4R}$, we need to estimate
    \begin{align*}
        \int_{\cD_{4R}}    \varrho_{t}(\psi_1(x)) \big\| Z(x)\big\|_2  \de x.
    \end{align*}
We estimate the ratio $\varrho_{t}(\psi_1(y)) / \varrho_{t}(\psi_2(y))$ for $y \in \cD_{4R}$. Again, by the definition of $q_t$ as in \eqref{e: definition of q_t},
\begin{align*}
            \begin{split}
                &\varrho_t(\psi_1(y)) =\int_{\mathbb{R}^d} \frac{1}{{(\sqrt{2\pi} \sigma_{T-t})}^{d}} \cdot e^{-\frac{\|\psi_1(y)-\lambda_{T-t} x\|_2 ^2}{2{(\sigma_{T-t})}^2}} \ d\muast(x)
                \\  & \leq \int_{\mathbb{R}^d} \frac{1}{{(\sqrt{2\pi} \sigma_{T-t})}^{d}} \cdot e^{\frac{2\|\psi_1(y)-\psi_2 (y)\|_2 \cdot \|\psi_2(y)-\lambda_{T-t} x\|_2 }{2{(\sigma_{T-t})}^2}}    \cdot e^{-\frac{\|\psi_2(y)-\lambda_{T-t} x\|_2 ^2}{2{(\sigma_{T-t})}^2}} \ d\muast(x)
                \\ &\leq e^{\frac{\|\psi_1(y)-\psi_2 (y)\|_2 \cdot (8R+2\sqrt{d}D)}{{(\sigma_{T-t})}^2}} \varrho_{t}(\psi_2(y)),
            \end{split}
\end{align*}
where the last inequality is because when $y \in \cD_{4R}$ and $x \in K_{\ast}$, similar arguments like \eqref{e: psi_1 in D4R} can deduce that $\psi_2(y) \in \cD_{8R}$, and then $\|\psi_2(y)-\lambda_{T-t} x\|_2 \leq \|\psi_2(y)-\lambda_{T-t} K_{\ast}\|_2+ 2\sqrt{d}D \leq 8R+2\sqrt{d}D$. For this error term, $\|\psi_1(y)-\psi_2 (y)\|_2$, we see that 
    \begin{align*}
        \begin{split}
            &\|\psi_1(y)-\psi_2 (y)\|_2 \leq \|\psi_1(y) - y\|_2 + \|\psi_2(y) - y \|_2 \leq H \big[(\bm\al_1+\bm\al_2)+(\bm\beta_1+\bm\beta_2)\|y\|_2 \big]
            \\  &\leq  H \big[(\bm\al_1+\bm\al_2)+(\bm\beta_1+\bm\beta_2) (\sqrt{d}D + 4R) \big].
        \end{split}
    \end{align*}
Hence, we need to choose $H$ small so that $100 H \big[(\bm\al_1+\bm\al_2)+(\bm\beta_1+\bm\beta_2) (\sqrt{d}D + R) \big] \cdot (\sqrt{d}D + R)\leq \sigma_{T-t} ^2$, and then $\frac{\|\psi_1(y)-\psi_2 (y)\|_2 \cdot (8R+2\sqrt{d}D)}{{(\sigma_{T-t})}^2} \leq 1$. So, the claim of \Cref{lemma: score error cut-off general scheme} follows.

\end{proof}

We need the following lemma for the  $E(m,A,B,\eta)$ term in \Cref{lemma: score error cut-off general scheme}.

\begin{lemma}\label{lemma: exponential error term in cut-off}
    Let $a$ be a positive constant and $b$ be a positive integer. Assume that $a^2 >  (b-1)$. Then, there is a universal constant $C_u$ such that
    \begin{align}\label{e:exponential tail of Gaussian}
        \int_{r \geq a} e^{-r ^2} r^b \de r \leq a^{b-1} e^{-a^2}.
    \end{align}
    As a corollary, in \Cref{lemma: score error cut-off general scheme}, we take $\eta = 10^{-2} (\bm\beta_1+\bm\beta_2)^{-1} H^{-1/2}$ and assume that 
    \begin{align}\label{e: H in cut-off scheme 1}
            H \leq (2d \log (2d))^{-1} 10^{-6} (\bm\beta_1+\bm\beta_2)^{-2}
        \end{align}
    and 
        \begin{align}\label{e: H in cut-off scheme 2}
            10^3 H \big[(\bm\al_1+\bm\al_2)\sqrt{d}D+(\bm\beta_1+\bm\beta_2)dD^2+d(\bm\xi_1+\bm\xi_2) \big] + H^{\frac{1}{2}} \frac{\bm\al_1+\bm\al_2}{\bm\beta_1+\bm\beta_2} \sigma_{T-t} \leq \sigma_{T-t} ^2.
        \end{align}
    Then, \eqref{e: H in cut-off scheme 1} and \eqref{e: H in cut-off scheme 2} imply \eqref{e: H in cut-off scheme} in \Cref{lemma: score error cut-off general scheme}. Also, under the assumption \eqref{e: H in cut-off scheme 1} only, we can deduce that for $m=1,2,3$, the error term $E(m,A,B,\eta)$ satisfies that 
    \begin{align*}
        E(m,A,B, \eta) \leq {4^m(A+B \sqrt{d}D)}^m \pi^{-\frac{d}{2}} e^{- 10^{-4} (\bm\beta_1+\bm\beta_2)^{-2} H^{-1}}.
    \end{align*}
\end{lemma}
\begin{proof}[Proof of \Cref{lemma: exponential error term in cut-off}]
    \begin{align*}
        \begin{split}
            \int_{r \geq a} e^{-r ^2} r^b \de r &= \bigg(-\frac{1}{2} r^{b-1} e^{-r^2} \bigg) \bigg|_{r=a} ^{+\infty} + \frac{b-1}{2} \int_{r \geq a} e^{-r ^2} r^{b-2} \de r
            \\  &\leq \frac{1}{2} a^{b-1} e^{-a^2} + \frac{1}{2} \int_{r \geq a} e^{-r ^2} r^b \de r.
        \end{split}
    \end{align*}
This proves \eqref{e:exponential tail of Gaussian}. When we take $\eta = 10^{-2} (\bm\beta_1+\bm\beta_2)^{-1} H^{-1/2}$ and assume \eqref{e: H in cut-off scheme 1} and \eqref{e: H in cut-off scheme 2} in \Cref{lemma: score error cut-off general scheme}, we first see that \eqref{e: H in cut-off scheme 1} and \eqref{e: H in cut-off scheme 2} imply \eqref{e: H in cut-off scheme} by directly plugging this $\eta$ into \eqref{e: H in cut-off scheme}. Then, by \eqref{e: H in cut-off scheme 1}, $\eta \geq 10 (2d \log (2d))^{1/2} \geq 10$. Using \eqref{e:exponential tail of Gaussian}, we see that for $m=1,2,3$,
    \begin{align*}
        E(m,A,B,\eta) \leq {4^m(A+B \sqrt{d}D)}^m \pi^{-\frac{d}{2}} e^{(d+2)\log(2\eta)-4\eta^2}.
    \end{align*}
Because $\eta \geq 10$ and $d \geq 1$, $(d+2)\log(2\eta) \leq 3d \log (\eta ^2) = 6d \log (\eta)$. We can also show that the function $G(\eta) = \eta^2 / \log (\eta)$ is increasing when $\eta \geq 10$, because $G'(\eta) = (2\eta \log (\eta) - \eta)/ (\log(\eta))^2 >0$. So, 
        \begin{align*}
            \eta^2 / \log (\eta) = G(\eta) \geq G(10 (2d \log (2d))^{1/2}) = \frac{200 d \log (2d) }{\frac{1}{2} \log( 200 d \log (2d) ) } = 4d \cdot \frac{\log((2d)^{100})}{\log(200 d \log (2d)) } \geq 4d,
        \end{align*}
    because $d \geq 1$. This shows that  $(d+2)\log(2\eta) \leq 2\eta^2$. Hence, 
    \begin{align*}
        E(m,A,B,\eta) \leq {4^m(A+B \sqrt{d}D)}^m \pi^{-\frac{d}{2}} e^{-2\eta^2}.
    \end{align*}
    We complete the proof of \Cref{lemma: exponential error term in cut-off} because $\eta = 10^{-2} (\bm\beta_1+\bm\beta_2)^{-1} H^{-1/2}$. 
\end{proof}

In Lemma~\ref{lemma: score error cut-off general scheme}, we particularly want to let $\psi_2(x)$ be some $\phi_h ^{-1}(x)$'s, for $0<h < H $, because $\del_h \phi_h(y)=V_{t+h}(\phi_h(y))$ with $\phi_{0}(y) = y$ for $y \in \mathbb{R}^d$, and $\varrho_{t+h} (x)= \varrho_{t} (\phi_h ^{-1}(x)) \cdot  \det[ \nabla (\phi_h ^{-1}(x))] $. So, we need to collect some growth estimates for $\phi_h (x)$ and $\phi_h ^{-1}(x)$ like \Cref{lemma: Growth of interpolation maps}.

\begin{lemma}\label{lemma: Growth of the original phi_h map}
   Adopt the assumptions and notations in \Cref{p:ABCD}. Take $H$ small enough such that $H  \cdot (\overline W ^{(1)} + \overline L ^{(1)} + d \overline K ^{(1)}) < 1/100$. Then, for $h < H$,
        \begin{align*}
            \begin{split}
                &\|x - \phi_h(x)\|_2 \leq 4 h {( \overline W ^{(0)} +  \overline W ^{(1)} \|x\|_2)}, \quad \|x - \phi_h ^{-1} (x)\|_2 \leq 8 h {( \overline W ^{(0)} +  \overline W ^{(1)} \|x\|_2)},
                \\  & \| \nabla \phi_{h} (x) -\mathbb{I}_d \|_{\op} \leq 4 \overline L ^{(1)} h, \quad \|(\nabla \phi_{h} ^{-1})(x) -\mathbb{I}_d \|_{\op} \leq 8 \overline L ^{(1)} h,
                \\  & \|\nabla \phi_{h}(x) -\mathbb{I}_d \|_{\infty} \leq 2 \overline K ^{(1)}h, \quad \|(\nabla \phi_{h} ^{-1})(x) -\mathbb{I}_d \|_{\infty} \leq 4 \overline K ^{(1)}h,
                \\  & 0<e^{-8  \overline K ^{(1)} hd} \leq  \det [\nabla \phi_{h}  (x)]  \leq e^{4  \overline K ^{(1)} hd} ,\quad 0< e^{-4  \overline K ^{(1)} hd} \leq  \det [\nabla \phi_{h} ^{-1} (x)] \leq e^{8  \overline K^{(1)} hd},
            \end{split}
        \end{align*}
    for any $x \in \mathbb{R}^d$.
\end{lemma}
\begin{proof}
    By definition, $\del_h \phi_h(y)=V_{t+h}(\phi_h(y))$ with $\phi_{0}(y) = y$ for $y \in \mathbb{R}^d$. For $h < H$,
    \begin{align*}
            \bigg | \frac{d}{dh} \|x - \phi_h(x)\|_2 ^2 \bigg | = \bigg | 2 V_{t+h}(\phi_h(x)) \cdot (\phi_h(x) - x) \bigg | \leq 2 ( \overline W ^{(0)} + \overline W ^{(1)} \|\phi_h(x)\|_2) \|x - \phi_h(x)\|_2.
    \end{align*}
    So, 
    \begin{align*}
            \bigg | \frac{d}{dh} \|x - \phi_h(x)\|_2 \bigg | \leq 2 (\overline W ^{(0)} +  \overline W ^{(1)} \|\phi_h(x)\|_2) \leq 2 \overline W ^{(1)} \|\phi_h(x)- x\|_2 + 2 ( \overline W ^{(0)} +  \overline W ^{(1)} \|x\|_2).
    \end{align*}
    Hence, for $h<H$,
        \begin{align*}
            \|x - \phi_h(x)\|_2 \leq \frac{e^{(2  \overline W ^{(1)} )h} - 1}{  \overline W ^{(1)} } {( \overline W ^{(0)} +  \overline W ^{(1)} \|x\|_2)} \leq e^{(2 \overline W^{(1)} )h} (2h  ){( \overline W ^{(0)} +  \overline W ^{(1)} \|x\|_2)}.
        \end{align*}
    Because $H  \cdot  \overline W ^{(1)} < 1/100$, we can obtain that $\|x - \phi_h(x)\|_2 \leq 4h {( \overline W ^{(0)} + \overline W ^{(1)} \|x\|_2)}$ for any $x \in \mathbb{R}^d$. Hence,  $\|x - \phi_h ^{-1}(x)\|_2 \leq 4 h {( \overline W ^{(0)} + \overline W ^{(1)} \|\phi_h ^{-1}(x)\|_2)} \leq 4 h {( \overline W ^{(0)} + \overline W ^{(1)}\|x\|_2 +  \overline W ^{(1)} \|\phi_h ^{-1}(x)-x\|_2)}$.  
    Because $H  \cdot \overline W ^{(1)} < 1/100$, we can obtain that $\|x - \phi_h ^{-1} (x)\|_2 \leq 8 h {( \overline W ^{(0)} + \overline W ^{(1)} \|x\|_2)}$ for any $x \in \mathbb{R}^d$.

    Next, we need to find a suitable bound for $\|\nabla \phi_h(x) -\mathbb{I}_d \|_{\op}$. Again, by definition, $\del_h \phi_h(y)=V_{t+h}(\phi_h(y))$ with $\phi_{0}(y) = y$ for $y \in \mathbb{R}^d$. For any fixed vector $v \in \mathbb{R}^d$ with $\|v\|_2 =1$, we see that
    \begin{align*}
        \begin{split}
            &\partial_h (\nabla \phi_h(y) \cdot v - v) = (\nabla V_{t+h})(\phi_h(y)) \cdot \nabla \phi_h(y) \cdot v \\ &= (\nabla V_{t+h})(\phi_h(y)) \cdot (\nabla \phi_h(y) \cdot v - v) + (\nabla V_{t+h})(\phi_h(y)) \cdot v.
        \end{split}
    \end{align*}
By considering the derivative of $\|\nabla \phi_h(y) \cdot v - v\|_2 ^2$, and the fact that $\|\nabla V_{t+h}(x) \|_{op} \leq \overline L ^{(1)} $ for any $x \in \mathbb{R}^d$, one can obtain that
    \begin{align*}
        \del_h \|\nabla \phi_h(y) \cdot v - v\|_2 \leq 2\overline L ^{(1)}  (\|\nabla \phi_h(y) \cdot v - v\|_2 + 1).
    \end{align*}
Hence, since $H \cdot \overline L ^{(1)} < 1/100$, we obtain that for any $h<H$,
    \begin{align*}
        \|\nabla \phi_h(y) \cdot v - v\|_2 \leq 4 \overline L ^{(1)}  \cdot h.
    \end{align*}
Hence, $\|\nabla \phi_{h}(y) -\mathbb{I}_d \|_{op} \leq 4 \overline L ^{(1)} h$ for any $y \in \mathbb{R}^d$. By Lemma~\ref{lemma: Growth of inverse matrix}, because $H  \cdot \overline L ^{(1)} < 1/100$, we see that $\|\nabla (\phi_{h} ^{-1}(y)) -\mathbb{I}_d \|_{op} = \|{(\nabla \phi_{h} (\phi_{h} ^{-1}(y)))}^{-1} -\mathbb{I}_d \|_{op} \leq 8 \overline L ^{(1)} h$ for any $y \in \mathbb{R}^d$.

For $\|\nabla \phi_{h}(x) -\mathbb{I}_d \|_{\infty}$, we consider the evolution of $\partial_i \phi_h ^{(j)} (x) - \delta_{ij}$ for arbitrary $i,j \in \llbracket 1 , d \rrbracket$. We see that
    \begin{align*}
        \begin{split}
            &\del_h |\partial_i \phi_h ^{(j)}(x) - \delta_{ij}| \leq |\del_h (\partial_i \phi_h ^{(j)}(x) - \delta_{ij})| = \bigg | \sum_{k=1} ^d \partial_i \phi_h ^{(k)} (x) \partial_k V_{t+h} ^{(j)}(\phi_h(x)) \bigg| 
            \\  &= \bigg| \partial_i V_{t+h} ^{(j)}(\phi_h(x))+ \sum_{k=1} ^d (\partial_i \phi_h ^{(k)} (x) - \delta_{ik}) \partial_k V_{t+h} ^{(j)}(\phi_h(x)) \bigg|
            \\  &\leq \overline K ^{(1)} \bigg(1+ \sum_{k=1} ^d |\partial_i \phi_h ^{(k)} (x) - \delta_{ik}|\bigg)\leq \overline K ^{(1)} (1+ d\|\nabla \phi_{h}(x) -\mathbb{I}_d \|_{\infty}).
        \end{split}
    \end{align*}
Hence,
    \begin{align*}
        \del_h  \|\nabla \phi_{h}(x) -\mathbb{I}_d \|_{\infty} \leq \overline K ^{(1)} d\|\nabla \phi_{h}(x) -\mathbb{I}_d \|_{\infty} + \overline K ^{(1)}.
    \end{align*}
Hence, for $h <H$, we have that 
    \begin{align*}
        \|\nabla \phi_{h}(x) -\mathbb{I}_d \|_{\infty} \leq \frac{e^{(\overline K ^{(1)} d)h} - 1}{  d}   \leq e^{(\overline K ^{(1)} d)h} (\overline K ^{(1)}h  ) \leq 2 \overline K ^{(1)}h ,
    \end{align*}
where the last inequality is because $H  \cdot d \overline K ^{(1)} < 1/100$.
\end{proof}

Now, we are able to estimate $J_1$ from \eqref{e:defJ1J2}. 

\begin{lemma}\label{lemma: interpolation error term 1}
Adopt the assumptions and notations in \Cref{lemma: Growth of interpolation maps}. If $H >0$ is small such that 
\begin{align*}
            400 H \big[\bm\al (\sqrt{d}D + \eta \sigma_{T-t})  + \bm\beta (\sqrt{d}D + \eta \sigma_{T-t})^2 + \bm\xi d\big]   \leq \sigma_{T-t} ^2,
    \end{align*}
then,
        \begin{align*}
            \begin{split}
            &J_1=\int_0 ^1 \int_{\mathbb{R}^d} \bigg| (\nabla \varrho_t) (\varphi_s(x)) \cdot (\varphi_1(x) - \varphi_0(x)) \det[ \nabla (\varphi_s(x))] \bigg| \de x \de s
        \\          
                & \leq\frac{10^3 }{({\sigma_{T-t})}^2} \int_{\mathbb{R}^d} \varrho_t (x) \cdot \big(   \sqrt{d} D+  \|x\|_2 \big) \cdot \| \baphi_H(x) - \dphi_H(x) \|_2   \de x \\   & \quad +\frac{10^4 }{({\sigma_{T-t})}^2} E(1,H\sqrt{d} D(\widetilde A + \overline A),H(\widetilde A + \overline A) + H\sqrt{d} D (\widetilde B + \overline B),\eta)
                \\  & \quad +\frac{10^5 }{({\sigma_{T-t})}^2} E(2,0, H(\widetilde B + \overline B), \eta) .
            \end{split}
        \end{align*}   
\end{lemma}
\begin{proof}
    According to Lemma~\ref{Lemma: Hessian estimates}, we see that $\|\nabla \log \varrho_t (y)\|_2 \leq \frac{2}{{\sigma_{T-t}}^2}(\|y\|_2 + \sqrt{d}D)$ for any $y \in \mathbb{R}^d$. $\big\| {(\dphi_H ^{-1})}(x) - {(\baphi_H ^{-1})}(x)\big\|_{2} = \big\| \baphi_H ^{-1}((\baphi_H \circ \dphi_H ^{-1})(x)) - {(\baphi_H ^{-1})}(x)\big\|_{2} \leq 2 \big\| {(\baphi_H \circ \dphi_H ^{-1})}(x) - x\big\|_{2}$ because $\| \nabla \baphi_H ^{-1} (y)\|_{op} = \| {(\nabla \baphi_H(\baphi_H ^{-1}(y)))}^{-1} \|_{op}  \leq 2$ for any $y \in \mathbb{R}^d$. Hence, for any $s \in [0,1]$,
        \begin{align}\label{e: interpolation error step 1}
            \begin{split}
                &\int_{\mathbb{R}^d} \bigg| (\nabla \varrho_t) (\varphi_s(x)) \cdot (\varphi_1(x) - \varphi_0(x)) \det[ \nabla (\varphi_s(x))] \bigg| \de x
                \\  &\leq \frac{4}{{(\sigma_{T-t})}^2} \int_{\mathbb{R}^d} \varrho_t (\varphi_s(x)) \cdot \big( \|\varphi_s(x)\|_2 + \sqrt{d} D\big) \cdot \big\| {(\baphi_H \circ \dphi_H ^{-1})}(x) - x\big\|_{2}  \big| \det[ \nabla (\varphi_s(x))] \big| \de x 
                \\  &\leq \frac{4}{{(\sigma_{T-t})}^2} \int_{\mathbb{R}^d} \varrho_t (\varphi_s(\dphi_H(x))) \cdot \big( \|\varphi_s(\dphi_H(x))\|_2 + \sqrt{d} D\big) \cdot \| \baphi_H(x) - \dphi_H(x) \|_2  \big| \det[ \nabla (\varphi_s(\dphi_H(x)))] \big| \de x .
            \end{split}
        \end{align}
    Notice that, by Lemma~\ref{lemma: Growth of interpolation maps}, $\|\varphi_s(\dphi_H(x)) - x\|_2 \leq \|\varphi_s(\dphi_H(x)) - \dphi_H(x)\|_2 + \|\dphi_H(x) - x\|_2 \leq  H(\bm\alpha+ \bm\beta\|\dphi_H(x)\|_2) + \|\dphi_H(x) - x\|_2 \leq  H(\bm\alpha + \bm\beta\|x\|_2) + (1+ \bm\beta H)\|\dphi_H(x) - x\|_2 \leq  H(\bm\alpha +\bm\beta \|x\|_2) + H(1+ \bm\beta H)(\widetilde A + \widetilde B \|x\|_2) = H[\bm\alpha + (1+ \bm\beta H)\widetilde A + (\bm\beta + (1+ \bm\beta H)  \widetilde B) \|x\|_2] \leq 4 H[\bm\alpha +   \bm\beta\|x\|_2]$, where we used the choice of $H$ in Lemma~\ref{lemma: Growth of interpolation maps} to get $(1+ \bm\beta H) \leq 2$. By Lemma~\ref{lemma: Growth of interpolation maps} again, $\|\nabla \varphi_s(\dphi_H(x)) \cdot \nabla \dphi_H(x) -\mathbb{I}_d \|_{\infty} \leq \bm\xi H \|\nabla \dphi_H(x)\|_{\infty} + \|\nabla \dphi_H(x) - \mathbb{I}_d \|_{\infty} \leq \bm\xi H + (1+ \bm\xi H) \|\nabla \dphi_H(x) - \mathbb{I}_d \|_{\infty} \leq (\bm\xi+ (1+ \bm\xi H)  \widetilde D)H \leq 3 \bm\xi H $,  where by the assumption of $H$, we have that $(1+ \bm\xi H) \leq 2$. According to the proof of Lemma~\ref{lemma: Growth of inverse matrix}, we see that $\big| \det[ \nabla (\varphi_s(\dphi_H(x)))] \big| \leq e ^ {6 \bm\xi Hd} \leq 5$ by the assumption on $H$.

    So, \eqref{e: interpolation error step 1} can be bounded by
        \begin{align*}
            \frac{80}{({\sigma_{T-t})}^2} \int_{\mathbb{R}^d} \varrho_t (\varphi_s(\dphi_H(x))) \cdot \big( (H\bm\alpha  + \sqrt{d} D)+ (1+\bm\beta H) \|x\|_2 \big) \cdot \| \baphi_H(x) - \dphi_H(x) \|_2   \de x .
        \end{align*}
        If we choose $H$ such that
            \begin{align*}
            200 H \big[\bm\alpha  + \bm\beta(\sqrt{d}D + \eta \sigma_{T-t}) \big] \leq \sigma_{T-t},
        \end{align*}
        where $\eta >1$ is a constant, then the above equation can be further bounded by
            \begin{align*}
            \frac{200}{({\sigma_{T-t})}^2} \int_{\mathbb{R}^d} \varrho_t (\varphi_s(\dphi_H(x))) \cdot \big(   \sqrt{d} D+  \|x\|_2 \big) \cdot \| \baphi_H(x) - \dphi_H(x) \|_2   \de x .
        \end{align*}
    We also see that $\|\baphi_H(x) - \dphi_H(x)\|_2 \leq \|\baphi_H(x) - x\|_2+ \|\dphi_H(x) - x\|_2 \leq H((\widetilde A + \overline A) + (\widetilde B + \overline B) \|x\|_2 )$. Hence, by Lemma~\ref{lemma: score error cut-off general scheme} and the growth of $\varphi_s(\dphi_H(x))$, i.e., $\|\varphi_s(\dphi_H(x)) - x\|_2 \leq 4 H[\bm\alpha +   \bm\beta\|x\|_2]$ we obtained above, the above equation can be further bounded by
        \begin{align*}
            \begin{split}
                &\frac{10^3 }{({\sigma_{T-t})}^2} \int_{\mathbb{R}^d} \varrho_t (x) \cdot \big(   \sqrt{d} D+  \|x\|_2 \big) \cdot \| \baphi_H(x) - \dphi_H(x) \|_2   \de x \\   &+\frac{10^4 }{({\sigma_{T-t})}^2} E(1,H\sqrt{d} D(\widetilde A + \overline A),H(\widetilde A + \overline A) + H\sqrt{d} D (\widetilde B + \overline B),\eta)
                \\  &+\frac{10^5 }{({\sigma_{T-t})}^2} E(2,0, H(\widetilde B + \overline B), \eta) .
            \end{split}
        \end{align*}
\end{proof}

\subsection{Estimates of $J_2$}
\label{ssec:J_2_proof}
For the other term $J_2$ in \eqref{e:defJ1J2}, we have the following lemmas.
\begin{lemma}\label{lemma: trace term in interpolation maps}
Adopt the assumptions and notations in \Cref{lemma: Growth of interpolation maps}.
    For any $x\in \mathbb{R}^d$ and any $s \in [0,1]$,
    \begin{align*}
        \begin{split}
            &\bigg| \trace \big[ \big( \nabla [\varphi_1(x) - \varphi_0(x)] \big) \cdot {(\nabla \varphi_s(x) )}^{-1} \big] \bigg|
        \\  & \leq 20d^{\frac{5}{2}}\|\nabla^2 \baphi_H \|_{\infty}\big\| {(\baphi_H \circ \dphi_H ^{-1})}(x) - x\big\|_{2} + 20\sum_{i=1} ^d \sum_{j=1} ^d \big| \partial_i \dphi_H ^{(j)} (\dphi_H ^{-1}(x)) - \partial_i \baphi_H ^{(j)} (\dphi_H ^{-1}(x)) \big|. 
        \end{split}
    \end{align*}
\end{lemma}
\begin{proof}
    Notice that $\nabla \varphi_1(x) = \nabla \dphi_H ^{-1}(x) = {(\nabla \dphi_H(\dphi_H ^{-1}(x)))}^{-1}$, which we can denote as $Q_1 ^{-1}$. Similarly, we denote $\nabla \varphi_0(x)$ as $Q_0 ^{-1}$ and denote ${(\nabla \varphi_s(x) )}^{-1}$ as $Q_s$.
        \begin{align*}
            \trace [(Q_1 ^{-1} - Q_0 ^{-1}  ) Q_s] = \trace [Q_1 ^{-1} (Q_0  - Q_1  ) Q_0 ^{-1} Q_s] = \trace [(Q_0  - Q_1  ) Q_0 ^{-1} Q_s Q_1 ^{-1} ].
        \end{align*}
By our assumption, $\|{(\nabla \dphi_H(y))}^{-1}\|_{op} \leq 2$ and $\|{(\nabla \baphi_H(y))}^{-1}\|_{op} \leq 2$ for any $y \in \mathbb{R}^d$. By the choice of $H$ in Lemma~\ref{lemma: Growth of interpolation maps}, $\bm\gamma H < 1/2$. It follows from Lemma~\ref{lemma: Growth of inverse matrix} that $\|{(\nabla \varphi_s(x) )}^{-1}\|_{op} \leq 2$. If we denote the matrix $Q_0 ^{-1} Q_s Q_1 ^{-1}$ as $P$, we see that $\|P\|_{op} \leq 8$, and then each of its element, say $P_{ij}$, satisfies $|P_{ij}| \leq 8$. Hence, $\big|\trace [(Q_0  - Q_1  )P ] \big| \leq 8 \sum_{i=1} \sum_{j=1} |{(Q_0  - Q_1  )}_{ij}|$. We also notice that, since $\| \nabla \baphi_H ^{-1} (y)\|_{op} = \| {(\nabla \baphi_H(\baphi_H ^{-1}(y)))}^{-1} \|_{op}  \leq 2$ for any $y \in \mathbb{R}^d$, 
    \begin{align*}
        \begin{split}
            &|{(Q_0  - Q_1  )}_{ij}| \leq \big| \partial_i \dphi_H ^{(j)} (\dphi_H ^{-1}(x)) - \partial_i \baphi_H ^{(j)} (\dphi_H ^{-1}(x)) \big| + \big| \partial_i \baphi_H ^{(j)} (\dphi_H ^{-1}(x)) - \partial_i \baphi_H ^{(j)} (\baphi_H ^{-1}(x)) \big|
            \\  & \leq \big| \partial_i \dphi_H ^{(j)} (\dphi_H ^{-1}(x)) - \partial_i \baphi_H ^{(j)} (\dphi_H ^{-1}(x)) \big| + \|\nabla^2 \baphi_H \|_{\infty}\sum_{k=1} ^d \big| {(\dphi_H ^{-1})}^{(k)}(x) - {(\baphi_H ^{-1})}^{(k)}(x)\big|
            \\  & \leq \big| \partial_i \dphi_H ^{(j)} (\dphi_H ^{-1}(x)) - \partial_i \baphi_H ^{(j)} (\dphi_H ^{-1}(x)) \big| + \sqrt{d}\|\nabla^2 \baphi_H \|_{\infty}\big\| {(\dphi_H ^{-1})}(x) - {(\baphi_H ^{-1})}(x)\big\|_{2}
            \\  & \leq \big| \partial_i \dphi_H ^{(j)} (\dphi_H ^{-1}(x)) - \partial_i \baphi_H ^{(j)} (\dphi_H ^{-1}(x)) \big| + 2\sqrt{d}\|\nabla^2 \baphi_H \|_{\infty}\big\| {(\baphi_H \circ \dphi_H ^{-1})}(x) - x\big\|_{2}.
        \end{split}
    \end{align*}

\end{proof}

\begin{lemma}\label{lemma: interpolation error term 2}
Adopt the assumptions and notations in \Cref{lemma: Growth of interpolation maps}. If $H >0$ is small such that
\begin{align*}
            4000 H \big[\bm\alpha (\sqrt{d}D + \eta \sigma_{T-t}) + \bm\beta(\sqrt{d}D + \eta \sigma_{T-t})^2 + \bm\xi d \big] \leq \sigma_{T-t}^2,
    \end{align*}
then,
    \begin{align*}
        \begin{split}
            &J_2= \int_0 ^1 \int_{\mathbb{R}^d} \varrho_t(\varphi_s(x)) \bigg| \trace \bigg( \bigg( \nabla [\varphi_1(x) - \varphi_0(x)] \bigg) \cdot {(\nabla \varphi_s(x) )}^{-1} \bigg) \det[ \nabla \varphi_s(x)] \bigg| \de x \de s
            \\  &\leq 600 d^{\frac{5}{2}}\|\nabla^2 \baphi_H \|_{\infty} \int_{\mathbb{R}^d} \varrho_t(x)   \big\|  \baphi_H (x) - \dphi_H(x) \big\|_{2}  \de x 
            + 600\sum_{i=1} ^d \sum_{j=1} ^d \int_{\mathbb{R}^d} \varrho_t(x) \big| \partial_i \dphi_H ^{(j)} (x) - \partial_i \baphi_H ^{(j)} (x) \big| \de x \\    & + 1600 d^{\frac{5}{2}}\|\nabla^2 \baphi_H \|_{\infty}E(1,H (\widetilde A + \overline A),H (\widetilde B + \overline B), \eta) + 1600 d^2 E(1,H (\widetilde D + \overline D),0, \eta).
        \end{split}
    \end{align*}
\end{lemma}

\begin{proof}

By change of variables, we see that $J_2$ equals to
    \begin{align*}
        \begin{split}
            &\int_0 ^1 \int_{\mathbb{R}^d} \varrho_t(\varphi_s(\dphi_H(x))) \bigg| \trace \bigg( \bigg( [(\nabla \varphi_1)( \dphi_H(x)) - (\nabla \varphi_0)(\dphi_H(x))] \bigg) \cdot {[(\nabla \varphi_s)(\dphi_H(x) )]}^{-1} \bigg) 
            \\  & \quad \cdot \det[ \nabla (\varphi_s  (\dphi_H (x)))] \bigg| \de x \de s.
        \end{split}
    \end{align*}
By the proof of Lemma~\ref{lemma: interpolation error term 1}, 
$\big| \det[ \nabla (\varphi_s(\dphi_H(x)))] \big| \leq e ^ {6\bm\xi Hd} \leq 10$ according to the assumption on $H$. Together with Lemma~\ref{lemma: trace term in interpolation maps}, the above equation is bounded by
    \begin{align*}
        \begin{split}
            &200 d^{\frac{5}{2}}\|\nabla^2 \baphi_H \|_{\infty} \int_0 ^1 \int_{\mathbb{R}^d} \varrho_t(\varphi_s(\dphi_H(x)))   \big\|  \baphi_H (x) - \dphi_H(x) \big\|_{2}  \de x \de s
            \\  &+ 200\sum_{i=1} ^d \sum_{j=1} ^d \int_0 ^1 \int_{\mathbb{R}^d} \varrho_t(\varphi_s(\dphi_H(x))) \big| \partial_i \dphi_H ^{(j)} (x) - \partial_i \baphi_H ^{(j)} (x) \big| \de x \de s. 
        \end{split}
    \end{align*}
We have seen that $\|\baphi_H(x) - \dphi_H(x)\|_2 \leq \|\baphi_H(x) - x\|_2+ \|\dphi_H(x) - x\|_2 \leq H((\widetilde A + \overline A) + (\widetilde B + \overline B) \|x\|_2 )$. Also, $\big| \partial_i \dphi_H ^{(j)} (x) - \partial_i \baphi_H ^{(j)} (x) \big| \leq H (\widetilde D + \overline D)$. In the proof of Lemma~\ref{lemma: interpolation error term 1}, we showed that $\|\varphi_s(\dphi_H(x)) - x\|_2 \leq 4 H[\bm\alpha+  \bm\beta\|x\|_2]$. Hence, by Lemma~\ref{lemma: score error cut-off general scheme}, the above term can be bounded by
    \begin{align*}
        \begin{split}
            &600 d^{\frac{5}{2}}\|\nabla^2 \baphi_H \|_{\infty} \int_{\mathbb{R}^d} \varrho_t(x)   \big\|  \baphi_H (x) - \dphi_H(x) \big\|_{2}  \de x 
            + 600\sum_{i=1} ^d \sum_{j=1} ^d \int_{\mathbb{R}^d} \varrho_t(x) \big| \partial_i \dphi_H ^{(j)} (x) - \partial_i \baphi_H ^{(j)} (x) \big| \de x \\    & + 1600 d^{\frac{5}{2}}\|\nabla^2 \baphi_H \|_{\infty}E(1,H (\widetilde A + \overline A),H (\widetilde B + \overline B), \eta) + 1600d^2 E(1,H (\widetilde D + \overline D),0, \eta),
        \end{split}
    \end{align*}
    provided that 
    \begin{align*}
            4000 H \big[\bm\alpha (\sqrt{d}D + \eta \sigma_{T-t}) + \bm\beta(\sqrt{d}D + \eta \sigma_{T-t})^2 \big] \leq \sigma_{T-t} ^2,
    \end{align*}
    according to Lemma~\ref{lemma: score error cut-off general scheme}.

\end{proof}

\subsection{Estimates of $I_2$}
\label{ssec:I2-estimate}
Combine \eqref{e: interpolation error in second term}, Lemma~\ref{lemma: interpolation error term 1}, and Lemma~\ref{lemma: interpolation error term 2}, we see that $I_2$ from  \eqref{e:RK2-I123-Appendix}
can be bounded by 
\begin{align}\begin{split}\label{e:I2bb}
&I_2 \leq J_1+J_2      
                 \leq\frac{10^3 }{({\sigma_{T-t})}^2} \int_{\mathbb{R}^d} \varrho_t (x) \cdot \big(   \sqrt{d} D+  \|x\|_2 \big) \cdot \| \baphi_H(x) - \dphi_H(x) \|_2   \de x \\  
                &+600 d^{\frac{5}{2}}\|\nabla^2 \baphi_H \|_{\infty} \int_{\mathbb{R}^d} \varrho_t(x)   \big\|  \baphi_H (x) - \dphi_H(x) \big\|_{2}  \de x+ 600\sum_{i=1} ^d \sum_{j=1} ^d \int_{\mathbb{R}^d} \varrho_t(x) \big| \partial_i \dphi_H ^{(j)} (x) - \partial_i \baphi_H ^{(j)} (x) \big| \de x \\
                &  +\frac{10^4 }{({\sigma_{T-t})}^2} E(1,H\sqrt{d} D(\widetilde A + \overline A),H(\widetilde A + \overline A) + H\sqrt{d} D (\widetilde B + \overline B),\eta)
                +\frac{10^5 }{({\sigma_{T-t})}^2} E(2,0, H(\widetilde B + \overline B), \eta) \\
                &+ 1600 d^{\frac{5}{2}}\|\nabla^2 \baphi_H \|_{\infty}E(1,H (\widetilde A + \overline A),H (\widetilde B + \overline B), \eta) + 1600 d^2 E(1,H (\widetilde D + \overline D),0, \eta),
\end{split}\end{align}
which involves terms in the following forms
    \begin{align}\begin{split}\label{e:three_term}
        &\int_{\mathbb{R}^d} \varrho_t(x)   \big\|  \baphi_H (x) - \dphi_H(x) \big\|_{2}  \de x, \\
    & \int_{\mathbb{R}^d} \varrho_t(x)  \|x\|_2 \big\|  \baphi_H (x) - \dphi_H(x) \big\|_{2}  \de x ,  \\ &\int_{\mathbb{R}^d} \varrho_t(x) \big| \partial_i \dphi_H ^{(j)} (x) - \partial_i \baphi_H ^{(j)} (x) \big| \de x,
   \end{split} \end{align}
and exponentially small terms $E$'s. For these exponentially small terms $E$'s, we use the fact that $\sigma_{T-t}^2 \geq C_u (T-t) \geq C_u \tau$, the bound of $\|\nabla^2 \baphi_H \|_{\infty} \leq 48 H B_p s^3 4^s D^3/\tau^3$ from \Cref{lemma: Constant Convention} and \Cref{lemma: Constant Convention1},  and use \Cref{lemma: exponential error term in cut-off}, we can bound the exponentially small terms in \eqref{e:I2bb} as
\begin{align}\label{e:I2bb2}
    C(s,A_p, B_p ,\widetilde K, \widetilde W) \cdot \frac{d^2 H D^4}{\tau^4}\cdot \pi^{-\frac{d}{2}} e^{-\frac{1}{10^4 (4\bm\beta)^{2} H} },
\end{align}
provided the assumption \eqref{e:Hcondition} for $H$ holds. Here,  $\bm\beta$ is defined in \Cref{lemma: Growth of interpolation maps}. This is because we choose $\eta = 10^{-2} (4\bm\beta)^{-1} H^{-1/2}$ in \Cref{lemma: exponential error term in cut-off}, and \eqref{e:Hcondition} implies that $H \leq (2d \log (2d))^{-1} 10^{-6} (4\bm\beta)^{-2}$.

In the rest of the subsection, we estimate these three terms in \eqref{e:three_term}. Since $\dphi_h(y) = y + h \dM_t(y)$ and $\baphi_h(y) = y + h \baM_t(y)$, the term $\|  \baphi_H (x) - \dphi_H(x) \|_{2}$ reduces to estimates on $\|\dM_t (x) - \baM_t(x)\|_2$. For the term $\int_{\mathbb{R}^d} \varrho_t(x) \big| \partial_i \dphi_H ^{(j)} (x) - \partial_i \baphi_H ^{(j)} (x) \big| \de x$, notice that
    \begin{align*}
        \begin{split}
            \int_{\mathbb{R}^d} \varrho_t(x) \big| \partial_i \dphi_H ^{(j)} (x) - \partial_i \baphi_H ^{(j)} (x) \big| \de x &\leq \int_{\mathbb{R}^d} \big|\partial_i\varrho_t(x) \big| \cdot \big|  \dphi_H ^{(j)} (x) -  \baphi_H ^{(j)} (x) \big| \de x 
            \\  & \quad + \int_{\mathbb{R}^d}  \big| \partial_i \big(\varrho_t(x) (\dphi_H ^{(j)} (x) -  \baphi_H ^{(j)} (x))\big) \big| \de x,
        \end{split}
    \end{align*}
where we can estimate $|\partial_i\varrho_t(x)|$ by $|\varrho_t(x)| \sigma_{T-t} ^{-2} (|x_i| + \lambda_{T-t} D)$ according to Lemma~\ref{Lemma: Hessian estimates}, and we can estimate the term $\int_{\mathbb{R}^d}  \big| \partial_i \big(\varrho_t(x) (\dphi_H ^{(j)} (x) -  \baphi_H ^{(j)} (x))\big) \big| \de x$ by $\int_{\mathbb{R}^d}  \big|  \varrho_t(x) (\dphi_H ^{(j)} (x) -  \baphi_H ^{(j)} (x)) \big| \de x$ multiplying $\int_{\mathbb{R}^d}  \big| \partial_{ii} ^2 \big(\varrho_t(x) (\dphi_H ^{(j)} (x) -  \baphi_H ^{(j)} (x))\big) \big| \de x$, according to the Gagliardo-Nirenberg inequality Lemma~\ref{lemma: Gagliardo-Nirenberg}.

\begin{lemma}\label{lemma: Second derivative bound on errors}
   Adopt the assumptions and notations in \Cref{p:ABCD}. Take $H$ small enough such that
    \begin{align*}
        H \cdot d A_p(\overline K ^{(1)} + \widetilde K ^{(1)}) \leq 1/10.  
    \end{align*}
Then, there is a universal constant $C_u >0$, such that
        \begin{align*}
            \begin{split}
                &\quad \sum_{i=1} ^d \sum_{j=1} ^d \int_{\mathbb{R}^d} \varrho_t(x) \big| \partial_i \dphi_H ^{(j)} (x) - \partial_i \baphi_H ^{(j)} (x) \big| \de x \\   &\leq \frac{C_u  d}{\sigma_{T-t} ^2} {\bigg (  \int_{\mathbb{R}^d}  \varrho_t(x) \| (\dphi_H (x) -  \baphi_H  (x)) \|_{2} \de x \bigg)}^{\frac{1}{2}}{\bigg ( \int_{\mathbb{R}^d}   \varrho_t(x) (\|x\|_2 ^2 + dD^2) \|\dphi_H  (x) -  \baphi_H (x) \|_2 \de x \bigg)}^{\frac{1}{2}} 
            \\ &+   \frac{C_u {H}^{\frac{1}{2}} d^{\frac{7}{4}} {B_p}^{\frac{1}{2}}}{\sigma_{T-t}} {\bigg (  \int_{\mathbb{R}^d}  \varrho_t(x) \| (\dphi_H (x) -  \baphi_H  (x)) \|_{2} \de x \bigg)}^{\frac{1}{2}} {\big[  D   (\widetilde K ^{(1)} +\overline K ^{(1)} )   + \sigma_{T-t} ^2(\widetilde K ^{(2)} +\overline K ^{(2)} )\big]}^{\frac{1}{2}} 
            \\  &+ \frac{d}{\sigma_{T-t} ^2}  \int_{\mathbb{R}^d} \varrho_t(x)  (\|x\|_2 + \sqrt{d} D) \cdot \|  \dphi_H  (x) -  \baphi_H (x) \|_2 \de x .
            \end{split}
        \end{align*}
\end{lemma}
\begin{proof}
    Notice that
    \begin{align}\label{e:ht1}
        \begin{split}
            &\sum_{i=1} ^d \sum_{j=1} ^d \int_{\mathbb{R}^d} \varrho_t(x) \big| \partial_i \dphi_H ^{(j)} (x) - \partial_i \baphi_H ^{(j)} (x) \big| \de x \\ &\leq \sum_{i=1} ^d \sum_{j=1} ^d \int_{\mathbb{R}^d} \big|\partial_i\varrho_t(x) \big| \cdot \big|  \dphi_H ^{(j)} (x) -  \baphi_H ^{(j)} (x) \big| \de x 
             + \sum_{i=1} ^d \sum_{j=1} ^d\int_{\mathbb{R}^d}  \big| \partial_i \big(\varrho_t(x) (\dphi_H ^{(j)} (x) -  \baphi_H ^{(j)} (x))\big) \big| \de x.
        \end{split}
    \end{align}

We estimate the second part first. According to the Gagliardo-Nirenberg inequality Lemma~\ref{lemma: Gagliardo-Nirenberg}, we see that
    \begin{align*}
        \begin{split}
            &\sum_{i=1} ^d \sum_{j=1} ^d\int_{\mathbb{R}^d}  \big| \partial_i \big(\varrho_t(x) (\dphi_H ^{(j)} (x) -  \baphi_H ^{(j)} (x))\big) \big| \de x 
            \\  &\leq C_u  \sum_{i=1} ^d \sum_{j=1} ^d {\bigg (  \int_{\mathbb{R}^d}  \big|  \varrho_t(x) (\dphi_H ^{(j)} (x) -  \baphi_H ^{(j)} (x)) \big| \de x \bigg)}^{\frac{1}{2}}{\bigg (\int_{\mathbb{R}^d}  \big| \partial_{ii} ^2 \big(\varrho_t(x) (\dphi_H ^{(j)} (x) -  \baphi_H ^{(j)} (x))\big) \big| \de x \bigg)}^{\frac{1}{2}}
            \\  &\leq C_u  \sum_{i=1} ^d {\bigg (  \sum_{j=1} ^d\int_{\mathbb{R}^d}  \big|  \varrho_t(x) (\dphi_H ^{(j)} (x) -  \baphi_H ^{(j)} (x)) \big| \de x \bigg)}^{\frac{1}{2}}{\bigg ( \sum_{j=1} ^d\int_{\mathbb{R}^d}  \big| \partial_{ii} ^2 \big(\varrho_t(x) (\dphi_H ^{(j)} (x) -  \baphi_H ^{(j)} (x))\big) \big| \de x \bigg)}^{\frac{1}{2}}
            \\  &\leq C_u  \sum_{i=1} ^d {\bigg (  \sqrt{d}\int_{\mathbb{R}^d}  \varrho_t(x) \| (\dphi_H (x) -  \baphi_H  (x)) \|_{2} \de x \bigg)}^{\frac{1}{2}}{\bigg ( \sum_{j=1} ^d\int_{\mathbb{R}^d}  \big| \partial_{ii} ^2 \big(\varrho_t(x) (\dphi_H ^{(j)} (x) -  \baphi_H ^{(j)} (x))\big) \big| \de x \bigg)}^{\frac{1}{2}}
            \\  &\leq C_u  d^{\frac{3}{4}} {\bigg (  \int_{\mathbb{R}^d}  \varrho_t(x) \| (\dphi_H (x) -  \baphi_H  (x)) \|_{2} \de x \bigg)}^{\frac{1}{2}}{\bigg ( \sum_{i=1} ^d \sum_{j=1} ^d\int_{\mathbb{R}^d}  \big| \partial_{ii} ^2 \big(\varrho_t(x) (\dphi_H ^{(j)} (x) -  \baphi_H ^{(j)} (x))\big) \big| \de x \bigg)}^{\frac{1}{2}},
        \end{split}
    \end{align*}
where we also used H{\" o}lder inequality. 
We notice that 
    \begin{align*}
        \begin{split}
            & \sum_{i=1} ^d \sum_{j=1} ^d\int_{\mathbb{R}^d}  \big| \partial_{ii} ^2 \big(\varrho_t(x) (\dphi_H ^{(j)} (x) -  \baphi_H ^{(j)} (x))\big) \big| \de x \leq \sum_{i=1} ^d \int_{\mathbb{R}^d}  |\partial_{ii} ^2 \varrho_t(x) |\sqrt{d}\|\dphi_H  (x) -  \baphi_H (x) \|_2 \de x
            \\  &+ 2d \sum_{i=1} ^d \int_{\mathbb{R}^d}  |\partial_{i} \varrho_t(x) | \|\nabla \dphi_H  (x) -  \nabla \baphi_H (x) \|_{\infty} \de x + d^2 \int_{\mathbb{R}^d}  | \varrho_t(x) | \|\nabla^2 \dphi_H  (x) -  \nabla ^2\baphi_H (x) \|_{\infty} \de x.
        \end{split}
    \end{align*}
We then need to use Lemma~\ref{lemma: infinity norm pth order growth in RK}. We see that by the assumption on $H$,
    \begin{align*}
       & \|\nabla \dphi_H  (x) -  \nabla \baphi_H (x) \|_{\infty} \leq HB_p s 2^s (\widetilde K ^{(1)} +\overline K ^{(1)} ),
   \\
        &\|\nabla ^2 \dphi_H  (x) -  \nabla ^2 \baphi_H (x) \|_{\infty} \leq 2HB_p s^3 4^s (\widetilde K ^{(2)} +\overline K ^{(2)} ).
    \end{align*}
According to Lemma~\ref{Lemma: Hessian estimates}, we know that 
    \begin{align}\label{e:ratio density bounds}
        \sum_{i=1} ^d \bigg| \frac{\partial _{i} \varrho_t(x)}{\varrho_t(x)} \bigg| \leq \frac{(\sqrt{d}\|x\|_2 + d D) }{\sigma_{T-t} ^2}, \quad  \sum_{i=1} ^d \bigg| \frac{\partial^2 _{ii} \varrho_t(x)}{\varrho_t(x)} \bigg| \leq \frac{3(\|x\|_2 ^2 + dD^2) }{\sigma_{T-t} ^4}.
    \end{align}
   By~\Cref{lemma: moments of L^2 and infinity norms}, we know that
    \begin{align*}
        \int_{\mathbb{R}^d} \|x\|_2 ^m \varrho_t(x) \ \de x \leq C(m) {(\sqrt{d}D)}^m.
    \end{align*}
Also,
    \begin{align*}
        \begin{split}
            &\sum_{i=1} ^d \int_{\mathbb{R}^d}  |\partial_{ii} ^2 \varrho_t(x) |\sqrt{d}\|\dphi_H  (x) -  \baphi_H (x) \|_2 \de x \leq \frac{3\sqrt{d}}{\sigma_{T-t} ^4}\int_{\mathbb{R}^d}   \varrho_t(x) (\|x\|_2 ^2 + dD^2) \|\dphi_H  (x) -  \baphi_H (x) \|_2 \de x.  
        \end{split}
    \end{align*}
Hence, there is a universal constant $C_u >0$ such that
    \begin{align*}
        \begin{split}
            &\sum_{i=1} ^d \sum_{j=1} ^d\int_{\mathbb{R}^d}  \big| \partial_{ii} ^2 \big(\varrho_t(x) (\dphi_H ^{(j)} (x) -  \baphi_H ^{(j)} (x))\big) \big| \de x \leq \frac{3\sqrt{d}}{\sigma_{T-t} ^4}\int_{\mathbb{R}^d}   \varrho_t(x) (\|x\|_2 ^2 + dD^2) \|\dphi_H  (x) -  \baphi_H (x) \|_2 \de x
            \\  &+C_u \frac{HB_p d^2}{\sigma_{T-t} ^2}\big[  D   (\widetilde K ^{(1)} +\overline K ^{(1)} )   + \sigma_{T-t} ^2(\widetilde K ^{(2)} +\overline K ^{(2)} )\big].
        \end{split}
    \end{align*}
So,
    \begin{align}
        \begin{split}\label{e:ht2}
            & \quad \sum_{i=1} ^d \sum_{j=1} ^d\int_{\mathbb{R}^d}  \big| \partial_i \big(\varrho_t(x) (\dphi_H ^{(j)} (x) -  \baphi_H ^{(j)} (x))\big) \big| \de x  
            \\  & \leq \frac{C_u  d}{\sigma_{T-t} ^2} {\bigg (  \int_{\mathbb{R}^d}  \varrho_t(x) \| (\dphi_H (x) -  \baphi_H  (x)) \|_{2} \de x \bigg)}^{\frac{1}{2}}{\bigg ( \int_{\mathbb{R}^d}   \varrho_t(x) (\|x\|_2 ^2 + dD^2) \|\dphi_H  (x) -  \baphi_H (x) \|_2 \de x \bigg)}^{\frac{1}{2}} 
            \\ &+   \frac{C_u {H}^{\frac{1}{2}} d^{\frac{7}{4}} {B_p}^{\frac{1}{2}}}{\sigma_{T-t}} {\bigg (  \int_{\mathbb{R}^d}  \varrho_t(x) \| (\dphi_H (x) -  \baphi_H  (x)) \|_{2} \de x \bigg)}^{\frac{1}{2}} {\big[  D   (\widetilde K ^{(1)} +\overline K ^{(1)} )   + \sigma_{T-t} ^2(\widetilde K ^{(2)} +\overline K ^{(2)} )\big]}^{\frac{1}{2}} .
        \end{split}
    \end{align}
    Finally, by \eqref{e:ratio density bounds} again, we notice that
        \begin{align}
            \begin{split}\label{e:ht3}
                &\sum_{i=1} ^d \sum_{j=1} ^d \int_{\mathbb{R}^d} \big|\partial_i\varrho_t(x) \big| \cdot \big|  \dphi_H ^{(j)} (x) -  \baphi_H ^{(j)} (x) \big| \de x \leq \frac{d}{\sigma_{T-t} ^2}  \int_{\mathbb{R}^d} \varrho_t(x)  (\|x\|_2 + \sqrt{d} D) \cdot \|  \dphi_H  (x) -  \baphi_H (x) \|_2 \de x .
            \end{split}
        \end{align}
The claim of \Cref{lemma: Second derivative bound on errors} follows from combining \eqref{e:ht1}, \eqref{e:ht2} and \eqref{e:ht3}.
\end{proof}

Next, we need to use the explicit formula \eqref{e:deftMt} and \eqref{e: score error p RK method} to explicitly estimate $\|\dM_t (x) - \baM_t(x) \|_{2}$ and $\|\dM_t (x) - \baM_t(x) \|_{\infty}$. We have the following lemma.

\begin{lemma}\label{lemma: pth order score error}
  Adopt the assumptions and notations in \Cref{p:ABCD}. Take $H$ small enough such that $t+H \leq T-\tau$, and
    \begin{align*}
        H \cdot  A_p s 2 ^s (\widetilde L ^{(1)} +\overline L ^{(1)} + d \overline K ^{(1)} + \overline W ^{(1)}) \leq 1/10.
    \end{align*}
    For $\dM_t, \baM_t$ as in \eqref{e:deftMt} and \eqref{e: score error p RK method}, we have that 
        \begin{align*}
            \begin{split}
                \| \dM_t (x) - \baM_t(x) \|_2 &\leq 2 B_p \bigg[ \| \sV_{t + c_s H}(\psi_{s,H}(x)) - V_{t + c_s H}(\psi_{s,H}(x)) \|_2 \\ & \quad +  \| \sV_{t + c_{s-1} H}(\psi_{s-1,H}(x)) - V_{t + c_{s-1} H}(\psi_{s-1,H}(x)) \|_2
                \\  &\quad +  \| \sV_{t + c_{s-2} H}(\psi_{s-2,H}(x)) - V_{t + c_{s-2} H}(\psi_{s-2,H}(x)) \|_2
                \\  &\quad + \cdots +   \| \sV_{t + c_{1} H}(\psi_{1,H}(x)) - V_{t + c_{1} H}(\psi_{1,H}(x)) \|_2\bigg],
            \end{split}
        \end{align*}
    where each $\psi_{j,H}$ for $j\in\llbracket 1,s\rrbracket$ is a diffeomorphism on $\mathbb{R}^d$, and satisfies that
            \begin{align*}
            \begin{split}
                & \|\psi_{j,H} ^{-1} (x) -x\|_2 \leq 2A_p H s{2}^{s} \big[   \overline W ^{(0)}+\overline W ^{(1)} \|x\|_{2}  \big] ,
                \\  & \big \|(\nabla \psi_{j,H} ^{-1})(x) -\mathbb{I}_d\big \|_{op} \leq 2 A_p H s 2^s  \overline L ^{(1)} \leq 1/10,
                \\  &\big \|(\nabla \psi_{j,H} ^{-1})(x) -\mathbb{I}_d\big \|_{\infty}  \leq 2 A_p H s 2^s  \overline K ^{(1)} ,
                \\  &e^{-8 A_p  s 2^s  \overline K ^{(1)} H d} \leq | \det[(\nabla \psi_{j,H} ^{-1})(x)]| \leq e^{4 A_p  s 2^s  \overline K ^{(1)} H d}.
            \end{split}
        \end{align*}
Also, if we further assume that $H>0$ is small such that 
    \begin{align*}
        H \cdot  A_p d s 2 ^s (\overline K ^{(1)} +\widetilde K ^{(1)} ) \leq 1/10,
    \end{align*}
then
    \begin{align*}
            \begin{split}
                \| \dM_t (x) - \baM_t(x) \|_{\infty} &\leq 2 B_p \bigg[ \| \sV_{t + c_s H}(\psi_{s,H}(x)) - V_{t + c_s H}(\psi_{s,H}(x)) \|_{\infty} \\ & \quad +  \| \sV_{t + c_{s-1} H}(\psi_{s-1,H}(x)) - V_{t + c_{s-1} H}(\psi_{s-1,H}(x)) \|_{\infty}
                \\  &\quad +  \| \sV_{t + c_{s-2} H}(\psi_{s-2,H}(x)) - V_{t + c_{s-2} H}(\psi_{s-2,H}(x)) \|_{\infty}
                \\  &\quad + \cdots +   \| \sV_{t + c_{1} H}(\psi_{1,H}(x)) - V_{t + c_{1} H}(\psi_{1,H}(x)) \|_{\infty}\bigg].
            \end{split}
        \end{align*}
\end{lemma}

\begin{proof}
    First, $\| \dM_t (x) - \baM_t(x) \|_2\leq B_p \sum_{i=1} ^s \| \widetilde k_i - \overline k_i\|_2$. For each $\| \widetilde k_i - \overline k_i\|_2$, denoted as $\epsilon_i$, according to \eqref{e: p RK interpolation terms}, we see that 
        \begin{align*}
            \begin{split}
                \epsilon_i &\leq \| \sV_{t + c_i H}\bigl(x + (a_{i1} \widetilde k_1 + a_{i2} \widetilde k_2 + \cdots + a_{i,i-1} \widetilde k_{i-1}) H\bigr) \\ & \quad - \sV_{t + c_i H}\bigl(x + (a_{i1} \overline k_1 + a_{i2} \overline k_2 + \cdots + a_{i,i-1} \overline k_{i-1}) H\bigr)\|_{2}
                \\  &+ \| \sV_{t + c_i H}\bigl(x + (a_{i1} \overline k_1 + a_{i2} \overline k_2 + \cdots + a_{i,i-1} \overline k_{i-1}) H\bigr) \\ & \quad - V_{t + c_i H}\bigl(x + (a_{i1} \overline k_1 + a_{i2} \overline k_2 + \cdots + a_{i,i-1} \overline k_{i-1}) H\bigr)\|_{2}
                \\  &\leq \widetilde L ^{(1)} A_p H (\epsilon_1 + \epsilon_2 + \cdots + \epsilon_{i-1}) + \| \sV_{t + c_i H}(\psi_{i,H}(x)) - V_{t + c_i H}(\psi_{i,H}(x)) \|_2,
            \end{split}
        \end{align*}
    where we used the notation $\psi_{i,H} (x) \coloneq x + (a_{i1} \overline k_1 + a_{i2} \overline k_2 + \cdots + a_{i,i-1} \overline k_{i-1}) H$. If we define $\mathbf{T}_j \coloneq \sum_{i=1} ^j \epsilon_i$, we see that $\mathbf{T}_j \leq (1+\widetilde L ^{(1)} A_p H)\mathbf{T}_{j-1} +  \| \sV_{t + c_j H}(\psi_{j,H}(x)) - V_{t + c_j H}(\psi_{j,H}(x)) \|_2$. Hence,
        \begin{align*}
            \begin{split}
                \mathbf{T}_s &\leq \| \sV_{t + c_s H}(\psi_{s,H}(x)) - V_{t + c_i H}(\psi_{i,H}(x)) \|_2 \\ & \quad +(1+\widetilde L ^{(1)} A_p H) \| \sV_{t + c_{s-1} H}(\psi_{s-1,H}(x)) - V_{t + c_{s-1} H}(\psi_{s-1,H}(x)) \|_2
                \\  &\quad + {(1+\widetilde L ^{(1)} A_p H)}^2 \| \sV_{t + c_{s-2} H}(\psi_{s-2,H}(x)) - V_{t + c_{s-2} H}(\psi_{s-2,H}(x)) \|_2
                \\  &\quad + \cdots + {(1+\widetilde L ^{(1)} A_p H)}^{s-1}  \| \sV_{t + c_{1} H}(\psi_{1,H}(x)) - V_{t + c_{1} H}(\psi_{1,H}(x)) \|_2.
            \end{split}
        \end{align*}
    According to \Cref{lemma: pth order growth in RK}, \Cref{lemma: infinity norm pth order growth in RK}, and the assumption on $H$, for every $j \in \llbracket 1, s \rrbracket$,
        \begin{align*}
            \begin{split}
                & \|\psi_{j,H}(x) -x\|_2 \leq  A_p H \sum_{i=1} ^{j-1} \|\overline k_i(x) \|_{2} \leq A_p H s{2}^{s} \big[   \overline W ^{(0)}+\overline W ^{(1)} \|x\|_{2}  \big],
                \\  & \big \|\nabla \psi_{j,H}(x) -\mathbb{I}_d\big \|_{op} \leq A_p H \sum_{i=1} ^{j-1} \|\nabla \overline k_i(x) \|_{op} \leq A_p H s 2^s  \overline L ^{(1)},
                \\  & \|\nabla \psi_{j,H}(x) -\mathbb{I}_d\big \|_{\infty} \leq A_p H \sum_{i=1} ^{j-1} \|\nabla \overline k_i(x) \|_{\infty} \leq A_p H s 2^s \overline K ^{(1)} .
            \end{split}
        \end{align*}
    By the assumption on $H$, all these $\psi_{j,H}$'s are diffeomorphisms on $\mathbb{R}^d$. Similar to the proof of \Cref{lemma: Growth of inverse matrix} and \Cref{lemma: Growth of the original phi_h map}, we have that 
        \begin{align*}
            \begin{split}
                & \|\psi_{j,H} ^{-1} (x) -x\|_2 \leq 2\|\psi_{j,H}(x) -x\|_2 \leq 2 A_p H s{2}^{s} \big[   \overline W ^{(0)}+\overline W ^{(1)} \|x\|_{2}  \big] ,
                \\  & \big \|(\nabla \psi_{j,H} ^{-1})(x) -\mathbb{I}_d\big \|_{op} \leq 2 \big \|(\nabla \psi_{j,H} )(x) -\mathbb{I}_d\big \|_{op} \leq 2 A_p H s 2^s  \overline L ^{(1)},
                \\  & \big \|(\nabla \psi_{j,H} ^{-1})(x) -\mathbb{I}_d\big \|_{\infty} \leq 2 \big \|(\nabla \psi_{j,H} )(x) -\mathbb{I}_d\big \|_{\infty} \leq 2 A_p H s 2^s  \overline K ^{(1)}.
            \end{split}
        \end{align*}

One can similarly obtain the estimate for $\| \dM_t (x) - \baM_t(x) \|_{\infty}$. Indeed, similar proof holds true when one replaces $\| \dM_t (x) - \baM_t(x) \|_{\infty}$ with the $j$-th coordinate $|\dM_t ^{(j)} (x) - \baM_t ^{(j)} (x)|$ and replaces those $\| \sV_{t + c_{1} H}(\psi_{1,H}(x)) - V_{t + c_{1} H}(\psi_{1,H}(x)) \|_{\infty}$ with $|\sV_{t + c_{1} H} ^{(j)} (\psi_{1,H}(x)) - V_{t + c_{1} H} ^{(j)} (\psi_{1,H}(x)) |$.
        
\end{proof}

\begin{lemma}\label{lemma: errors of phi_H}
    Use the notation $\mathbb{B} \coloneq \frac{A_p+B_p+1}{B_p+1} \cdot (4\bm \beta+1)$, and take $\eta = 10^{-2} \mathbb{B}^{-1} H^{-1/2}$ for $\bm\beta$ obtained in \Cref{lemma: Growth of interpolation maps}.
      Adopt the assumptions and notations in \Cref{p:ABCD}. Take $H$ small enough such that 
    \begin{align}\label{e:H small logd 1}
        H \leq (2d \log (2d))^{-1} 10^{-6} \mathbb{B}^{-2}
    \end{align}
    and
    \begin{align}\label{e:H small logd 2}
        4000 H \cdot (A_p+B_p + 1)  s{2}^{s} \big[ \widetilde L ^{(1)} +\overline L ^{(1)} +  \overline W ^{(0)} \sqrt{d}D  + (\overline W ^{(0)})^2+ \overline W ^{(1)} d D^2 + d \overline K ^{(1)}   \big] \leq \sigma_{T-t} ^2.
    \end{align}
    Then, we have that
    \begin{align}
        \begin{split}\label{e:goal0}
            &\int_{\mathbb{R}^d} \varrho_t(x)  \big\|    \dphi_H(x) -\baphi_H (x)\big\|_{2}  \de x \\   &\leq C_u H B_p \bigg( \sum_{r=1} ^s \int_{\mathbb{R}^d} \varrho_{t+c_r H} (x) \| \sV_{t + c_{r} H}(x) - V_{t + c_{r} H}(x) \|_2  \de x 
        \\  & \quad +s\pi^{-\frac{d}{2}} e^{-\frac{1}{10^4 \mathbb{B}^{2} H} }\big[(\widetilde W ^{(0)} + \overline W ^{(0)}) + \sqrt{d}D (\widetilde W ^{(1)} + \overline W ^{(1)})\big] \bigg).
        \end{split}
    \end{align}
    Similarly, for $m=1,2$, there is a universal constant $C_u >0$ such that
    \begin{align}
        \begin{split}\label{e:goal1}
            &\int_{\mathbb{R}^d} \varrho_t(x) \|x\|^m_2 \big\|    \dphi_H(x) -\baphi_H (x)\big\|_{2}  \de x \\   &\leq C_u H B_p \bigg( \sum_{r=1} ^s \int_{\mathbb{R}^d} \varrho_{t+c_r H} (x) \|x\|^m_2 \| \sV_{t + c_{r} H}(x) - V_{t + c_{r} H}(x) \|_2  \de x 
        \\  & \quad +s\pi^{-\frac{d}{2}} e^{-\frac{1}{10^4 \mathbb{B}^{2} H} }(\sqrt d D)^m \big[(\widetilde W ^{(0)} + \overline W ^{(0)})^m + \sqrt{d}D (\widetilde W ^{(1)} + \overline W ^{(1)})^{m+1}\big] \bigg)
        \\   &\leq C_u H B_p (\sqrt{d}D)^m\bigg( \sum_{r=1} ^s {\bigg(\int_{\mathbb{R}^d} \varrho_{t+c_r H} (x)  \| \sV_{t + c_{r} H}(x) - V_{t + c_{r} H}(x) \|_2 ^2  \de x \bigg)}^{\frac{1}{2}}
        \\  & \quad +s\pi^{-\frac{d}{2}} e^{-\frac{1}{10^4 \mathbb{B}^{2} H} } \big[(\widetilde W ^{(0)} + \overline W ^{(0)})^m + \sqrt{d}D (\widetilde W ^{(1)} + \overline W ^{(1)})^{m+1} \big]\bigg).
        \end{split}
    \end{align}
\end{lemma}
\begin{proof}
First, the assumptions \eqref{e:H small logd 1} and \eqref{e:H small logd 2} enable us to use \Cref{lemma: pth order score error}, \Cref{lemma: score error cut-off general scheme}, and \Cref{lemma: exponential error term in cut-off}.

Then, by \Cref{lemma: pth order score error}, we have
\begin{align}         
\begin{split}\label{e:start}
&\int_{\mathbb{R}^d} \varrho_t(x) \|x\|_2^m  \big\|    \dphi_H(x) -\baphi_H (x)\big\|_{2}  \de x \\ 
 &\leq      
 2 H B_p\int_{\mathbb{R}^d} \varrho_t(x) \|x\|_2^m  \bigg[ \| \sV_{t + c_s H}(\psi_{s,H}(x)) - V_{t + c_s H}(\psi_{s,H}(x)) \|_2 \\ & \quad +  \| \sV_{t + c_{s-1} H}(\psi_{s-1,H}(x)) - V_{t + c_{s-1} H}(\psi_{s-1,H}(x)) \|_2
                \\  &\quad +  \| \sV_{t + c_{s-2} H}(\psi_{s-2,H}(x)) - V_{t + c_{s-2} H}(\psi_{s-2,H}(x)) \|_2
                \\  &\quad + \cdots +   \| \sV_{t + c_{1} H}(\psi_{1,H}(x)) - V_{t + c_{1} H}(\psi_{1,H}(x)) \|_2\bigg] \de x.
            \end{split}
        \end{align}
The right-hand side of \eqref{e:start} is a sum of $s$ terms, and they can be analyzed in the same way. So, we only study the term corresponding to $\| \sV_{t + c_{1} H}(\psi_{1,H}(x)) - V_{t + c_{1} H}(\psi_{1,H}(x)) \|_2$. We notice that $\varrho_{t+c_1 H} (x) = \varrho_{t} (\phi_{c_1 H} ^{-1}(x)) \cdot  \det[ \nabla (\phi_{c_1 H} ^{-1}(x))] $. The estimates for $\phi_{c_1 H} ^{-1}(x)$ are in \Cref{lemma: Growth of the original phi_h map}. Those determinants term can be bounded by universal constants according to our assumption on $H$.  By change of variables,  Lemma~\ref{lemma: Growth of the original phi_h map},  Lemma~\ref{lemma: Growth of inverse matrix} and Lemma~\ref{lemma: pth order score error}, we see that for $m\geq 1$
\begin{align}
    \begin{split}\label{e:start1}
        &\int_{\mathbb{R}^d} \varrho_t(x)  \|x\|^m_2 \| \sV_{t + c_{1} H}(\psi_{1,H}(x)) - V_{t + c_{1} H}(\psi_{1,H}(x)) \|_2  \de x
        \\  &\leq 4\int_{\mathbb{R}^d} \varrho_t(\psi_{1,H} ^{-1}(x)) \|x\|^m_2 \| \sV_{t + c_{1} H}(x) - V_{t + c_{1} H}(x) \|_2  \de x 
        \\  &\leq 24\int_{\mathbb{R}^d} \varrho_t(\phi_{c_1 H} ^{-1}(x))  \cdot  \det[ \nabla (\phi_{c_1 H} ^{-1}(x))] \cdot \|x\|^m_2 \| \sV_{t + c_{1} H}(x) - V_{t + c_{1} H}(x) \|_2  \de x 
        \\  & \quad 
        + 32 E(m,0,(\widetilde W ^{(0)} + \overline W ^{(0)}),\eta) + 200 E(m+1,0,(\widetilde W ^{(1)} + \overline W ^{(1)}),\eta) 
        \\  &=24\int_{\mathbb{R}^d} \varrho_{t+c_1 H} (x) \|x\|^m_2 \| \sV_{t + c_{1} H}(x) - V_{t + c_{1} H}(x) \|_2  \de x 
        \\  & \quad +  32 E(m,0,(\widetilde W ^{(0)} + \overline W ^{(0)}),\eta) + 200 E(m+1,0,(\widetilde W ^{(1)} + \overline W ^{(1)}),\eta) .
    \end{split}
\end{align}
where the first inequality is by change of variables and the growth estimate on $\psi_{1,H} ^{-1}(x)$ in Lemma~\ref{lemma: pth order score error}, and the second inequality is by Lemma~\ref{lemma: score error cut-off general scheme} because $\| \sV_{t + c_{1} H}(x) - V_{t + c_{1} H}(x) \|_2 \leq (\widetilde W ^{(0)} + \overline W ^{(0)}) + (\widetilde W ^{(1)} + \overline W ^{(1)})\|x\|_2$, and we use the growth estimates for  $\psi_{1,H} ^{-1}(x)$ and $\phi_{c_1 H} ^{-1}(x)$ in Lemma~\ref{lemma: pth order score error} and \Cref{lemma: Growth of the original phi_h map}.
Similarly, for $m=0$, we have
\begin{align}\label{e:start2}
    \begin{split}
        &\int_{\mathbb{R}^d} \varrho_t(x)   \| \sV_{t + c_{1} H}(\psi_{1,H}(x)) - V_{t + c_{1} H}(\psi_{1,H}(x)) \|_2  \de x
       \\  &\leq 24\int_{\mathbb{R}^d} \varrho_{t+c_1 H} (x) \| \sV_{t + c_{1} H}(x) - V_{t + c_{1} H}(x) \|_2  \de x 
        + 32 E(1,(\widetilde W ^{(0)} + \overline W ^{(0)}),(\widetilde W ^{(1)} + \overline W ^{(1)}),\eta)  .
   \end{split}
\end{align}
Hence, by our choice of $\eta$ in the assumption (as in Lemma~\ref{lemma: exponential error term in cut-off}), after plugging \eqref{e:start1} and \eqref{e:start2} into \eqref{e:start}, we conclude that there is a universal constant $C_u >0$ such that for $m=0$
   \begin{align}
        \begin{split}\label{e:start3}
            &\int_{\mathbb{R}^d} \varrho_t(x) \big\|  \baphi_H (x) - \dphi_H(x) \big\|_{2}  \de x \leq C_u H B_p \bigg( \sum_{r=1} ^s \int_{\mathbb{R}^d} \varrho_{t+c_r H} (x) \| \sV_{t + c_{r} H}(x) - V_{t + c_{r} H}(x) \|_2  \de x 
        \\  & \quad +s \pi^{-\frac{d}{2}} e^{-\frac{1}{10^4 \mathbb{B}^{2} H} }  \big[(\widetilde W ^{(0)} + \overline W ^{(0)}) + \sqrt{d}D{(\widetilde W ^{(1)} + \overline W ^{(1)})} \big]\bigg).
        \end{split}
    \end{align}
The claim \eqref{e:goal0} follows.   
Similarly, for $m\geq 1$, we have 
    \begin{align}
        \begin{split}\label{e:start4}
            &\int_{\mathbb{R}^d} \varrho_t(x) \|x\|_2^m \big\|  \baphi_H (x) - \dphi_H(x) \big\|_{2}  \de x \leq C_u H B_p \bigg( \sum_{r=1} ^s \int_{\mathbb{R}^d} \varrho_{t+c_r H} (x) \|x\|_2^m  \| \sV_{t + c_{r} H}(x) - V_{t + c_{r} H}(x) \|_2  \de x 
        \\  & \quad +s\pi^{-\frac{d}{2}} e^{-\frac{1}{10^4 \mathbb{B}^{2} H} }(\sqrt d D)^m \big[(\widetilde W ^{(0)} + \overline W ^{(0)})^m + \sqrt{d}D (\widetilde W ^{(1)} + \overline W ^{(1)})^{m+1}\big] \bigg).
        \end{split}
    \end{align}
    Moreover, H{\"o}lder inequality  and \Cref{lemma: moments of L^2 and infinity norms} imply
\begin{align}\begin{split}\label{e:start5}
    &\int_{\mathbb{R}^d} \varrho_{t+c_r H} (x) \|x\|_2^m  \| \sV_{t + c_{r} H}(x) - V_{t + c_{r} H}(x) \|_2  \de x\\
    &\leq \left(\int_{\mathbb{R}^d} \varrho_{t+c_r H} (x) \|x\|_2^{2m}  \de x  \right)^{1/2}\left(\int_{\mathbb{R}^d} \varrho_{t+c_r H} (x)  \| \sV_{t + c_{r} H}(x) - V_{t + c_{r} H}(x) \|^2_2  \de x \right)^{1/2}\\
    &\leq C_u (\sqrt d D)^m\left(\int_{\mathbb{R}^d} \varrho_{t+c_r H} (x) \| \sV_{t + c_{r} H}(x) - V_{t + c_{r} H}(x) \|^2_2  \de x \right)^{1/2}.
\end{split}\end{align}  
    The statement \eqref{e:goal1} follows from combining \eqref{e:start4} and \eqref{e:start5}.

\end{proof}

\subsection{Proof of \Cref{Prop: Conclusion second error term}}
\label{ssec:second-error-proof}
\begin{proof}[Proof of \Cref{Prop: Conclusion second error term}]
    We combine the estimates in this section with the bound $\sigma_{T-t}^2 \geq C_u (T-t) \geq C_u \tau$. We recall from \Cref{lemma: Second derivative bound on errors}, 
     \begin{align}\label{e:I2bb3}
            \begin{split}
                &\sum_{i=1} ^d \sum_{j=1} ^d \int_{\mathbb{R}^d} \varrho_t(x) \big| \partial_i \dphi_H ^{(j)} (x) - \partial_i \baphi_H ^{(j)} (x) \big| \de x  
            \\  &\leq \frac{C_u d}{\tau}  \int_{\mathbb{R}^d} \varrho_t(x)  \|x\|_2  \cdot \|  \dphi_H  (x) -  \baphi_H (x) \|_2 \de x + \frac{C_u d^{3/2}D}{\tau}  \int_{\mathbb{R}^d} \varrho_t(x)   \cdot \|  \dphi_H  (x) -  \baphi_H (x) \|_2 \de x \\
            &+\frac{C_u  d}{\tau} {\bigg (  \int_{\mathbb{R}^d}  \varrho_t(x) \| (\dphi_H (x) -  \baphi_H  (x)) \|_{2} \de x \bigg)}^{\frac{1}{2}}{\bigg ( \int_{\mathbb{R}^d}   \varrho_t(x) \|x\|_2 ^2 \|\dphi_H  (x) -  \baphi_H (x) \|_2 \de x \bigg)}^{\frac{1}{2}} \\
            &+\frac{C(\widetilde K) {H}^{\frac{1}{2}} d^{\frac{7}{4}} {B_p}^{\frac{1}{2}}D^{3/2}}{\tau^{3/2}} {\bigg (  \int_{\mathbb{R}^d}  \varrho_t(x) \| (\dphi_H (x) -  \baphi_H  (x)) \|_{2} \de x \bigg)}^{\frac{1}{2}}, 
            \end{split}
        \end{align}
    where we used H{\"o}lder's inequality and \Cref{lemma: Constant Convention1} to simplify the expression. 

    We recall the bound $\|\nabla^2 \baphi_H \|_{\infty} \leq 48 H B_p s^3 4^s D^3/\tau^3$ from \Cref{lemma: Constant Convention} and \Cref{lemma: Constant Convention1}. By our assumption \eqref{e:Hcondition}, and $\overline{K}^{(1)}=2D^2\tau^{-2}$ from \eqref{e:tbound2}, we have $B_p dH\overline{K}^{(1)}=2B_pdHD^2\tau^{-2}\leq 1$.
    Using \eqref{e:I2bb3}, we can estimate the first three terms in \eqref{e:I2bb} as
    \begin{align}\begin{split}
        &C(s, B_p, \widetilde K)\Bigg[\frac{d^{3/2} D}{\tau}
        \int_{\mathbb{R}^d} \varrho_t (x) \cdot  \| \baphi_H(x) - \dphi_H(x) \|_2   \de x +    \frac{ d}{\tau}
        \int_{\mathbb{R}^d} \varrho_t (x) \cdot \|x\|_2  \cdot \| \baphi_H(x) - \dphi_H(x) \|_2   \de x\\
        &+\frac{  d}{\tau} {\bigg (  \int_{\mathbb{R}^d}  \varrho_t(x) \| (\dphi_H (x) -  \baphi_H  (x)) \|_{2} \de x \bigg)}^{\frac{1}{2}}{\bigg ( \int_{\mathbb{R}^d}   \varrho_t(x) \|x\|_2 ^2 \|\dphi_H  (x) -  \baphi_H (x) \|_2 \de x \bigg)}^{\frac{1}{2}} \\
            &+\frac{{H}^{\frac{1}{2}} d^{\frac{7}{4}} D^{3/2}}{\tau^{3/2}} {\bigg (  \int_{\mathbb{R}^d}  \varrho_t(x) \| (\dphi_H (x) -  \baphi_H  (x)) \|_{2} \de x \bigg)}^{\frac{1}{2}}\Bigg], 
    \end{split}\end{align}
where $C(s, B_p, \widetilde K)$ is a constant depending only on $s, B_p, \widetilde K$. Also, the exponential term  $e^{-\frac{1}{10^4 (4\bm\beta)^{2} H} }$ in \eqref{e:I2bb2} can be replaced by $e^{-\frac{1}{10^4 (\mathbb{B})^{2} H} }$ because $\mathbb{B} \geq 4\bm\beta$, where $\mathbb{B}$ is defined in \Cref{lemma: errors of phi_H}. Putting together \eqref{e:I2bb}, \eqref{e:I2bb2}, \Cref{lemma: errors of phi_H} and \eqref{e:I2bb3}, we conclude:
        \begin{align}\label{e:final1}
        \begin{split}
        &\int_{\mathbb{R}^d} | \varrho_{t} (\dphi_H ^{-1}(x)) \cdot | \det[ \nabla (\dphi_H ^{-1}(x))]| -  \varrho_{t} (\baphi_H ^{-1}(x)) \cdot | \det[ \nabla (\baphi_H ^{-1}(x))]| | \de x
            \\  & \leq  C(s, B_p, \widetilde K, \widetilde W) H  d^{\frac{3}{2}} \tau^{-1} D\bigg[   \bigg( \sum_{r=1} ^s {\bigg(\int_{\mathbb{R}^d} \varrho_{t+c_r H} (x)  \| \sV_{t + c_{r} H}(x) - V_{t + c_{r} H}(x) \|_2 ^2  \de x \bigg)}^{\frac{1}{2}}\bigg)   
        \\  &\quad + d^{\frac{1}{4}} D^{\frac{1}{2}} \tau^{-\frac{1}{2}}    \bigg( \sum_{r=1} ^s {\bigg(\int_{\mathbb{R}^d} \varrho_{t+c_r H} (x)  \| \sV_{t + c_{r} H}(x) - V_{t + c_{r} H}(x) \|_2 ^2  \de x \bigg)}^{\frac{1}{4}} \bigg)
      +  \tau^{-4} d^2 D^6    {\pi}^{-d/4} e^{-\frac{1}{2\cdot 10^4 (\mathbb{B})^{2} H} }  \bigg],
        \end{split}
    \end{align}
where $C(s, B_p, \widetilde K, \widetilde W)$ is a constant depending only on $s, B_p, \widetilde K, \widetilde W$.

To further adopt the second assumption \eqref{e:score_bound} in \Cref{t:RK_1step} and to simplify our notations, we denote 
    \begin{align*}
        \epsilon_{r} \coloneq \int_{\mathbb{R}^d} \varrho_{t + c_{r} H} (x)  \| \sV_{t + c_{r} H}(x) - V_{t + c_{r} H}(x) \|_2 ^2  \de x , \quad \eta_{r} \coloneq \sigma_{T-(t +c_r H)} ^2.
    \end{align*}
According to the above inequality \eqref{e:final1}, we see that
    \begin{align*}
        \begin{split}
             &\sum_{r=1} ^s \epsilon_{r} ^{\frac{1}{2}} =  \sum_{r=1} ^s \left(\eta_{r} \epsilon_{r} \right) ^{\frac{1}{2}} \eta_{r} ^{-\frac{1}{2}} \leq \left(\sum_{r=1} ^s \eta_{r} \epsilon_{r} \right) ^{\frac{1}{2}} \left(\sum_{r=1} ^s \eta_{r} ^{-1} \right) ^{\frac{1}{2}}  \leq \varepsilon_{\rm score}(t) s ^{\frac{1}{2}} \tau ^{-\frac{1}{2}},
        \end{split}
    \end{align*}
where in the last inequality, we used the second assumption in \Cref{t:RK_1step} and the fact that $\sigma_t ^2 \geq \tau$ for any $t \in [\tau,T]$.
Similarly, 
    \begin{align*}
        \sum_{r=1} ^s \epsilon_{r} ^{\frac{1}{4}}
        \leq (\varepsilon_{\rm score} (t)) ^{\frac{1}{2}} s^{\frac{3}{4}} \tau ^{-\frac{1}{4}}.
    \end{align*}
Hence, there is a $C_{\rm disc} $ depending on $A_p, B_p,  s, p, D,\widetilde K, \widetilde W$, such that 
    \begin{align*}
        \begin{split}
        &\int_{\mathbb{R}^d} | \varrho_{t} (\dphi_H ^{-1}(x)) \cdot | \det[ \nabla (\dphi_H ^{-1}(x))]| -  \varrho_{t} (\baphi_H ^{-1}(x)) \cdot | \det[ \nabla (\baphi_H ^{-1}(x))]| | \de x
            \\  & \leq C_{\rm disc}  \tau^{-2}  H\bigg[  d^{\frac{7}{4}}   \cdot (\varepsilon_{\rm score}(t)) ^{\frac{1}{2}} +  \tau^{-3}  {\pi}^{-\frac{d}{10}} e^{-\frac{1}{2\cdot 10^4 (\mathbb{B})^{2} H} }   \bigg] .
        \end{split}
    \end{align*}
Finally, because $\mathbb{B} = \frac{A_p+B_p+1}{B_p+1} \cdot (4\bm \beta+1)$ and $\bm \beta \leq 2B_p s 2^s (\widetilde W ^{(1)} + \overline W ^{(1)})$ according to \Cref{lemma: Growth of interpolation maps} and \Cref{lemma: Constant Convention}, we see that $2\cdot 10^4 (\mathbb{B})^{2} \leq \xi$.
This finishes  the proof of \Cref{Prop: Conclusion second error term}.
\end{proof}

\section{Proof of \Cref{Prop: Conclusion third error term}}\label{Section: third term conclusion}
This section proves the oracle Runge-Kutta discretization estimate for $I_3$. The first part constructs the interpolation map $F_h$, introduces the auxiliary velocity $\overline V_{t+h}$, and reduces the total variation error to a continuity-equation stability bound. \Cref{ssec:high-order-velocity-proofs} proves the high-order velocity error estimate and the derivative bounds needed to obtain the final $O(H^{p+1})$ contribution.

Throughout this section, we assume that $t+H \leq T- \tau$ and that $t,H$ are fixed. We will also use the estimates for $\overline K ^{(1)}$ and $\overline L ^{(1)}$ as shown in~\Cref{p:ABCD}.

We recall $I_3$ from \eqref{e:RK2-I123-Appendix} that
    \begin{align}\label{e:I3_copy}
       I_3 = \int_{\mathbb{R}^d} |\overline \varrho_{t+H}(x)-\varrho_{t+H}(x)| \de x,
    \end{align}
where $\overline \varrho_{t+H}(x)$ is as defined in \eqref{e:brhoth}.

For any $0\leq h\leq H$, we introduce the following function $F_h(x)$
    \begin{align}\label{e:Phitmap0}
        F_h(x) \coloneq x + h \sum_{j=1}^s b_j(h)  k_j(x),
    \end{align}
with
    \begin{align}\begin{split}\label{e:RK2}
    &k_1(x) = V_{t + h c_1 }(x),\\
    & k_2(x) =  V_{t + h c_2 }\bigl(x + h(a_{21}(h) k_1(x)) \bigr),\\
    & k_3(x) = V_{t + h c_3 }\bigl(x + h (a_{31}(h) k_1 + a_{32}(h) k_2(x)) \bigr),\\
    &\qquad\qquad\qquad\vdots
    \\
    & k_s(x) =  V_{t + h c_s }\bigl(x + h (a_{s1}(h) k_1(x) + a_{s2}(h) k_2 + \cdots + a_{s,s-1}(h) k_{s-1}(x)) \bigr).
\end{split}\end{align}
We remark that these $k_i(x)$'s depend on $h$, and they are different from the $\overline k_i(x)$'s we defined in \eqref{e: p RK interpolation terms}, because now we use a varying $h$ instead of a fixed $H$. But we see that $\baphi_H(x) = F_H(x)$, and $\overline{\varrho}_{t+H}$ is also the pushforward of the law $\varrho_t$ by $F_H(x)$. Denote
\begin{align}\label{e:defbaY}
    \overline{Y}_{t+h}=F_h(Y_t), \quad 0\leq h\leq H, \quad \overline{Y}_{t}=Y_t,
\end{align}
then $\overline{\varrho}_{t+H}$ is the law of $\overline{Y}_{t+H}$. Moreover, $k_i(x)$ enjoys the same estimates as $\bar k_i(x)$ as in \Cref{lemma: infinity norm pth order growth in RK}, \Cref{lemma: pth order growth in RK} and \Cref{lemma: infinity norm for k_i(x)} with $H$ replaced by $h$.
   
   Because $V_{t+h}(\cdot)$ is differentiable in $h$ on $[0,H]$, one can see from the above construction \eqref{e:RK2}, that $F_h(\cdot)$ is also differentiable in $h$. By taking $h$ derivative on both sides of \eqref{e:defbaY}
\begin{align}\label{e:dttY}
    \del_h \overline Y_{t+h} =\del_h F_h(\overline Y_{t}),\quad 0\leq h\leq H.
\end{align}

Let us first assume that for all $0\leq h\leq H$, $F_h$ is invertible, which will be verified later in \Cref{lemma: high_order_error}. Let 
\begin{align}\label{e:discreteODE}
    \overline V_{t+h}(x) \coloneq (\partial_h F_h)(F_h ^{-1}(x)),
\end{align}
then we can rewrite \eqref{e:dttY} as the following ODE flow
\begin{align}\label{e:dttY2}
     \del_h \overline Y_{t+h} = \overline V_{t+h}(\overline Y_{t+h}),\quad 0\leq h\leq H.
\end{align}
We also recall the ODE flow for $Y_{t+h}$ from \eqref{eq:reverse-ode}
\begin{align}\label{e:Ytflow}
     \del_h  Y_{t+h} =V_{t+h}( Y_{t+h}),\quad 0\leq h\leq H.
\end{align}
Then, according to \Cref{theorem: L^1 error}, using \eqref{e:dttY2} and \eqref{e:Ytflow} we can bound the total variation distance in \eqref{e:I3_copy}, as 
    \begin{align}\label{e:I3_copy2}
        I_3\leq \int_{0} ^H \int_{\mathbb{R}^d}  \bigg|  \nabla \cdot( (\varrho_{t+h}(x)(V_{t+h}(x) - \overline V_{t+h}(x) )) \bigg|\de x \de h.
    \end{align}

The following lemma states that for $H$ small enough, for any $0\leq h\leq H$, $F_h$ is invertible, and $\overline V_{t+h}(\cdot)$ is close to $V_{t+h}(\cdot)$ up to an error of size $\cO(H^p)$.

\begin{lemma}\label{lemma: high_order_error}
Adopt the assumptions in \Cref{Prop: Conclusion third error term}.
Denote $B:=1+A_p+B_p$. There exists a large constant $C=C(p,s, B)$, such that if $20 s 2^sBHd D^2 \tau^{-2}\leq 1$, then the following holds. For any $0\leq h\leq H$, $F_h$
    is a global diffeomorphism from $\bR^d$ to $\bR^d$. We denote its functional inverse as $F^{-1}_{h}(x)$, and
    \begin{align}\label{e:def_wtV}
      \overline V_{t+h}(x)=  (\del_h F_h)(F_h ^{-1} (x)).
    \end{align}
Then, for any $0\leq h\leq H$
    \begin{align}\label{e:stdiff}
        \begin{split}\|\overline V_{t+h}(x)- V_{t+h}(x)\|_{\infty} &\leq C\Big[  {(H\sqrt{d}(\|x\|_{2} + \sqrt{d} D))}^p {(\tau^{-1})}^{2^{p+1} p! + 2p} {(\|x\|_{\infty} +D)}^{3^{p+1} p!}\\
        &+ {(H\sqrt{d}(\|x \|_{2} + \sqrt{d} D))}^{p+1} {(\tau^{-1})}^{2^{p+2} (p+1)! + 2(p+1)} {(\|x\|_{\infty} +D)}^{3^{p+2} (p+1)!}\Big],
   \end{split} \end{align}
   and the same estimate holds for $ \|\nabla (\overline V_{t+h}(x)- V_{t+h}(x))\|_{\infty}$.
\end{lemma}

Before proving \Cref{lemma: high_order_error}, we need another lemma.

\begin{lemma}\label{lemma:V derivative bound}
Adopt the assumptions in \Cref{Prop: Conclusion third error term}.
Denote $B:=1+A_p+B_p$. There exists a large constant $C=C(p,s, B)$, such that if $20 s 2^sBHd D^2 \tau^{-2}\leq 1$, then the following holds for any $0 \leq h \leq H$.
    \begin{align}\begin{split}\label{e:dtVbound}
        \|\del_h^p \baV_{t+h}(F_h(x))\|_\infty, \ 
         \|\del_h^p\nabla \baV_{t+h}(F_h(x))\|_\infty&\leq \frac{C  {[\sqrt{d}(\|x\|_{2} + \sqrt{d} D)]}^p  {(\|x\|_{\infty} +D)}^{3^{p+1} p!}}{\tau^{2^{p+1} p! + 2p}}\\
         &+Ch \frac{  {[\sqrt{d}(\|x\|_{2} + \sqrt{d} D)]}^{p+1}  {(\|x\|_{\infty} +D)}^{3^{p+2} (p+1)!}}{\tau^{2^{p+2} (p+1)! + 2(p+1)}},
    \end{split}\end{align}
and 
\begin{align}\begin{split}\label{e:dhVbound}
        \|\del_h^p V_{t+h}(F_h(x))\|_\infty , \ 
         \|\del_h^p\nabla V_{t+h}(F_h(x))\|_\infty
         &\leq \frac{C  {[\sqrt{d}(\|x\|_{2} + \sqrt{d} D)]}^p  {(\|x\|_{\infty} +D)}^{3^{p+1} p!}}{\tau^{2^{p+1} p! + 2p}}\\
         &+Ch \frac{  {[\sqrt{d}(\|x\|_{2} + \sqrt{d} D)]}^{p+1}  {(\|x\|_{\infty} +D)}^{3^{p+2} (p+1)!}}{\tau^{2^{p+2} (p+1)! + 2(p+1)}}.
    \end{split}\end{align}
\end{lemma}

Now, we can estimate the term $I_3$ in \eqref{e:RK2-I3-continuity}.

\begin{proof}[Proof of \Cref{Prop: Conclusion third error term}]
    
 By~\Cref{lemma: high_order_error}, we let $\gamma(p) \coloneq 3^{p+1} p!$ and see that there is a constant $C(p,s,B)>0$, such that
    \begin{align}\label{e:int1}
        \begin{split}
            &\int_{0} ^H \int_{\mathbb{R}^d}  \bigg|  \nabla \cdot ((\varrho_{t+h}(x)(V_{t+h}(x) - \overline V_{t+h}(x) )) \bigg|\de x \de h
    \\  &   \leq
        C   {[Hd^{\frac{1}{2}}  ]}^p {(\tau^{-1})}^{2^{p+1} p! + 2p} \int_{0} ^H \int_{\mathbb{R}^d} \varrho_{t+h}(x) (d+ \frac{\|x\|_1 +d D}{\tau}){(\|x\|_{\infty} +D)}^{\gamma(p)} {(\sqrt{d}D+ \|x\|_{2})}^p \de x \de h\\
      &  +C   {[Hd^{\frac{1}{2}}  ]}^{p+1} {(\tau^{-1})}^{2^{p+2} (p+1)! + 2p+2} \int_{0} ^H \int_{\mathbb{R}^d} \varrho_{t+h}(x) (d+ \frac{\|x\|_1 +d D}{\tau}){(\|x\|_{\infty} +D)}^{\gamma(p+1)} {(\sqrt{d}D+ \|x\|_{2})}^{p+1} \de x \de h,
        \end{split}
    \end{align}
where we also used \Cref{Lemma: Hessian estimates} to replace $\nabla \varrho_{t+h}(x)$ terms with $\varrho_{t+h}(x)$.
By~\Cref{lemma: moments of L^2 and infinity norms}, we know that there is a positive constant $C(m)$, such that
    \begin{align*}
        \int_{\mathbb{R}^d} \|x\|_2 ^m \varrho_{t+h}(x) \ \de x \leq C(m) {(\sqrt{d}D)}^m, \quad \int_{\mathbb{R}^d}  \|x\|_{\infty} ^m \varrho_{t+h}(x) \de y \leq C(m) {\big(D+\sqrt{ \log d}\big)}^m .
    \end{align*}
Hence, by H{\"older} inequality,
    \begin{align}\label{e:int2}
        \begin{split}
        &\phantom{{}={}}\int_{0} ^H \int_{\mathbb{R}^d} \varrho_{t+h}(x) (d+ \frac{\|x\|_1 +d D}{\tau}){(\|x\|_{\infty} +D)}^{\gamma(p)} {(\sqrt{d}D+ \|x\|_{2})}^p \de x \de h\\
            &\leq\frac{\sqrt{d}}{\tau}\int_{0} ^H \int_{\mathbb{R}^d} \varrho_{t+h}(x){(\|x\|_{\infty} +D)}^{\gamma(p)} {(\sqrt{d} D+ \|x\|_{2})}^{p+1} \de x \de h \\  &\leq C(p) \frac{\sqrt{d} H}{\tau} {(\sqrt{d}D)}^{p+1} {\big(D+\sqrt{ \log d}\big)}^{\gamma(p)} .
        \end{split}
    \end{align}
The same argument implies
\begin{align}\begin{split}\label{e:int3}
 &\phantom{{}={}}\int_{0} ^H \int_{\mathbb{R}^d} \varrho_{t+h}(x) (d+ \frac{\|x\|_1 +d D}{\tau}){(\|x\|_{\infty} +D)}^{\gamma(p+1)} {(\sqrt{d}D+ \|x\|_{2})}^{p+1} \de x \de h
 \\  &\leq C(p+1) \frac{\sqrt{d} H}{\tau} {(\sqrt{d}D)}^{p+2} {\big(D+\sqrt{ \log d}\big)}^{\gamma(p+1)} .
\end{split}\end{align}

We can then conclude from \eqref{e:I3_copy2}, \eqref{e:int1}, \eqref{e:int2} and \eqref{e:int3} that
    \begin{align*}
    \begin{split}
           I_3 &\leq \int_{0} ^H \int_{\mathbb{R}^d}  \bigg|  \nabla \cdot ((\varrho_{t+h}(x)(V_{t+h}(x) - \overline V_{t+h}(x) )) \bigg|\de x \de h 
            \\    &\leq  C(p,s,B) H^{p+1} d^{p+1}D^{p+1}  {(\tau^{-1})}^{2^{p+2} (p+1)! + 2p+3} {\big(D+\sqrt{ \log d}\big)}^{\gamma(p+1)},
        \end{split}
    \end{align*}
where we used that $dHD<1$.

\end{proof}

\subsection{Proofs of \Cref{lemma: high_order_error} and \Cref{lemma:V derivative bound}}
\label{ssec:high-order-velocity-proofs}

In this section, we first prove \Cref{lemma: high_order_error} assuming \Cref{lemma:V derivative bound}. Then, we give the proof of \Cref{lemma:V derivative bound}.
\begin{proof}[Proof of \Cref{lemma: high_order_error}]
We first show that for any $0\leq h\leq H$, $F_h$ is a global diffeomorphism from $\bR^d$ to $\bR^d$.
Recall $F_h(x)=x+h\sum_{j=1}^s b_j(h) k_j$ from \eqref{e:Phitmap0}. It follows from \Cref{lemma: pth order growth in RK} that 
\begin{align}\begin{split}\label{e:Fr_bound}
    \|F_h(x)-x\|_2&\leq h B \sum_{i=1}^s \|k_i(x)\|_2\leq hs {(1+ \overline W ^{(1)}A_p H)}^{s-1} \big[   \overline W^{(0)} + \overline W ^{(1)} \|x\|_2 \big]\leq \frac{1}{2}(1+\|x\|_2),
\end{split}\end{align}
where we used~\Cref{p:ABCD} and our assumption that $20 s 2^sBHd D^2 \tau^{-2}\leq 1$. It follows that 
\begin{align}\label{e:change}
    1+\|F_h(x)\|_2\geq 1+\|x\|_2-\|F_h(x)-x\|_2\geq \frac{1}{2}(1+\|x\|_2).
\end{align}
Also, by~\Cref{lemma: infinity norm for k_i(x)}, we notice that 
    \begin{align*}
        \|F_h(x)-x\|_{\infty}&\leq 2h B s\tau^{-1}{(1+2\tau ^{-1}A_p H)}^{s-1} \big[D + 2\|x\|_{\infty} \big] \leq \frac{1}{2}(1+\|x\|_{\infty}).
    \end{align*}
Next, we show that $F_h(x)$ is a local diffeomorphism by checking its Jacobian matrix 
    \begin{align}\label{e:Jacobian}
        \nabla F_h(x)={\mathbb I}_d +A,\quad A:=h \sum_{j=1}^s b_j(h) \nabla k_j(x).
    \end{align}
   By \Cref{lemma: infinity norm pth order growth in RK}, the $(i,j)$-th entry of $A$ is bounded by
    \begin{align}\label{e:Aijbound}
        |A_{ij}|\leq h B \sum_{i=1}^s \|\nabla k_i(x)\|_\infty\leq hB s \overline K ^{(1)} {(1+ \overline K ^{(1)} A_p H d)}^{s-1} ,\quad 1\leq i,j\leq d.
    \end{align}
    The operator norm $\|A\|_{op}$ of the matrix $A$ is bounded by its  Frobenius norm as
    \begin{align*}
        \|A\|_{op}\leq \|A\|_{\rm F}\leq \sqrt{\sum_{ij}A_{ij}^2}\leq d hB s \overline K ^{(1)} {(1+ \overline K ^{(1)} A_p H d)}^{s-1}\leq 1/2.
    \end{align*}
    where again we used our assumption that $20 s 2^sBHd D^2 \tau^{-2}\leq 1$. It follows that $\nabla F_h(x)=\mathbb I_d+A$ is invertible, and $F_h$ is a local diffeomorphism. Then, Hadamard-Cacciopoli theorem implies that $F_h$ is also a bijection from $\bR^d$ to itself because $\|F_h(x)\|_2 \to + \infty$ if $\|x\|_2 \to + \infty$ by \eqref{e:change}. Therefore, $F_h$ is a diffeomorphism from $\bR^d$ to itself. Moreover, by \eqref{e:Aijbound}, we have the following entrywise bound for the inverse matrix $(\nabla F_h(x))^{-1}$:
\begin{align}\begin{split}\label{e:DFr_inverse}
       |((\nabla F_h(x))^{-1}-\mathbb I_d)_{ij}|
       &= |(\mathbb I_d +A)_{ij}^{-1}-\delta_{ij}|\leq |(\mathbb I_d
       -A+A^2-A^3+\cdots)_{ij}-\delta_{ij}|\\
       &\leq \sum_{k\geq 1}(hB s \overline K ^{(1)} {(1+ \overline K ^{(1)} A_p H d)}^{s-1})^k d^{k-1}
       \\   &\leq 2hB s \overline K ^{(1)} {(1+ \overline K ^{(1)} A_p H d)}^{s-1},\quad 1\leq i,j\leq d,
    \end{split}\end{align}
   where we used that $8sBhd\tau^{-1}\leq 1$.
  We denote the functional inverse of $F_h$ as $F^{-1}_{h}(x)$, then \eqref{e:def_wtV} follows from \eqref{e:dttY}.

 Next, we show that the claim \eqref{e:stdiff} follows from the following statements: for $0\leq h\leq H$,
\begin{align}\begin{split}\label{e:tV-V}
     &\|\overline V_{t+h}(F_h(x))-V_{t+h}(F_h(x))\|_{\infty}\leq C(p,s, B)   {[h\sqrt{d}(\|x\|_{2} + \sqrt{d} D)]}^p {(\tau^{-1})}^{2^{p+1} p! + 2p} {(\|x\|_{\infty} +D)}^{3^{p+1} p!},\\
     &\|\nabla(\overline V_{t+h}(F_h(x))-V_{t+h}(F_h(x)))\|_{\infty}\leq C(p,s, B)   {[h\sqrt{d}(\|x\|_{2} + \sqrt{d} D)]}^p {(\tau^{-1})}^{2^{p+1} p! + 2p} {(\|x\|_{\infty} +D)}^{3^{p+1} p!}.
 \end{split}\end{align}
In fact, if we denote $y=F_h(x)$, then
\begin{align*}
    \|\overline V_{t+h}(y)-V_{t+h}(y)\|_{\infty}
    &\leq C(p,s, B)   {[h\sqrt{d}(\|F_h ^{-1} (x) \|_{2} + \sqrt{d} D)]}^p {(\tau^{-1})}^{2^{p+1} p! + 2p} {(\|F_h ^{-1} (x)\|_{\infty} +D)}^{3^{p+1} p!}\\
    &\leq C(p,s, B) 2^p 2^{3^{p+1} p!}   {[h\sqrt{d}(\|x \|_{2} + \sqrt{d} D)]}^p {(\tau^{-1})}^{2^{p+1} p! + 2p} {(\|x\|_{\infty} +D)}^{3^{p+1} p!},
\end{align*}
where in the last inequality we used \eqref{e:change} to bound $\|F_h ^{-1} (x) \|_{2} + \sqrt{d} D$ by $2(\|x \|_{2} + \sqrt{d} D)$, and bound $\|F_h ^{-1} (x)\|_{\infty} +D$ by $2(\|x\|_{\infty} +D)$. For the gradient part in \eqref{e:stdiff}, by the chain rule, we have
\begin{align}\label{e:DeltaV}
    (\nabla \overline V_{t+h})(y)-(\nabla  V_{t+h})(y)=\nabla(\overline V_{t+h}(F_h(x))-  V_{t+h}(F_h(x)))(\nabla F_h(x))^{-1}.
\end{align}
By plugging \eqref{e:DFr_inverse} into \eqref{e:DeltaV}, we conclude that
    \begin{align}\begin{split}\label{e:VFdiff}
        &\phantom{{}={}}\|(\nabla \overline V_{t+h})(y)-(\nabla  V_{t+h})(y)\|_\infty \\
        &\leq (1+2dhB s \overline K ^{(1)} {(1+ \overline K ^{(1)} A_p H d)}^{s-1})\|\nabla(\overline V_{t+h}(F_h(x))- V_{t+h}(F_h(x)))\|_{\infty}\\
    &\leq 2\|\nabla(\overline V_{t+h}(F_h(x))- V_{t+h}(F_h(x)))\|_{\infty}\\
    &\leq 2C(p,s, B)   {[h\sqrt{d}(\|x \|_{2} + \sqrt{d} D)]}^p {(\tau^{-1})}^{2^{p+1} p! + 2p} {(\|x\|_{\infty} +D)}^{3^{p+1} p!},
    \end{split}\end{align}
  where in the last inequality we used the gradient part in \eqref{e:tV-V}.

In the rest, we prove \eqref{e:tV-V}. 
We denote the characteristic flow corresponding to $V_{t+h}$ as $\Phi_h(x)$, i.e. $\Phi_0(x)=x$ and $\del_h\Phi_h(x)= V_{t+h}(\Phi_h(x))$. 
The  Runge-Kutta matrix $[a_{jk}]$, weights $b_j$ and nodes $c_j$ are carefully chosen such $\Phi_h(x)$ and $F_h(x)$ matches for the first $p$-th derivative at $h=0$. It follows that for any $0\leq m\leq p$,
\begin{align*}
    \left.\frac{\rd^m \Phi_h(x)}{\rd^m h}\right|_{h=0}=\left.\frac{\rd^m F_h(x)}{\rd^m h}\right|_{h=0},\quad 
    \left.\frac{\rd^m \nabla \Phi_h(x)}{\rd^m h}\right|_{h=0}=\left.\frac{\rd^m \nabla F_h(x)}{\rd^m h}\right|_{h=0}.
\end{align*}
Thus, by the chain rule, we have for $0\leq m\leq p-1$,
\begin{align*}
    &\left.\frac{\rd^m  V_{t+h}(F_h(x))}{\rd^m h}\right|_{h=0} = \left.\frac{\rd^m  V_{t+h}(\Phi_h(x))}{\rd^m h}\right|_{h=0}= \left.\frac{\rd^{m+1} \Phi_h(x)}{\rd^{m+1} h}\right|_{h=0} = \left.\frac{\rd^{m+1} F_h(x)}{\rd^{m+1} h}\right|_{h=0} = \left.\frac{\rd^m  \overline V_{t+h}(F_h(x))}{\rd^m h}\right|_{h=0},\\
    &\left.\frac{\rd^m \nabla  V_{t+h}(F_h(x))}{\rd^m h}\right|_{h=0}=\left.\frac{\rd^m \nabla   \overline V_{t+h}(F_h(x))}{\rd^m h}\right|_{h=0}.
\end{align*}
Then, by Taylor expansion we conclude that
\begin{align}\begin{split}\label{e:Taylor}
    & V_{t+h }(F_h(x))-\overline V_{t+h }(F_h(x))
    =R_1(h,x)-R_2(h,x),\\
    &\nabla ( V_{t+h }(F_h(x))  )-\nabla ( \overline V_{t+h }(F_h(x)) )
    =\nabla R_1(h,x)-\nabla R_2(h,x).
\end{split}\end{align}
where the two remainder terms are given by 
\begin{align*}
    &R_1(h,x)=\frac{1}{p!}\int_0^{h}(h-\tau)^{p-1}\frac{\rd^p V_{t+\tau }(F_\tau(x))}{\rd \tau^p} \de \tau,\\
    &R_2(h,x)=\frac{1}{p!}\int_0^{h}(h-\tau)^{p-1}\frac{\rd^p \overline V_{t+\tau }(F_\tau(x))}{\rd \tau^p} \de \tau .
\end{align*}
We conclude from  \Cref{lemma:V derivative bound} that 
\begin{align}\begin{split}\label{e:gradinfite_bound}
&\phantom{{}={}}\| R_1(h,x)\|_{\infty},\| R_2(h,x)\|_{\infty}
    \|\nabla R_1(h,x)\|_{\infty},\|\nabla  R_2(h,x)\|_{\infty}
    \\  &\leq  C(p,s, B)   {[h\sqrt{d}(\|x \|_{2} + \sqrt{d} D)]}^p {(\tau^{-1})}^{2^{p+1} p! + 2p} {(\|x\|_{\infty} +D)}^{3^{p+1} p!}\\
    &+C(p,s, B)   {[h\sqrt{d}(\|x \|_{2} + \sqrt{d} D)]}^{p+1} {(\tau^{-1})}^{2^{p+2} (p+1)! + 2(p+1)} {(\|x\|_{\infty} +D)}^{3^{p+2} (p+1)!}.
\end{split}\end{align}
The estimates \eqref{e:Taylor} and \eqref{e:gradinfite_bound} together give \eqref{e:tV-V}.
This finishes the proof of \Cref{lemma: high_order_error}. 
\end{proof}

\begin{proof}[Proof of~\Cref{lemma:V derivative bound}]
    We prove \eqref{e:dtVbound} first. We use the notation $X\lesssim Y$ if there exists a constant $C$ depending only on $p,s,B$ such that $|X|\leq C Y$. Given a symmetric tensor $T=(T_{j_1 j_2 \cdots j_{\al+1}})\in \bR^{d^{\al+1}}$, and vectors $u_1, u_2, \cdots, u_\al\in \bR^d$, we denote the vector $T[u_1, u_2, \cdots, u_\al]\in \bR^d$ as
\begin{align*}
    (T[u_1, u_2, \cdots, u_\al])_j
    =\sum_{1\leq j_1,j_2,\cdots, j_\al\leq d}T_{j j_1 j_2 \cdots j_\al}u_1^{(j_1)}u_2^{(j_2)}\cdots u_\al^{(j_\al)}.
\end{align*}
For any matrix $A=(A_{j_1j_2})\in \bR^{d\times d}$, we denote the matrix $T[u_1, u_2, \cdots, u_{\al-1}, A]\in \bR^{d\times d}$ as
\begin{align*}
    (T[u_1, u_2, \cdots, u_{\al-1},A])_{jj'}
    =\sum_{1\leq j_1,j_2,\cdots, j_\al\leq d}T_{j j_1 j_2 \cdots j_\al}u_1^{(j_1)}u_2^{(j_2)}\cdots u_{\al-1}^{(j_{\al-1})}A_{j_\al j'}.
\end{align*}

    Recall that $F_h(x)=x+h\sum_{j=1}^s b_j(h) k_j(x)$. So, by \eqref{e:discreteODE}, we have
    \begin{align}\label{e:tV=dF}
        \overline V_{t+h}(F_h(x))=\del_h F_h(x)=\sum_{j=1}^s \del_h (h b_j(h) k_j(x)).
    \end{align}

   In the following, we prove by induction on $j \in \llbracket 1 , s \rrbracket$ that for any $1\leq m\leq p+1$,
   \begin{align}\label{e:kjbound}
      \|\del_h^m k_j\|_\infty, \|\del_h^m\nabla  k_j\|_\infty\lesssim  {(\sqrt{d}(\|x\|_{2} + \sqrt{d} D))}^{m} {(\|x\|_{\infty} +D)}^{(3^{m+1} m!)}/{ \tau}^{2^{m+1} m! + 2m}.
   \end{align}
   We recall from \Cref{a:bounds on RK} that $|\del^l_h b_j(h)|\leq B_p{/}\tau^l$ for $0\leq l\leq p+1$. Then the claim \eqref{e:dtVbound} follows from plugging \eqref{e:kjbound} to \eqref{e:tV=dF}:
\begin{align*}\begin{split}
    \|\del_h^p \baV_{t+h}(F_h(x))\|_\infty
   =\|\del_h^{p+1} F_h(x)\|_\infty 
   &\leq C\frac{  {[\sqrt{d}(\|x\|_{2} + \sqrt{d} D)]}^{p}  {(\|x\|_{\infty} +D)}^{3^{p+1} p!}}{\tau^{2^{p+1} p! + 2p}},\\
   &+Ch \frac{  {[\sqrt{d}(\|x\|_{2} + \sqrt{d} D)]}^{p+1}  {(\|x\|_{\infty} +D)}^{3^{p+2} (p+1)!}}{\tau^{2^{p+2} (p+1)! + 2(p+1)}}.
\end{split}\end{align*}
   Since $k_1(x)= V_{t+h c_1}(x)$, by \Cref{lemma: time-space derivatives log qt}, \eqref{e:kjbound} holds for $j=1$ and any $1\leq m\leq p+1$. In the following we assume statement \eqref{e:kjbound} holds for $j-1$, we prove it for $j$ by induction on $m \in \llbracket 1 , p+1 \rrbracket$.

   We define the following set of vectors $\cD_0, \cD_1,\cD_2, \cdots$. Let $\cD_0=\{k_1, k_2, \cdots, k_j\}$. Then, by \Cref{lemma: pth order growth in RK} and \Cref{p:ABCD}, for any $v\in \cD_0$,
   \begin{align*}
       \|v\|_2\leq 2(\overline W ^{(0)} + \overline W ^{(1)}\|x\|_{2}) \leq 6dD^2 \tau^{-2} (1+ \|x\|_{2}).
   \end{align*}

   For $m\geq 1$, $\cD_m$ is defined as the set of vectors in the following form: for  $\beta\geq 1, \zeta \geq 0, \theta\geq 0, \al\leq \beta\leq m$,
    \begin{align}\label{e:defcDm}
        {\mathcal P}_\theta \cdot h^{\zeta}\cdot \del_h^{\beta-\al}\nabla^{\al} V_{t+h c_j}(x+h(a_{j1}(h)k_1+a_{j2}(h)k_2+\cdots+a_{jj-1}(h)k_{j-1}))[u_1, u_2, \cdots, u_\al],
    \end{align}
    where $\del_h^{\beta-\al}\nabla^{\al}  V_{t+h c_j}\in \bR^{d^{\al+1}}$ is a tensor, for each $1\leq \gamma\leq \al$, \begin{align*}
        u_\gamma\in \{\del_h^{\ell_\gamma} k_1, \del_h^{\ell_\gamma} k_2, \cdots, \del_h^{\ell_\gamma} k_{j-1} \},
    \end{align*}
    for some $\ell_\gamma\geq 0$, and ${\mathcal P}_\theta$ is of the form
    \begin{align*}
        \prod_{1\leq i<j\leq s}\del_h^{\ell_{ji}}a_{ji}(h),\quad \ell_{ji}\geq 0, \quad \theta=\sum_{1\leq i<j\leq s} \ell_{ji}.
    \end{align*}
Moreover, there exist nonnegative integers $m_0+m_1+m_2+\cdots+m_6+m_7=m$ such that
    \begin{align}\label{e:parameter}
    \begin{split}
        &\zeta=-m_1+m_4+m_7,\quad  \al=m_3+m_4+m_7,\quad \beta = m_2 + m_3 + m_4+m_7\\ &\theta=m_0+m_7,\quad
    \sum_{\gamma}\ell_\gamma=m_4+m_5+m_6,
        \quad 
        \sum_{\gamma:\ell_\gamma\geq 1}1=m_4+m_5.
    \end{split}
    \end{align}

    Next, we show that for each $v\in \cD_m$, $\del_h v$ is a linear combination of at most $C(s,p)$ terms in $\cD_{m+1}$, with coefficients bounded by $\max\{p,B\}$. Say $v$ is given in \eqref{e:defcDm}, satisfying \eqref{e:parameter}. By the chain rule, there are several cases:
    \begin{enumerate}
     \setcounter{enumi}{-1}
    \item If $\del_h$ hits $\mathcal P_\theta$ then we get
    \begin{align*}
        (\mathcal \del_h\mathcal P_\theta)\cdot h^{\zeta}\cdot \del_h^{\beta-\al}\nabla^{\al}  V_{t+h c_j}[ u_1, u_2, \cdots, u_\al],
    \end{align*}
      which is a sum of terms in $\cD_{m+1}$, with new parameters \eqref{e:parameter} as $(m_0',m_1', m_2',m_3', m_4', m_5',m_6',m_7')=(m_0+1,m_1, m_2,m_3, m_4, m_5,m_6, m_7)$.
    \item If $\del_h$ hits $h^\zeta$, we get
    \begin{align*}
        \mathcal P_\theta\cdot h^{\zeta-1}\cdot \del_h^{\beta-\al}\nabla^{\al} V_{t+h c_j}[u_1, u_2, \cdots, u_\al]\in \cD_{m+1},
    \end{align*}
    which is in $\cD_{m+1}$, with the parameters \eqref{e:parameter} given by $(m_0',m_1', m_2',m_3', m_4', m_5',m_6',m_7')=(m_0,m_1+1, m_2,m_3, m_4, m_5,m_6,m_7)$.
    \item If $\del_h$ hits $t+hc_j$ in $V_{t+h c_j}$, we get
    \begin{align*}
        \mathcal P_\theta\cdot h^\zeta\cdot \del_h^{(\beta+1)-\al}\nabla^{\al}  V_{t+h c_j}[u_1, u_2, \cdots, u_\al]\in \cD_{m+1},
    \end{align*}
    which is in $\cD_{m+1}$, with the parameters \eqref{e:parameter} given by  $(m_0',m_1', m_2',m_3', m_4', m_5',m_6',m_7')=(m_0,m_1, m_2+1,m_3, m_4, m_5,m_6, m_7)$.
    
    \item If $\del_h$ hits $h$  in $(x+h(a_{j1}(h)k_1+a_{j2}(h)k_2+\cdots+a_{jj-1}(h)k_{j-1})$, we get
    \begin{align*}
        \sum_{1\leq i<j}(a_{ji}(h)\mathcal P_\theta)\cdot h^{\zeta}\cdot \del_h^{\beta-\al}\nabla^{\al+1} V_{t+h c_j}[k_i, u_1, u_2, \cdots, u_\al],
    \end{align*}
    which is in $\cD_{m+1}$, with the parameters \eqref{e:parameter} given by $(m_0',m_1', m_2',m_3', m_4', m_5',m_6',m_7')=(m_0,m_1, m_2,m_3+1, m_4, m_5,m_6,m_7)$. 
    \item If $\del_h$ hits $k_i$ in $(x+h(a_{j1}k_1+a_{j2}k_2+\cdots+a_{jj-1}k_{j-1})$ for some $1\leq i\leq j-1$, we get
    \begin{align*}
        {\mathcal P}_{\theta} \cdot h^{\zeta+1} \cdot \del_h^{\beta-\al}\nabla^{\al+1} V_{t+h c_j}[\del_h k_i, u_1, u_2, \cdots, u_\al]\in \cD_{m+1},
    \end{align*}
    which is in $\cD_{m+1}$ with the parameters \eqref{e:parameter} given by $(m_0',m_1', m_2',m_3', m_4', m_5',m_6',m_7')=(m_0, m_1, m_2,m_3, m_4+1, m_5,m_6,m_7)$.
    
    \item If $\del_h$ hits $u_\gamma$ and $u_\gamma=k_i$ we get
    \begin{align*}
        {\mathcal P}_{\theta} \cdot h^{\zeta}\cdot \del_h^{\beta-\al}\nabla^{\al}  V_{t+h c_j}[ u_1, u_2, \cdots, u_{\gamma-1}, \del_h k_i, u_{\gamma+1},\cdots, u_\al]\in \cD_{m+1},
    \end{align*}
    which is in $\cD_{m+1}$, with the parameters \eqref{e:parameter} given by $(m_0',m_1', m_2',m_3', m_4', m_5',m_6',m_7')=(m_0,m_1, m_2,m_3, m_4, m_5+1,m_6,m_7)$.
    \item If $\del_h$ hits $u_\gamma$ and $u_\gamma=\del_h^{\ell_\gamma}k_i$ with $\ell_\gamma\geq 1$ we get
    \begin{align*}
       {\mathcal P}_{\theta} \cdot h^{\zeta}\cdot\del_h^{\beta-\al}\nabla^{\al}  V_{t+h c_j}[ u_1, u_2, \cdots, u_{\gamma-1}, \del_h^{\ell_\gamma+1}k_i, u_{\gamma+1},\cdots, u_\al]\in \cD_{m+1},
    \end{align*}
      which is in $\cD_{m+1}$ with the parameters \eqref{e:parameter} given by $(m_0',m_1', m_2',m_3', m_4', m_5',m_6',m_7')=(m_0,m_1, m_2,m_3, m_4, m_5,m_6+1,m_7)$.
      \item If $\del_h$ hits $a_{ji}(h)$ in $(x+h(a_{j1}(h)k_1+a_{j2}(h)k_2+\cdots+a_{jj-1}(h)k_{j-1})$, we get
    \begin{align*}
       \sum_{1\leq i<j}(\mathcal P_\theta \cdot \del_h a_{ji}(h) )\cdot h^{\zeta+1}\cdot \del_h^{\beta-\al}\nabla^{\al+1} V_{t+h c_j}[k_i, u_1, u_2, \cdots, u_\al],
    \end{align*}
    where each summand is in $\cD_{m+1}$, with parameters \eqref{e:parameter} as $(m_0',m_1', m_2',m_3', m_4', m_5',m_6',m_7')=(m_0, m_1, m_2,m_3, m_4, m_5,m_6, m_7+1)$. 
\end{enumerate}
We conclude that for $v\in \cD_m$, $\del_h v$ is a linear combination of finite terms in $\cD_{m+1}$. In particular, $\del_h^m k_j$ is a linear combination of at most $C(s,p)^m$ terms in $\cD_m$, with coefficients bounded by $\max\{p,B\}^m$.

 In the following, we show the following bound for vectors in $v\in \cD_m$
    \begin{align}\label{e:Dkbound}
        \|v\|_\infty\lesssim
            {(\sqrt{d}(\|x\|_{2} + \sqrt{d} D))}^{m}{( \tau^{-1})}^{2^{m+1} m! + 2m} {(\|x\|_{\infty} +D)}^{(3^{m+1} m!)}, \quad  m\geq 1. 
    \end{align}
Then it follows 
\begin{align*}
    \|\del_h^m k_j\|_{\infty} \lesssim {(\sqrt{d}(\|x\|_{2} + \sqrt{d} D))}^{m}{( \tau^{-1})}^{2^{m+1} m! + 2m} {(\|x\|_{\infty} +D)}^{(3^{m+1} m!)},
\end{align*}
and the first statement in \eqref{e:kjbound} holds.

 We prove \eqref{e:Dkbound} by induction. From \Cref{lemma: infinity norm for k_i(x)}, when $H A_p 2^s s \tau^{-1} \leq 1$, we have that $\|k_i\|_2\leq 4\tau ^{-1}(\|x\|_{2} + \sqrt{d} D)$. It follows that $\|k_i\|_1\leq 4\sqrt{d}\tau ^{-1}(\|x\|_{2} + \sqrt{d} D)$.
    We assume \eqref{e:Dkbound} holds for $ 1,2,\cdots, m-1$, and next we prove it for $m$. Notice that when we write $\del_h^{\beta-\al}\nabla^{\al}  V_{t+h c_j}$, this tensor is evaluated at those $F_{j,h}(x) \coloneq x+h(a_{j1}(h)k_1+a_{j2}(h)k_2+\cdots+a_{jj-1}(h)k_{j-1})$
    according to \eqref{e:defcDm}. According to \Cref{lemma: time-space derivatives log qt}, we see that 
    \begin{align*}
        \|  \del_h^{\beta-\al}\nabla^{\al}  V_{t+h c_j} (F_{j,h}(x))\|_\infty \leq C(p) d^{\beta-\alpha} \tau^{\alpha - 2\beta}{(\|F_{j,h}(x)\|_{\infty}+D)} ^{3\beta - 2\alpha+1  }.
    \end{align*}
We can further bound $\|F_{j,h}(x)\|_{\infty}$ by $4\|x\|_{\infty}+4D$ according to~\Cref{lemma: infinity norm for k_i(x)}. We also notice that $\beta = m_2 + m_3 + m_4+m_7$ and $\alpha = m_3+m_4+m_7$. So, we see that 
    \begin{align}\begin{split}\label{e:ubb}
        &\phantom{{}={}}\| \mathcal P_\theta \cdot h^{\zeta}\cdot \del_h^{\beta-\al}\nabla^{\al} V_{t+h c_j}[u_1, u_2, \cdots, u_\al]\|_\infty
        \lesssim \tau^{-\theta} h^{\zeta} \|  \del_h^{\beta-\al}\nabla^{\al} V_{t+h c_j}\|_\infty\|u_1\|_1\|u_2\|_1\cdots\|u_\al\|_1\\
        &\lesssim  h^{-m_1+m_4+m_7}d^{m_2} {(\tau^{-1})}^{m_0+ 2m_2 + m_3 + m_4+2m_7}{(\|x\|_{\infty}+D)} ^{3m_2+m_3+m_4+m_7+1  }  
        \\  & \quad \times \prod_{\gamma:\ell_\gamma=0}(\sqrt{d}\tau ^{-1}(\|x\|_{2} + \sqrt{d} D))\prod_{\gamma:\ell_\gamma\geq 1} d (\sqrt{d}(\|x\|_{2} + \sqrt{d} D))^{\ell_\gamma} {(\tau^{-1})}^{2^{\ell_\gamma+1} \ell_\gamma! + 2\ell_\gamma} {(\|x\|_{\infty} +D)}^{3^{\ell_\gamma+1} \ell_\gamma!}
        \\  &=  h^{-m_1+m_4+m_7}d^{m_2} {(\tau^{-1})}^{m_0+ 2m_2 + m_3 + m_4+2m_7}{(\|x\|_{\infty}+D)} ^{3m_2+m_3+m_4+m_7+1  } {(\sqrt{d}\tau ^{-1}(\|x\|_{2} + \sqrt{d} D))}^{m_3+m_4}
        \\  & \quad \times \prod_{\gamma:\ell_\gamma\geq 1} d (\sqrt{d}(\|x\|_{2} + \sqrt{d} D))^{\ell_\gamma-1} {(\tau^{-1})}^{2^{\ell_\gamma+1} \ell_\gamma! + 2(\ell_\gamma-1)} {(\|x\|_{\infty} +D)}^{3^{\ell_\gamma+1} \ell_\gamma!}
        \\  &\lesssim  h^{-m_1+m_4+m_7}d^{m_2} {(\tau^{-1})}^{m_0+ 2m_2 + m_3 + m_4+2m_7}{(\|x\|_{\infty}+D)} ^{3m_2+m_3+m_4+m_7+1  } {(\sqrt{d}\tau ^{-1}(\|x\|_{2} + \sqrt{d} D))}^{m_3+m_4}
        \\  & \quad \times d^{m_4+m_5} (\sqrt{d}(\|x\|_{2} + \sqrt{d} D))^{m_6} {(\tau^{-1})}^{m(2^{m} (m-1)! + 2(m-2))} {(\|x\|_{\infty} +D)}^{m(3^{m} (m-1)!)}
        \\  &\lesssim h^{m_7}{(hd)}^{m_4-m_1} d^{m_1+m_2+m_5} {(\sqrt{d}(\|x\|_{2} + \sqrt{d} D))}^{m_3+m_4+m_6}{( \tau^{-1})}^{2^{m} (m)! + 2m^2-2m} {(\|x\|_{\infty} +D)}^{(3^{m} (m)! + 3m+1)}
        \\ & 
        \lesssim {(hd)}^{m_7+m_4-m_1} d^{m_1+m_2+m_5} {(\sqrt{d}(\|x\|_{2} + \sqrt{d} D))}^{m_3+m_4+m_6}{( \tau^{-1})}^{2^{m} (m)! + 2m^2-2m} {(\|x\|_{\infty} +D)}^{(3^{m} (m)! + 3m+1)}
        \\ & \lesssim {(\sqrt{d}(\|x\|_{2} + \sqrt{d} D))}^{m}{( \tau^{-1})}^{2^{m+1} m! + 2m} {(\|x\|_{\infty} +D)}^{(3^{m+1} m!)},
    \end{split}\end{align}
    where in the first inequality we used \Cref{a:bounds on RK}, in the last line we used \eqref{e:parameter} and $m= m_0+m_1 + m_2 +m_3+m_4+m_5+m_6+m_7$, and by our assumption, $hd \leq 1$, $m_7\geq 0$ and $m_7+m_4- m_1 = \zeta \geq 0$. This finishes the proof of \eqref{e:Dkbound}.

  In the following we prove the second statement in \eqref{e:kjbound}. We define the following set of $d\times d$ matrices $\widetilde\cD_0, \widetilde\cD_1,\widetilde\cD_2, \cdots$. Let $\widetilde\cD_0=\{\nabla k_1, \nabla  k_2, \cdots, \nabla 
 k_j\}$. Then by \Cref{lemma: infinity norm pth order growth in RK}, for any $v\in \widetilde\cD_0$,
   \begin{align*}
       \|v\|_\infty\leq 2 \overline K ^{(1)} \leq 8 D^2 \tau^{-2}.
   \end{align*}
We denote $\widetilde \cD_m$ the set of matrices obtained from taking gradient of \eqref{e:defcDm}, with respect to $x$ 
   \begin{align}\begin{split}\label{e:defcDm2}
        &\mathcal P_\theta \cdot h^{\zeta}\cdot \del_h^{\beta-\al}\nabla^{\al+1}  V_{t+h c_j}[(\mathbb{I}_d+h(a_{j1}\nabla k_1+a_{j2}\nabla k_2+\cdots+a_{jj-1}\nabla k_{j-1})), u_1, u_2, \cdots, u_\al]\\
        &\mathcal P_\theta \cdot h^{\zeta}\cdot \del_h^{\beta-\al}\nabla^{\al}  V_{t+h c_j}[u_1, u_2, \cdots, u_{\gamma-1},\nabla u_{\gamma}, u_{\gamma+1},\cdots, u_\al], \quad 1\leq \gamma\leq \al.
    \end{split}\end{align}
  Then  $\del_h^m \nabla k_j$ is a finite linear combination of terms in $\widetilde \cD_m$. 

  Next we show by induction that $v\in \widetilde \cD_m$, 
   \begin{align}\label{e:wDkbound}
        \|v\|_\infty\lesssim
            {(\sqrt{d}(\|x\|_{2} + \sqrt{d} D))}^{m}{( \tau^{-1})}^{2^{m+1} m! + 2m} {(\|x\|_{\infty} +D)}^{(3^{m+1} m!)}, \quad  m\geq 1. 
    \end{align}
    Then it follows 
\begin{align*}
    \|\del_h^m \nabla k_j\|_{\infty}\lesssim {(\sqrt{d}(\|x\|_{2} + \sqrt{d} D))}^{m}{( \tau^{-1})}^{2^{m+1} m! + 2m} {(\|x\|_{\infty} +D)}^{(3^{m+1} m!)},
\end{align*}
and the second statement in \eqref{e:kjbound} holds.

  Similar to \eqref{e:ubb}, 
  we can bound the $L^\infty$-norm of $v$ as
  \begin{align*}
       & \phantom{{}={}}|\mathcal P_\theta| h^\zeta \|  \del_h^{\beta-\al}\nabla^{\al + 1} V_{t+h c_j}\|_\infty (1+8D^2 \tau^{-2}sBr) \|u_1\|_1\|u_2\|_1\cdots\|u_\al\|_1
       \\   &\lesssim 2\tau^{-\theta} h^\zeta \|  \del_h^{\beta-\al}\nabla^{\al+1} V_{t+h c_j}\|_\infty\|u_1\|_1\|u_2\|_1\cdots\|u_\al\|_1\\
       &\lesssim {(\sqrt{d}(\|x\|_{2} + \sqrt{d} D))}^{m}{( \tau^{-1})}^{2^{m+1} m! + 2m} {(\|x\|_{\infty} +D)}^{(3^{m+1} m!)}.
    \end{align*}
    Also, for any $q \in \{1, 2,\dots,d\}$, we can similarly obtain that
    \begin{align*}
       &\phantom{{}={}} |\mathcal P_\theta|  h^\zeta \|  \del_h^{\beta-\al}\nabla^{\al} V_{t+h c_j}\|_\infty \|\partial_q u_\gamma\|_1 \prod_{\gamma'\neq \gamma}\|u_{\gamma'}\|_1 \\  &\lesssim   \tau^{-\theta }h^\zeta \|  \del_h^{\beta-\al}\nabla^{\al+1} V_{t+h c_j}\|_\infty d
       {(\sqrt{d}(\|x\|_{2} + \sqrt{d} D))}^{\ell_\gamma}{( \tau^{-1})}^{2^{\ell_\gamma+1} \ell_\gamma! + 2\ell_\gamma} {(\|x\|_{\infty} +D)}^{(3^{\ell_\gamma+1} \ell_\gamma!)}\prod_{\gamma'\neq \gamma}\|u_{\gamma'}\|_1\\ &\lesssim  {(\sqrt{d}(\|x\|_{2} + \sqrt{d} D))}^{m}{( \tau^{-1})}^{2^{m+1} m! + 2m} {(\|x\|_{\infty} +D)}^{(3^{m+1} m!)},
  \end{align*}
as we did in \eqref{e:ubb}.
This gives \eqref{e:wDkbound}.

Finally, for \eqref{e:dhVbound}, we notice that 
    \begin{align*}
        V_{t+h}(F_h)= V_{t+h}(x+h\sum_{j=1}^s b_j(h) k_j),
    \end{align*}
    which is of the same form as $k_j$. Thus the same argument as for \eqref{e:kjbound} gives \eqref{e:dhVbound}
\end{proof}


\section{Derivatives and Moments Estimates on Forward and Reverse Densities}
\label{s:prel-estim-q}

Under \Cref{assumption:secon-moment}, we take $\lambda_t = e^{-t}$ and $\sigma_t = \sqrt{1- \lambda_t ^2}$ ($t>0$), and we have that
    \begin{align}\label{e: definition of q_t}
        q_t (y) \coloneq \int_{\mathbb{R}^d} \frac{1}{{(\sqrt{2\pi} \sigma_t)}^{d}} \cdot e^{-\frac{\|y-\lambda_t x\|^2}{2{\sigma_t}^2}} \cdot \de \muast(x).
    \end{align}

We recall the following estimates on derivatives of $\log q_t (y)$ proved in Appendix B of~\cite{huang2024convergence}.
\begin{lemma}\label{Lemma: Hessian estimates}
   For any $p \geq 3$, any $y\in\mathbb{R}^d$ and any $t >0$, 
        \begin{align}\label{e: all derivatives log qt}
            \|\nabla^p \log q_t (y)\|_{\infty} \leq (4 p!) \frac{ {\lambda_t} ^p D^p}{ {\sigma_t} ^{2p}} .
        \end{align}
    Also, for any $i,j$,
        \begin{align*}
            |\partial_{y_i} \log q_t (y)| \leq \frac{|y_i| +\lambda_t D}{\sigma_t ^2}, \text{ and } |\partial^2 _{y_i y_j} \log q_t (y)| \leq \frac{\delta_{ij}}{\sigma_t ^2} +  2\frac{\lambda_t ^2 D^2}{\sigma_t ^4}.
        \end{align*}
\end{lemma}   

Next, we estimate the time direction derivatives of $\log q_t(y)$.
\begin{lemma}\label{lemma: time-space derivatives log qt}
 For any $p \geq m \geq 1$, and any multi-index $\gamma$ with $m+ |\gamma| = p$, there exists a computable positive constant $C(m,p)$, such that
    \begin{align*}
        |\partial_t ^m \partial_y^{\gamma} \log q_t(y)| \leq C(m,p) \cdot d^m \cdot {(\|y\|_{\infty}+1)} ^m \cdot {\bigg(\frac{\|y\|_{\infty }+\lambda_t D}{\sigma_t ^2}\bigg)}^{p+m}.
    \end{align*}
\end{lemma}
\begin{proof}
    Let $h_t (y)\coloneq \log q_t(y) + \frac{\|y\|^2}{2}$. A direct computation shows that 
        \begin{align}\label{e: time derivatives of q_t}
            \partial_t h_t (y)= \sum_{i=1} ^d \partial^2 _{ii} h_t(y) + \sum_{i=1} ^d \partial_i h_t (y)\partial_i h_t (y)- \sum_{i=1} ^d \partial_i h_t (y)y_i.
        \end{align}
    Now, we assume that for an $m \in \mathbb{Z}_+$, $\partial_t ^m h_t(u)$ has the following form:
    \begin{align}\label{e: expansion of m time derivative}
        \partial_t ^m h_t(y) = \sum_{\alpha,\beta} y^{\alpha} \partial_y ^{\beta_1} h_t(y) \partial_y ^{\beta_2} h_t(y) \ldots \partial_y ^{\beta_{2m}} h_t(y),
    \end{align}
    where $\alpha$ and $\beta_i$'s are multi-indices and they satisfy that $|\alpha| + |\beta_1| + \ldots +|\beta_{2m}| \leq 2m$, $|\alpha| \leq m$, and there are at most $a_m$ terms in the summation. When some $\beta_i = 0$, we mean that $\partial_y ^{0} h_t(y) = 1$. Then, by induction,
        \begin{align}\label{e: expansion of m+1 time derivative}
            \partial_t ^{m+1} h_t(y) = \sum_{\alpha,\beta} \sum_{j=1} ^{2m} y^{\alpha} \partial_y ^{\beta_1} h_t(y) \ldots \partial_y ^{\beta_j} \partial_t h_t(y) \ldots \partial_y ^{\beta_{2m}} h_t(y). 
        \end{align}
    By \eqref{e: time derivatives of q_t}, we know that
        \begin{align*}
            \partial_y ^{\beta_j} \partial_t h_t(y) = \sum_{i=1} ^d \partial_y ^{\beta_j} \partial^2 _{ii} h_t(y) + \sum_{i=1} ^d \sum_{\gamma_{j} + \gamma' _{j} } (\partial_{y} ^{\gamma_j}\partial_i h_t (y) )(\partial_{y} ^{\gamma' _j}\partial_i h_t (y))- \sum_{i=1} ^d \sum_{\gamma_{j} + \gamma' _{j} } (\partial_{y} ^{\gamma_j}\partial_i h_t (y) )(\partial_{y} ^{\gamma' _j}y_i),
        \end{align*}
    where $\beta_j = \gamma_j + \gamma' _j$, and there are at most $d(1+ 2^{|\beta_j|} + 2^{|\beta_j|}) \leq 3d \cdot 4^m $ terms. Hence, there are at most $a_m  \cdot 6md \cdot 4^m$ terms on the right hand side of \eqref{e: expansion of m+1 time derivative}. Hence, $a_1 = 3d$ and $a_{m+1} = a_m  \cdot 6md \cdot 4^m$, we see that $a_m \leq (6d)^m \cdot m! \cdot 2^{m(m-1)}$ for any $m \in \mathbb{Z}_+$.

    Now, for a derivative $\partial_t ^m \partial_y^{\gamma} h_t(y)$ with $|\gamma| = p-m$ and $m \geq 1$, according to \eqref{e: expansion of m time derivative},
    \begin{align*}
        \partial_t ^m \partial_y ^{\gamma} h_t(y) = \sum_{\alpha',\beta'} y^{\alpha'} \partial_y ^{\beta_1 '} h_t(y) \partial_y ^{\beta_2 ' } h_t(y) \ldots \partial_y ^{\beta_{2m} '} h_t(y),
    \end{align*}
    with $|\alpha '| + |\beta_1 '| + \ldots +|\beta_{2m} '| \leq 2m + |\gamma| = p+m$, $|\alpha '| \leq m$, and there are at most $a_m (2m+1)^{|\gamma|} \leq (6d)^m \cdot m! \cdot 2^{m(m-1)} \cdot (2m+1)^{p-m}$ on the right hand side. For each term of it, according to Lemma~\ref{Lemma: Hessian estimates}, we know that each term on the right hand side has an upper bound
        \begin{align*}
            C(m,p)  \cdot {(\|y\|_{\infty}+1)} ^m \cdot {\bigg(\frac{\|y\|_{\infty }+\lambda_t D}{\sigma_t ^2}\bigg)}^{p+m},
        \end{align*}
    which can also be the upper bound of $|\partial_t ^m \partial_y ^{\gamma} h_t(y)|$ after replacing $C(m,p)$ with another much larger positive constant $C(m,p)$ multiplying $d^m$. 
    
\end{proof}

The following is the Lemma 2 in Appendix B of \cite{huang2024convergence}, which describes the situation when $t >0$ is very large. We remark that one can also use Lemma 9 in \cite{chen2023improved} to obtain a similar result as \Cref{lemma: qt close to Gaussian}.
\begin{lemma}\label{lemma: qt close to Gaussian}
    Let $q_t(y)$ be defined in \eqref{e: definition of q_t}, and let
        \begin{align*}
            g(y) \coloneq \frac{1}{{(\sqrt{2\pi} )}^{d}} \cdot e^{-\frac{\|y\|^2}{2}}.
        \end{align*}
    Then, we have that for $t >0$,
        \begin{align}\label{e:qt-q_bound}
            \int_{\mathbb{R}^d} | q_t(y)  - g(y)| \de y \leq \frac{2d(1-\sigma_t)}{{\sigma_t}^{d+3}} + \lambda_t \sqrt{d}D \frac{4+ \lambda_t \sqrt{d}D}{\sigma_t ^2} \cdot e^{\frac{ \lambda_t ^2 d D^2}{2{\sigma_t}^2}}.
        \end{align}
 In particular, there exists a universal constant $C_u>0$, such that
\begin{align}\label{e:simplified_bound}
         \int_{\mathbb{R}^d} | q_t(y)  - g(y)| \de y \leq C_u e^{-t}\sqrt d D,
    \end{align}
    which goes to $0$ exponentially as $t\rightarrow +\infty$.
    
\end{lemma}

Finally, we estimate the moments of $\|x\|_2$ and $\|x\|_{\infty}$.
\begin{lemma}\label{lemma: moments of L^2 and infinity norms}
    For any $t >0$, any $m \in \mathbb{Z}_+$, there is a positive constant $C(m)>0$, such that
        \begin{align*}
        \int_{\mathbb{R}^d} q_{t}(y) \|y\|_{2} ^{m} \de y \leq C(m){d}^{\frac{m}{2}} D^m .
        \end{align*}
and
        \begin{align*}
        \int_{\mathbb{R}^d} q_{t}(y) \|y\|_{\infty} ^{m} \de y \leq C(m) {\big(D+\sqrt{ \log d}\big)}^m.
        \end{align*}
\end{lemma}
\begin{proof}
    Notice that $\muast$ has a compact support $K_{\ast}$ and $D =  1+\max_{x \in K_{\ast}} \|x\|_{\infty} $.
        \begin{align*}
        \begin{split}
            \int_{\mathbb{R}^d} q_t(y) \|y\|_{2} ^m \ \de y &= \int_{\mathbb{R}^d} \int_{\mathbb{R}^d} \frac{\|y\|_{2} ^m }{{(\sqrt{2\pi} \sigma_t)}^{d}} \cdot e^{-\frac{\|y-\lambda_t x\|_2 ^2}{2{\sigma_t}^2}} \cdot \muast(x)  \de x \de y
            \\  &\leq 2^{m-1}\int_{\mathbb{R}^d} \int_{\mathbb{R}^d} \frac{\|y-\lambda_t x\|_{2} ^m + \|\lambda_t x\|_{2} ^m}{{(\sqrt{2\pi} \sigma_t)}^{d}} \cdot e^{-\frac{\|y-\lambda_t x\|_2 ^2}{2{\sigma_t}^2}} \cdot  \muast(x)  \de x \de y
            \\  &\leq 2^{m}\int_{\mathbb{R}^d} \int_{\mathbb{R}^d} \frac{\sigma_t ^m \|z\|_{2} ^m + (\sqrt{d}D)^m}{{(\sqrt{2\pi} )}^{d}} \cdot e^{-\frac{\|z\|_2 ^2}{2}} \cdot  \muast(x)  \de x \de z
            \\  &= 2^{m} \bigg( (\sqrt{d}D)^m + \sigma_t ^m \int_{\mathbb{R}^d} \frac{ \|z\|_{2} ^m }{{(\sqrt{2\pi} )}^{d}} \cdot e^{-\frac{\|z\|_2 ^2}{2}}  \de z \bigg).
        \end{split}
    \end{align*}
Also, 
    \begin{align*}
        \int_{\mathbb{R}^d} \frac{ \|z\|_{2} ^m }{{(\sqrt{2\pi} )}^{d}} \cdot e^{-\frac{\|z\|_2 ^2}{2}}  \de z \leq d^{\frac{m}{2}-1} \sum_{i=1} ^d \int_{\mathbb{R}^d} \frac{ {|z_i|} ^m }{{(\sqrt{2\pi} )}^{d}} \cdot e^{-\frac{\|z\|_2 ^2}{2}}  \de z = C(m)d^{\frac{m}{2}},
    \end{align*}
where $C(m) = \mathbb{E}(|X| ^m)$ for $X$ being a $1$-dimensional unit normal distribution.

Similarly,
    \begin{align*}
        \begin{split}
            \int_{\mathbb{R}^d} q_t(y) \|y\|_{\infty} ^m \ \de y \leq 2^{m} \bigg( D^m + \sigma_t ^m \int_{\mathbb{R}^d} \frac{ \|z\|_{\infty} ^m }{{(\sqrt{2\pi} )}^{d}} \cdot e^{-\frac{\|z\|_2 ^2}{2}}  \de z \bigg).
        \end{split}
    \end{align*}
Now we need to estimate the above integral. Assume that $X_1, X_2,\cdots,X_d$ are i.i.d.~unit normal distributions and we define $M \coloneq \max \{|X_1|, |X_2|,\cdots,|X_d|\}$. We see that for any $r>0$,
    \begin{align*}
        \mathbb{P}(M>r) \leq d \mathbb{P}(|X_1|>r) < 2d \cdot f(r) r^{-1},
    \end{align*}
where $f(r)$ is the density function of a unit normal distribution. Hence,
    \begin{align*}
        \begin{split}
            \mathbb{E}(M^m) &= \int_{0} ^{+\infty} m r^{m-1} \mathbb{P}(M>r) \de r \leq \int_{0} ^{\sqrt{2 \log d}} m r^{m-1}  \de r + 2md\int_{\sqrt{2 \log d}} ^{+\infty}  r^{m-2}f(r) \de r
            \\  &= (2 \log d)^{\frac{m}{2}} + \frac{2^{\frac{m}{2}}md}{\sqrt{\pi}}  \int_{\sqrt{\log d}} ^{+\infty}  r^{m-2}e^{-r^2} \de r
            \\  &\leq {(2 \log d)}^{\frac{m}{2}} + \frac{2^{\frac{m}{2}}m}{\sqrt{\pi}} {(\log{d})}^{\frac{m-3}{2}} \leq (m+1){(2 \log d)}^{\frac{m}{2}},
        \end{split}
    \end{align*}
where we used~\Cref{lemma: exponential error term in cut-off} in the second inequality.
Hence,
    \begin{align*}
        \int_{\mathbb{R}^d} q_t(y) \|y\|_{\infty} ^m \ \de y \leq 2^{m} ( D^m + (m+1){(2 \log d)}^{\frac{m}{2}}) \leq (m+1) 2^{m} {\big(D+\sqrt{2 \log d}\big)}^m .
    \end{align*}

\end{proof}

\section{Total Variation Estimates along Probability Flow}
\label{section: L^1 error of transport equation}

In this section, we collect some general theorems for continuity equations proved in Appendix A of~\cite{huang2024convergence}.
\begin{theorem}\label{theorem: L^1 error}
    Fix any $0 < \mft < T$. Let $\widehat q_t(x), q_t(x) \in C^1([\mft, T] \times \mathbb{R}^d) \cap L^1 ([\mft, T] \times \mathbb{R}^d) $ solve the following two continuity equations on $\mathbb{R}^d$ respectively,
    \begin{align*}
    \del_t q_t =\nabla \cdot (Z_t
    q_t),\quad
    \del_t \widehat q_t =\nabla \cdot (\widehat Z_t
    \widehat q_t).
    \end{align*}
    We also assume that for $t \in [\mft, T]$, $Z_t, \widehat Z_t$ are locally Lipschitz on $\mathbb{R}^d$. Then, if we denote
    $\delta_t(x) \coloneq  \widehat Z_t(x) - Z_t(x)$, $\widehat\varepsilon_t(x) \coloneq \widehat q_t(x) - q_t(x)$, we have that, 
    \begin{align*}
        \bigg| \int_{\mathbb{R}^d}  | \widehat\varepsilon_{\mft}(x)  | \ \de x  - \int_{\mathbb{R}^d}  | \widehat\varepsilon_{T}(x)  | \ \de x \bigg | \leq \int_{\mft} ^T E(t) \ \de t, \quad \text{with } E(t) \coloneq \int_{\mathbb{R}^d}  \bigg|  (\nabla \cdot (q_{t}\delta_{t} ))(x) \bigg| \ \de x.
    \end{align*}
\end{theorem}

\begin{theorem}\label{theorem: general L^1 error}
    Fix any $0 < \mft < T$. Let $p_t(x) \in C^1([\mft, T] \times \mathbb{R}^d) \cap L^1 ([\mft, T] \times \mathbb{R}^d) $ solve the following continuity equation on $\mathbb{R}^d$ with $h_t(x) \in L^1 ([\mft, T] \times \mathbb{R}^d)$,
    \begin{align*}
    \del_t p_t(x) =(\nabla \cdot (Z_t
    p_t))(x) + h_t(x).
    \end{align*}
    We also assume that for $t \in [\mft, T]$, $Z_t$ is locally Lipschitz on $\mathbb{R}^d$. Then, for almost all $t \in [\mft,T]$, we have that, 
    \begin{align*}
        \bigg| \frac{\de}{\de t} \int_{\mathbb{R}^d}  | p_{t}(x)  | \ \de x  \bigg | \leq \int_{\mathbb{R}^d}  |h_t(x)| \ \de x.
    \end{align*}
Hence,
    \begin{align*}
        \bigg| \int_{\mathbb{R}^d}  | p_{\mft}(x)  | \de x - \int_{\mathbb{R}^d}  | p_{T}(x)  | \de x \bigg | \leq \int_{\mft} ^T \int_{\mathbb{R}^d}  |h_t(x)| \ \de x.
    \end{align*}
    
\end{theorem}


\section{Gagliardo-Nirenberg Interpolation Inequality}\label{section:score estimation error}

Finally, we recall the following Gagliardo-Nirenberg interpolation inequality proved in Appendix C of~\cite{huang2024convergence}.
\begin{lemma}[Gagliardo-Nirenberg]\label{lemma: Gagliardo-Nirenberg}
    There is a positive universal constant $C_u$, such that for any $d \in \mathbb{Z}_{+}$, any $w \in L^1(\bR^d)$, any $i \in \llbracket 1,d\rrbracket$, if $\partial^2 _{ii} w \in L^1(\bR^d)$, then
        \begin{align}\label{e: Gagliardo k=2}
        {\bigg(\int_{\mathbb{R}^d}|\partial_i w(x)| \ \de x\bigg)}^{2} \leq C_u \bigg(\int_{\mathbb{R}^d}|\partial^2 _{ii} w(x)| \ \de x\bigg)\bigg(\int_{\mathbb{R}^d}|w(x)| \ \de x\bigg).
        \end{align}
In general, if $\partial^k _{i} w \in L^1 (\bR^d)$ with $k \geq 2$, then
        \begin{align}\label{e: Gagliardo general k}
        \int_{\mathbb{R}^d}|\partial_i w(x)| \ \de x \leq C_u ^{\frac{k-1}{2}}{\bigg(\int_{\mathbb{R}^d}|\partial^k _{i} w(x)| \ \de x\bigg)}^{\frac{1}{k}} {\bigg(\int_{\mathbb{R}^d}|w(x)| \ \de x\bigg)}^{\frac{k-1}{k}}.
        \end{align}
\end{lemma}

\section{Enhanced Exponential Runge-Kutta Schemes}\label{section:rk scheme}

In this section, we present the derivation of the enhanced exponential Runge–Kutta schemes introduced in \cref{s:general} for solving the probability flow ODE with a general variance schedule $\beta_t$:
\begin{align} 
\label{e:ODE_flow0}
	\del_t \sY_t = \frac{\beta_{T-t}}{2}\left( \sY_t +  s_t(\sY_t)\right) = \frac{\beta_{T-t}}{2} \sV_t(\sY_t), \qquad \sY_0 \sim \svarrho_0, \qquad 0 \leq t \leq T.
\end{align}
The schemes are constructed to ensure that all primary and intermediate time steps align with the discrete time sequence used during score matching: $\bT=\{j\Delta t : 0 \leq j \leq N,\,  \Delta t = \frac{T}{N}\}$. This alignment is necessary to achieve high-order convergence to the target distribution in \Cref{theorem: main L^1 theorem} and \Cref{theorem: main L^1 theorem general}. \Cref{s:exp-rk-coefficient-derivation} derives the first-, second-, and third-order enhanced exponential Runge-Kutta updates from the integral formulation. \Cref{s:coefficient_est} proves the coefficient bounds used in \Cref{a:bounds on RK}.

\subsection{Coefficient Derivation}
\label{s:exp-rk-coefficient-derivation}
The derivation follows the approaches in \cite{lu2022dpm,hochbruck2005explicit}.
Starting from the integral equation \eqref{eq:exp-rk-intg}, 
\begin{align*} 
     \sY_{t_{i+1}} = e^{\zeta(t_{i+1}) - \zeta(t_i)}\sY_{t_{i}} - \sigma_{T - t_{i+1}}e^{\alpha_{t_{i+1}}}\int_{t_i}^{t_{i+1}} \frac{\rd}{\rd t} e^{-\alpha_{t}} \rscore_t(\sY_t) \rd t.
\end{align*}
Recall the time change variable $\alpha_t$ defined in \eqref{eq:t-alpha change of variable}, which is monotonically decreasing in $t$.
We treat $\rscore_t(\sY_t)$ as a function of $\alpha_t$ and  apply a Taylor expansion of 
$\rscore_t(\sY_t)$ with respect to the changed time variable $\alpha_t$, yielding the following approximation:
\begin{align} 
     \sY_{t_{i+1}} &= e^{\zeta(t_{i+1}) - \zeta(t_i)}\sY_{t_{i}} -  \sigma_{T - t_{i+1}}e^{\alpha_{t_{i+1}}}\sum_{k=0}^{q-1}\int_{t_i}^{t_{i+1}} \frac{\bigl(\alpha_{t} - \alpha_{t_i}\bigr)^k}{k!}\frac{\rd}{\rd t} e^{-\alpha_{t}} \rd t \frac{\de^{(k)}\rscore_{t_i}(\sY_{t_i})}{\de{\alpha}^{(k)}}  + \mathcal{O}(H^{q+1}) \nonumber\\
     &= e^{\zeta(t_{i+1}) - \zeta(t_i)}\sY_{t_{i}} + \sigma_{T - t_{i+1}}\sum_{k=1}^{q}\bigl(\alpha_{t_{i+1}} - \alpha_{t_i}\bigr)^{k}\varphi_k(\alpha_{t_{i+1}} - \alpha_{t_i}) \frac{\de^{(k-1)}\rscore_{t_i}(\sY_{t_i})}{\de \alpha^{(k-1)}}  + \mathcal{O}(H^{q+1}) \label{eq:app-integral}.
\end{align}
where $\varphi_k(h) = \int_0^1 \frac{x^{k-1}}{(k-1)!}e^{h(1-x)}dx$. In particular,
\begin{equation}
\label{eq:varphi}
\varphi_1(h) = \frac{e^{h}-1}{h}\quad \textrm{and} \quad \varphi_2(h) = \frac{e^{h}-h-1}{h^2}.
\end{equation}

The first order scheme ($q=1$) with $c_1 = 0$ is given by 
\begin{equation}\label{eq:exp-RK-stiff1}\begin{split}
       k_1(x) &= \rscore_{t_i}(x), \\
     \dY_{t_{i+1}} &= e^{\zeta(t_{i+1}) - \zeta(t_i)}\dY_{t_{i}} 
     + \sigma_{T - t_{i+1}}\bigl(\alpha_{t_{i+1}} - \alpha_{t_i}\bigr)\varphi_1\bigl(\alpha_{t_{i+1}} - \alpha_{t_i}\bigr) k_1(\dY_{t_i}). 
\end{split}
\end{equation} 
Assume that $\dY_{t_i} = \sY_{t_i}$ is exact at time $t_i$, and subtract the above equation from \eqref{eq:app-integral}, the local truncation error 
is given by $\dY_{t_{i+1}} - \sY_{t_{i+1}} = \mathcal{O}(H_i^2)$.

The second order scheme ($q=2$), with $c_1 = 0$ and any $0 < c_2 \leq 1$, is given by 
\begin{equation}
\label{eq:exp-RK-stiff2}
    \begin{split}
k_1(x) =& \rscore_{t_i}(x), \\
k_2(x) =& \rscore_{t_i + c_2H_i}\bigl(e^{\zeta(t_{i}+c_2H_i) - \zeta(t_i)}x + \sigma_{T - t_{i} - c_2H_i}\bigl(\alpha_{t_{i}+ c_2H_{i}} - \alpha_{t_i}\bigr)\varphi_1(\alpha_{t_{i} + c_2H_{i}} - \alpha_{t_i}) k_1(x)\bigr), \\
\dY_{t_{i+1}} =& e^{\zeta(t_{i+1}) - \zeta(t_i)}\dY_{t_{i}} 
     + b_1 k_1(\dY_{t_i}) H_i + b_2 k_2(\dY_{t_i}) H_i, 
\\
     b_2 = & \sigma_{T - t_{i+1}}\frac{(\alpha_{t_{i+1}} - \alpha_{t_i})^2}{(\alpha_{t_{i}+c_2H_i} - \alpha_{t_i})H_i} \varphi_2(\alpha_{t_{i+1}} - \alpha_{t_i})
     \quad 
     b_1 =  \sigma_{T - t_{i+1}}\frac{\alpha_{t_{i+1}} - \alpha_{t_i}}{H_i}\varphi_1(\alpha_{t_{i+1}} - \alpha_{t_i}) - b_2.
\end{split}
\end{equation}
To analyze the local truncation error, we assume that  $\dY_{t_i} = \sY_{t_i}$ is exact at time $t_i$. Then, the slope evaluations satisfy the following approximations:
\begin{align*} 
k_1(\dY_{t_i}) =& \rscore_{t_i}(\dY_{t_i}), 
\\
k_2(\dY_{t_i}) =& \rscore_{t_i + c_2H_i}\bigl(\sY_{t_i + c_2H_i}\bigr) + \mathcal{O}(H_i^2)
= \rscore_{t_i}(\sY_{t_i}) + \bigl(\alpha_{t_i+c_2H_i} - \alpha_{t_i}\bigr) \frac{\de\rscore_{t_i}(\sY_{t_i})}{\de\alpha} + \mathcal{O}(H_i^2).
\end{align*}
The resulting update satisfies
$$\dY_{t_{i+1}} = e^{\zeta(t_{i+1}) - \zeta(t_i)}\dY_{t_{i}} 
     + b_1 k_1(\dY_{t_i})H_i 
     + b_2 k_2(\dY_{t_i})H_i 
     = \sY_{t_{i+1}} + \mathcal{O}(H_i^3).$$ 
This third order local accuracy is achieved by matching coefficients of $\rscore_{t_i}(\sY_{t_i})$ and $\frac{\de\rscore_{t_i}(\sY_{t_i})}{\de\alpha}$ in the update rules \eqref{eq:exp-RK-stiff2} and \eqref{eq:app-integral} to order $3$, using  the identities:
\begin{align*}
    &\Bigl(b_1  + b_2\Bigr)H_i = \sigma_{T - t_{i+1}}\bigl(\alpha_{t_{i+1}} - \alpha_{t_i}\bigr)\varphi_1(\alpha_{t_{i+1}} - \alpha_{t_i}), 
    \\
    &b_2 \bigl(\alpha_{t_i+c_2H_i} - \alpha_{t_i}\bigr) H_i = \sigma_{T - t_{i+1}}\bigl(\alpha_{t_{i+1}} - \alpha_{t_i}\bigr)^2\varphi_2(\alpha_{t_{i+1}} - \alpha_{t_i}). 
\end{align*}

The third order scheme ($q=3$), with $c_1 = 0$, any $0 < c_2  < c_3 \leq 1$, and $\gamma = \frac{3c_3^2 - 2c_3}{3c_2^2-2c_2}$, is given by 
{\small
\begin{equation}
\label{eq:exp-RK-stiff3}
    \begin{split}
        k_1(x) =& \rscore_{t_i}(x), 
\\
k_2(x) =& \rscore_{t_i + c_2H_i}\Bigl(e^{\zeta(t_{i}+c_2H_i) - \zeta(t_i)}x + \sigma_{T - t_{i} - c_2H_i}\bigl(\alpha_{t_{i}+ c_2H_{i}} - \alpha_{t_i}\bigr)\varphi_1\bigl(\alpha_{t_{i} + c_2H_{i}} - \alpha_{t_i}\bigr) k_1(x)\Bigr), 
\\
k_3(x) =& \rscore_{t_i + c_3H_i}\Bigl(e^{\zeta(t_{i}+c_3H_i) - \zeta(t_i)}x 
+ \sigma_{T - t_{i} - c_3H_i}\bigl(\alpha_{t_{i}+ c_3H_{i}} - \alpha_{t_i}\bigr)\varphi_1\bigl(\alpha_{t_{i} + c_3H_{i}} - \alpha_{t_i}\bigr) k_1(x) + a_{32} \bigl(k_2(x) - k_1(x)\bigr)H_i
\Bigr), 
\\
a_{32} =& \frac{\gamma\sigma_{T - t_{i} - c_2H_i}(\alpha_{t_{i}+ c_2H_{i}} - \alpha_{t_i}\bigr)^2\varphi_2\bigl(\alpha_{t_{i} + c_2H_{i}} - \alpha_{t_i}\bigr) + \sigma_{T - t_{i} - c_3H_i}\bigl(\alpha_{t_{i}+ c_3H_{i}} - \alpha_{t_i}\bigr)^2\varphi_2\bigl(\alpha_{t_{i} + c_3H_{i}} - \alpha_{t_i}\bigr)}{(\alpha_{t_{i}+ c_2H_{i}} - \alpha_{t_{i}})H_i},
\\
\dY_{t_{i+1}} =& e^{\zeta(t_{i+1}) - \zeta(t_i)}\dY_{t_{i}} 
     + b_1 k_1(\dY_{t_i})H_i 
     + \gamma b_3 k_2(\dY_{t_i})H_i 
     + b_3 k_3(\dY_{t_i})H_i, 
\\
b_{3}  =&  \frac{\sigma_{T - t_{i+1}}
     \bigl(\alpha_{t_{i+1}} - \alpha_{t_i}\bigr)^2\varphi_2(\alpha_{t_{i+1}} - \alpha_{t_i})
      }{\bigl(\gamma(\alpha_{t_{i}+ c_2H_{i}} - \alpha_{t_{i}}) + (\alpha_{t_{i}+ c_3H_{i}} - \alpha_{t_{i}})\bigr)H_i}\quad
b_1 = \frac{\sigma_{T - t_{i+1}}\bigl(\alpha_{t_{i+1}} - \alpha_{t_i}\bigr)\varphi_1\bigl(\alpha_{t_{i+1}} - \alpha_{t_i}\bigr)}{H_i}  - (1+\gamma)b_3.
    \end{split}
\end{equation}}The local truncation error analysis assumes that $\dY_{t_i} = \sY_{t_i}$ is exact at time $t_i$. Denote $J_{t_i} = \frac{\partial \rscore_{t_i}(x)}{\partial x}\Bigl|_{x=\sY_{t_i}}$,  then the slope evaluations satisfy the following approximations:
{\small
\begin{align*} 
k_1(\dY_{t_i}) =& \rscore_{t_i}(\dY_{t_i}) ,
\\
k_2(\dY_{t_i}) =& \rscore_{t_i + c_2H_i}(\sY_{t_i + c_2H_i}) 
- \Bigl(\rscore_{t_i + c_2H_i}(\sY_{t_i + c_2H_i}) -  k_2(\dY_{t_i})\Bigr)  \quad  \textrm{(Taylor expansions at $\rscore_{t_i(\alpha_{t_i})}(\sY_{t_i(\alpha_{t_i})})$ and $\rscore_{t_i}(\sY_{t_i})$)}
\\
=& \rscore_{t_i}(\sY_{t_i}) + \bigl(\alpha_{t_i+c_2H_i} - \alpha_{t_i}\bigr) \frac{\de\rscore_{t_i}(\sY_{t_i})}{\de\alpha} + \frac{1}{2}\bigl(\alpha_{t_i+c_2H_i} - \alpha_{t_i}\bigr)^2 \frac{\de^{(2)}\rscore_{t_i}(\sY_{t_i})}{\de\alpha^{(2)}}\\
&- J_{t_i}\Bigl(\sigma_{T - t_i - c_2H_i}\bigl(\alpha_{t_i+c_2H_i} - \alpha_{t_i}\bigr)^2\varphi_2(\alpha_{t_i+c_2H_i} - \alpha_{t_i})\Bigr)\frac{\de\rscore_{t_i}(\sY_{t_i})}{\de\alpha} + \mathcal{O}(H_i^3) , 
\\
k_3(\dY_{t_i}) =& \rscore_{t_i + c_3H_i}(\sY_{t_i + c_3H_i}) 
- \Bigl(\rscore_{t_i + c_3H_i}(\sY_{t_i + c_3H_i}) -  k_3(\dY_{t_i})\Bigr)  \quad  \textrm{(Taylor expansions at $\rscore_{t_i(\alpha_{t_i})}(\sY_{t_i(\alpha_{t_i})})$ and $\rscore_{t_i}(\sY_{t_i})$)}
\\
=& \rscore_{t_i}(\sY_{t_i}) + \bigl(\alpha_{t_i+c_3H_i} - \alpha_{t_i}\bigr) \frac{\de\rscore_{t_i}(\sY_{t_i})}{\de\alpha} + \frac{1}{2}\bigl(\alpha_{t_i+c_3H_i} - \alpha_{t_i}\bigr)^2  \frac{\de^{(2)}\rscore_{t_i}(\sY_{t_i})}{\de\alpha^{(2)}} \\
&- J_{t_i}\Bigl(\sigma_{T - t_{i} - c_3H_i}\bigl(\alpha_{t_{i}+ c_3H_{i}} - \alpha_{t_i}\bigr)^2\varphi_2\bigl(\alpha_{t_{i} + c_3H_{i}} - \alpha_{t_i}\bigr) \frac{\de\rscore_{t_i}(\sY_{t_i})}{\de\alpha} - a_{32} \bigl(k_2(\dY_{t_i}) - k_1(\dY_{t_i})\bigr)H_i\Bigr) + \mathcal{O}(H_i^3) .
\end{align*}}
The resulting update satisfies
$$\dY_{t_{i+1}} = e^{\zeta(t_{i+1}) - \zeta(t_i)}\dY_{t_{i}} 
     + b_1 k_1(\dY_{t_i})H_i 
     + \gamma b_3 k_2(\dY_{t_i})H_i 
     + b_3 k_3(\dY_{t_i})H_i = \sY_{t_{i+1}} + \mathcal{O}(H_i^4).$$ 
This accuracy is achieved by matching the coefficients $\rscore_{t_i}(\sY_{t_i})$, $\frac{\de\rscore_{t_i}(\sY_{t_i})}{\de\alpha}$, $\frac{\de^{(2)}\rscore_{t_i}(\sY_{t_i})}{\de\alpha^{(2)}}$,  and $J_{t_i}$ in the update rules \eqref{eq:exp-RK-stiff3} and \eqref{eq:app-integral} to order 4, using the following approximations:
{\small
\begin{equation}
\label{eq:rk3-exp-order-conditions}
    \begin{split}
        &\Bigl(b_1  + \gamma b_3 + b_3\Bigr)H_i = \sigma_{T - t_{i+1}}\bigl(\alpha_{t_{i+1}} - \alpha_{t_i}\bigr)\varphi_1\bigl(\alpha_{t_{i+1}} - \alpha_{t_i}\bigr) 
    \\
    &\Bigl(\gamma b_3 \bigl(\alpha_{t_i+c_2H_i} - \alpha_{t_i}\bigr) + b_3 \bigl(\alpha_{t_i+c_3H_i} - \alpha_{t_i}\bigr) \Bigr)H_i = \sigma_{T - t_{i+1}}\bigl(\alpha_{t_{i+1}} - \alpha_{t_i}\bigr)^2\varphi_2\bigl(\alpha_{t_{i+1}} - \alpha_{t_i}\bigr) 
    \\
    &\Bigl(\frac{\gamma b_3}{2}\bigl(\alpha_{t_i+c_2H_i} - \alpha_{t_i}\bigr)^2 + \frac{b_3}{2}\bigl(\alpha_{t_i+c_3H_i} - \alpha_{t_i}\bigr)^2 \Bigr)H_i = \sigma_{T - t_{i+1}}\bigl(\alpha_{t_{i+1}} - \alpha_{t_i}\bigr)^2\varphi_3\bigl(\alpha_{t_{i+1}} - \alpha_{t_i}\bigr) + \mathcal{O}(H_i^4) \\
    &\gamma b_3 \Bigl(\sigma_{T - t_i - c_2H_i}\bigl(\alpha_{t_i+c_2H_i} - \alpha_{t_i}\bigr)^2\varphi_2(\alpha_{t_i+c_2H_i} - \alpha_{t_i}) \frac{\de\rscore_{t_i}(\sY_{t_i})}{\de\alpha}\Bigr) H_i 
    \\&+ b_3 \Bigl(\sigma_{T - t_{i} - c_3H_i}\bigl(\alpha_{t_{i}+ c_3H_{i}} - \alpha_{t_i}\bigr)^2\varphi_2\bigl(\alpha_{t_{i} + c_3H_{i}} - \alpha_{t_i}\bigr) \frac{\de\rscore_{t_i}(\sY_{t_i})}{\de\alpha} - a_{32} \bigl(k_2(\dY_{t_i}) - k_1(\dY_{t_i})\bigr)H_i\Bigr)H_i = \mathcal{O}(H_i^4) .
    \end{split}
\end{equation}}The first two equations in \eqref{eq:rk3-exp-order-conditions} follow directly from the definitions of $b_1$ and $b_3$. The third equation holds up to fourth order, due to the definition of $\gamma$ and the approximations: 
\begin{align*}
    \Bigl(\frac{\gamma b_3}{2}&\bigl(\alpha_{t_i+c_2H_i} - \alpha_{t_i}\bigr)^2 +  \frac{b_3}{2}\bigl(\alpha_{t_i+c_3H_i} - \alpha_{t_i}\bigr)^2 \Bigr)H_i 
     \\&= \frac{1}{2}\frac{\gamma \bigl(\alpha_{t_i+c_2H_i} - \alpha_{t_i}\bigr)^2 + \bigl(\alpha_{t_i+c_3H_i} - \alpha_{t_i}\bigr)^2
      }{\gamma(\alpha_{t_{i}+ c_2H_{i}} - \alpha_{t_{i}}) + (\alpha_{t_{i}+ c_3H_{i}} - \alpha_{t_{i}})} \sigma_{T - t_{i+1}}
     \bigl(\alpha_{t_{i+1}} - \alpha_{t_i}\bigr)^2\varphi_2(\alpha_{t_{i+1}} - \alpha_{t_i})
     \\&= \frac{1}{2}\frac{\gamma c_2^2 + c_3^2}{\gamma c_2  + c_3} \sigma_{T - t_{i+1}}
     \bigl(\alpha_{t_{i+1}} - \alpha_{t_i}\bigr)^3\varphi_2(\alpha_{t_{i+1}} - \alpha_{t_i}) + O(H_i^4) 
     \\&= \sigma_{T - t_{i+1}}
     \bigl(\alpha_{t_{i+1}} - \alpha_{t_i}\bigr)^3\varphi_3(\alpha_{t_{i+1}} - \alpha_{t_i}) + O(H_i^4),
\end{align*}
where the second equation uses $\frac{\gamma \bigl(\alpha_{t_i+c_2H_i} - \alpha_{t_i}\bigr)^2 + \bigl(\alpha_{t_i+c_3H_i} - \alpha_{t_i}\bigr)^2
      }{\gamma(\alpha_{t_{i}+ c_2H_{i}} - \alpha_{t_{i}}) + (\alpha_{t_{i}+ c_3H_{i}} - \alpha_{t_{i}})} = \frac{\gamma c_2^2 + c_3^2}{\gamma c_2  + c_3} 
     \bigl(\alpha_{t_{i+1}} - \alpha_{t_i}\bigr)+ \mathcal{O}(H_i^2)$, and the final step applies the definition of $\gamma$ and  the identity $\varphi_2(\alpha_{t_{i+1}} - \alpha_{t_i})  = 3\varphi_3(\alpha_{t_{i+1}} - \alpha_{t_i}) + \mathcal{O}(H_i)$.
The last equation in \eqref{eq:rk3-exp-order-conditions} is obtained by first expressing
\begin{align*}
    k_2(\dY_{t_i})  - k_1(\dY_{t_i}) = \bigl(\alpha_{t_i+c_2H_i} - \alpha_{t_i}\bigr) \frac{\de\rscore_{t_i}(\sY_{t_i})}{\de\alpha} + \mathcal{O}(H_i^2),
\end{align*}
and then enforcing
{\small
\begin{align*}
    &\gamma  \Bigl(\sigma_{T - t_i - c_2H_i}\bigl(\alpha_{t_i+c_2H_i} - \alpha_{t_i}\bigr)^2\varphi_2(\alpha_{t_i+c_2H_i} - \alpha_{t_i}) \Bigr) H_i 
    \\&+  \Bigl(\sigma_{T - t_{i} - c_3H_i}\bigl(\alpha_{t_{i}+ c_3H_{i}} - \alpha_{t_i}\bigr)^2\varphi_2\bigl(\alpha_{t_{i} + c_3H_{i}} - \alpha_{t_i}\bigr) - a_{32} \bigl(\alpha_{t_i+c_2H_i} - \alpha_{t_i}\bigr) H_i\Bigr)H_i = \mathcal{O}(H_i^4) ,
\end{align*}}which leads to our definition of $a_{32}$. We note that $a_{32}$ can be further simplified to retain only the dominant $\mathcal{O}(1)$ term.

To construct $c_j$ such that all relevant time points, including $t_i$, $t_{i+1}$, and $\{t_i + c_j H_i\}$, lie within the discrete time sequence $\bT$.  For the 1-stage first-order scheme and the 2-stage second-order scheme with $c_2 = 1$, detailed in \eqref{eq:exp-RK-stiff1} and \eqref{eq:exp-RK-stiff2}, no intermediate time evaluations are required. In these cases, it suffices that $t_i \in \bT$, which implies that $H_i$ is an integer multiple of the base step size $\Delta t$. For the 3-stage third-order detailed in \eqref{eq:exp-RK-stiff3}, we can choose $c_2 = \frac{1}{3}$ and $c_3 = \frac{2}{3}$. To ensure these points align with $\bT$, we require that each $H_i$ is a multiple of $3\Delta t$.

\subsection{Coefficient Estimations}
\label{s:coefficient_est}
Finally, we present the proof of \cref{lem:coefficients},  which establishes bounds for the modified Runge-Kutta coefficients introduced in \eqref{eq:RK-update-beta-k-rf} and \eqref{eq:Exp-RK-update-beta-k-rf}.
\begin{proof}
For the standard Runge-Kutta schemes, the coefficients are constant. After being rewritten as \eqref{eq:RK-update-beta-k-rf}, the coefficients are scaled by $\beta_t$, and the resulting bounds follow from the regularity \Cref{a:bounds on beta_t} on $\beta_t$. 
For the enhanced exponential Runge-Kutta schemes up to third order discussed above, we  begin by  recalling the relevant time change variables
    \begin{align*}
    \lambda_t = e^{-\frac{1}{2} \int_0^t \beta_s \de s} \quad  \sigma_t  = \sqrt{1 - \lambda_t^2} \quad
\zeta(t) = \frac{1}{2} \int_{T-t}^T \beta_s \de s \quad \textrm{and} \quad 
\alpha_t = \zeta(t) - \log \sigma_{T-t},
    \end{align*}
and their  derivatives
    \begin{align}
    \label{eq:derivatives}
    \frac{\rd \sigma_t^2}{\rd t}  = \beta_t\bigl( 1 - \sigma_t^2 \bigr)
    \quad
        \frac{\rd \zeta(t)}{\rd t} = \beta_{T-t} 
        \quad \textrm{and} \quad 
        \frac{\rd \alpha_t}{\rd t} = -\frac{\beta_{T-t}}{2\sigma_{T-t}^2}. 
    \end{align}
The Runge-Kutta coefficients in \cref{eq:exp-RK-stiff1,eq:exp-RK-stiff2,eq:exp-RK-stiff3} are modified through \eqref{eq:Exp-RK-update-beta-k-rf}, these modified coefficients are linear combinations of the following functions: 
\begin{align}
\begin{split}
\label{eq:RK-f}
f_0(h; t, \fc) &=  \frac{
e^{\zeta(t + \fc h) - \zeta(t)}-1}{\fc h} ,
\\
    f_1(h; t, \fc , \fc _1) &= \sigma_{T - t - \fc_1h}\sigma_{T - t - \fc h}\frac{\alpha_{t + \fc h} - \alpha_{t}}{\fc h}\varphi_1(\alpha_{t  + \fc h} - \alpha_{t})  ,
    \\
    f_2(h; t, \fc , \fc_1, \fc_2, \fc_3, \gamma) &= \sigma_{T - t - \fc_1h}\sigma_{T - t - \fc h}\frac {(\alpha_{t + \fc h} - \alpha_{t})^2}{\bigl(\gamma (\alpha_{t + \fc_2h} - \alpha_{t}) + (\alpha_{t + \fc_3h} - \alpha_{t})\bigr)h}\varphi_2(\alpha_{t  + \fc h} - \alpha_{t}),
\end{split}
\end{align}
where these constants $0\leq\gamma$, $0 < \fc_i \leq \fc \leq 1$ are related to, but may differ in index from, $c_i$ and $\gamma$ constants in the Runge-Kutta schemes~\cref{eq:exp-RK-stiff1,eq:exp-RK-stiff2,eq:exp-RK-stiff3}.

Next, we show that these functions $f_i$  and their derivatives are bounded. These functions can be further decomposed into the following terms
\begin{align}
\begin{split}
\label{eq:RK-coeff}
&\frac{\zeta(t  + \fc h) - \zeta(t)}{\fc h} =   \int_0^1 \beta_{T - t - \fc h s}\de s \in [C_\beta^{-1}, C_{\beta}],
\\
&\sigma_{T-t-\fc_1h}\sigma_{T-t-\fc_2h}\frac{\alpha_{t  + \fc h} - \alpha_{t}}{\fc h} =  \sigma_{T-t-\fc_1h}\sigma_{T-t-\fc_2h} \int_0^1 \frac{\beta_{T-t-\fc h s}}{2\sigma_{T-t-\fc hs}^2}\de s \leq \frac{C_\beta}{2} \Bigl(\frac{\sigma_{T - t}}{\sigma_{T-t-\fc h}}\Bigr)^2 ,
\\
&\frac{ \alpha_{t + \fc h} - \alpha_{t} }{ \gamma (\alpha_{t + \fc_2h} - \alpha_{t}) + (\alpha_{t + \fc_3h} - \alpha_{t}) } = 
\frac{\int_0^1 \frac{\beta_{T-t-\fc h s}}{\sigma_{T-t-\fc hs}^2}\de s}{\gamma\int_0^1 \frac{\beta_{T-t-\fc_2h s}}{\sigma_{T-t-\fc_2hs}^2}\de s + \int_0^1 \frac{\beta_{T-t-\fc_3h s}}{\sigma_{T-t-\fc_3hs}^2}\de s} \leq \frac{C_\beta^2}{\gamma+1} \Bigl(\frac{\sigma_{T - t}}{\sigma_{T-t-\fc h}}\Bigr)^2,
\\
&\varphi_k(\alpha_{t  + \fc h} - \alpha_{t}) \leq   e^{\alpha_{t  + \fc h} - \alpha_{t}}  =   e^{\zeta(t  + \fc h) - \zeta(t)} \frac{\sigma_{T - t}}{\sigma_{T-t-\fc h}} .
\end{split}
\end{align}
The first three equations are due to the fact 
that $C_\beta^{-1} \leq \beta_t \leq C_\beta$, the fact that $\sigma_t \in [0,1]$ is increasing,  and the definition of $\frac{\de \zeta(t)}{\de t}$ and $\frac{\de \alpha_t}{\de t}$ in \eqref{eq:derivatives}. The last inequality follows from $\varphi_k(h) = \int_0^1 \frac{x^{k-1}}{(k-1)!}e^{h(1-x)}dx \leq e^h$. The boundedness of the functions $f_i$ in \eqref{eq:RK-f} can be obtained by substituting the following estimate into \eqref{eq:RK-coeff}:
\begin{align*}
     \Bigl(\frac{\sigma_{T - t}}{\sigma_{T - t - \fc h}}\Bigr)^2 
     = 
     \frac{1 - e^{-\int_0^{T-t-\fc h}\beta_s \de s - \int_{T-t-\fc h}^{T-t}\beta_s \de s}}{1 - e^{-\int_0^{T-t-\fc h}\beta_s \de s}}
     \leq 
    \frac{\int_0^{T-t-\fc h}\beta_s \de s + \int_{T-t-\fc h}^{T-t}\beta_s \de s}{\int_0^{T-t-\fc h}\beta_s \de s}
     \leq 
      1 + \fc C_\beta^2 \frac{h}{\tau}.
\end{align*}
Here, we used the inequality $\frac{1 - e^{-x-y}}{1-e^{-x}} \leq \frac{x+y}{x}$ for $x,y >0$, the time domain $t+\fc h \leq T - \tau$, and $h \leq H_{\rm max} \leq \tau$. The estimates for the derivatives of these $f_i$ functions can be obtained by the fact that these terms in \eqref{eq:RK-coeff} are smooth with bounded derivatives. Except that $\sigma_{t}$ has a singularity as $t$ approaches $0$, its behavior is like $\sqrt{t}$. For $t\in[\tau,T]$, we have 
\begin{align*}
      1\geq \sigma_{t} = \sqrt{1 - e^{-\int_0^{t}\beta_s \de s}}
      \geq \sqrt{1 - e^{-\frac{\tau}{C_\beta}}}
      \geq \sqrt{\min\Bigl\{\frac{\tau}{2C_\beta},\frac{1}{2}\Bigr\}} .
\end{align*}
Its derivative satisfies
\begin{equation*}
      |\partial_t \sigma_{t}| = \Bigl|\sigma_t \frac{\beta_t}{2} \Bigl(\frac{1}{\sigma_t^2} - 1\Bigr)\Bigr| \leq \sigma_t 
      C_\beta \max\{\frac{2C_\beta}{\tau}, 1\} .
\end{equation*}
Similarly, for any integer $\ell\in \mathbb Z$
\begin{equation*}
      |\partial_t \sigma^\ell_{t}| = \Bigl|\ell \sigma^{\ell}_t \frac{\beta_t}{2} \Bigl(\frac{1}{\sigma_t^2} - 1\Bigr)\Bigr| \leq \ell\sigma_t ^{\ell}
      C_\beta \max\{\frac{2C_\beta}{\tau}, 1\} .
\end{equation*}
For higher derivatives we treat $\beta_t$ and its derivatives as coefficients, and only track the change of $\sigma_t$. The $\ell$-th derivative of $\sigma_t$ is a Laurent polynomial of $\sigma_t$ with degree between $1-2\ell$ and $1$, and the coefficients of this polynomial involve $\beta_t$ and its derivatives. As a consequence
\begin{align}\label{e:high-der}
    |\partial^\ell_t \sigma_{t}| \leq C\max_{1-2\ell\leq r\leq 1}\{\sigma^r_t\}\leq \frac{2\ell C}{\tau^\ell} \sigma_t, 
\end{align}
where the constant $C$ depends on $C_\beta$. In a word, each derivative gives an extra factor $1/\tau$. 

Back to \eqref{eq:RK-coeff}. By \eqref{e:high-der}, when a term $\sigma_{T-t-\fc h}$ in \eqref{eq:RK-coeff} is differentiated with respect to $h$, it might introduce a factor proportional to $\frac{1}{\tau}$ in front, thereby explaining the $\tau$-scaling in the estimates of \eqref{e:derbound}. This completes the proof.
\end{proof}

\end{document}